\documentclass[journal]{IEEEtran}
\usepackage[numbers,sort&compress]{natbib}
\usepackage{amsfonts}
\usepackage[cmex10]{amsmath}
\usepackage{array}

\usepackage{eqparbox}
\usepackage{stfloats}

\usepackage{amssymb}
\usepackage{float}
\usepackage{indentfirst}
\usepackage{dsfont,enumerate,enumitem,comment,url}
\usepackage{multirow,booktabs,bigstrut,longtable,array}
\usepackage[normalem]{ulem}
\usepackage{graphicx,color,rotating,soul}
\usepackage{makecell,xcolor,pbox,framed}
\usepackage{wrapfig,setspace,subfigure,float}
\usepackage{last page,acronym}
\usepackage[most]{tcolorbox}    
\tcbuselibrary{breakable} 
\usepackage{hyperref}
\usepackage{makeidx}
\usepackage{tabularx}
\usepackage{cases}
\usepackage{multirow}
\usepackage{booktabs}
\usepackage{diagbox}
\usepackage{mathtools}
\usepackage{color}
\usepackage{stfloats}
\usepackage{lipsum}
\usepackage{amsfonts}
\usepackage{bm}
\usepackage{dsfont}
\usepackage{lipsum}
\usepackage{amsthm}
\usepackage{graphicx, epsfig, subfigure}
\usepackage[font=small]{caption}
\allowdisplaybreaks[4]
\makeatletter

\newif\if@restonecol
\makeatother

\usepackage[linesnumbered,ruled,vlined]{algorithm2e}%[ruled,vlined]{
\usepackage{algpseudocode}
\usepackage{amsmath, bm}
\begin{document}
\title{Semantic-based Distributed Learning for Diverse and Discriminative Representations}
\author{Zhuojun~Tian, Chaouki~Ben~Issaid, and Mehdi~Bennis
\thanks{
Z.~Tian (zhuojun@kth.se) was with the Center for Wireless Communications, University of Oulu, Oulu, Finland and now is with the Division of Information Science and Engineering, KTH Royal of Technology, Stockholm, Sweden. 
C.~B.~Issaid and M.~Bennis are with the Center for Wireless Communications, University of Oulu, Oulu 90014, Finland. Email: \{chaouki.benissaid, mehdi.bennis\}@oulu.fi.
}
\thanks{The code for the paper will be available on \href{https://github.com/ZhuoJTian/Distributed-MCR2}{https://github.com/ZhuoJTian/Distributed-MCR2}.}
}

\maketitle
	
\begin{abstract}
In large-scale distributed scenarios, increasingly complex tasks demand more intelligent collaboration across networks, requiring the joint extraction of structural representations from data samples. However, conventional task-specific approaches often result in nonstructural embeddings, leading to collapsed variability among data samples within the same class, particularly in classification tasks. To address this issue and fully leverage the intrinsic structure of data for downstream applications, we propose a novel distributed learning framework that ensures both diverse and discriminative representations.
For {independent and identically distributed (i.i.d.)} data, we reformulate and decouple the global optimization function by introducing constraints on representation variance. The update rules are then derived and simplified using a primal-dual approach. For non-i.i.d. data distributions, we tackle the problem by clustering and virtually replicating nodes, allowing model updates within each cluster using block coordinate descent. In both cases, the resulting optimal solutions are theoretically proven to maintain discriminative and diverse properties, with a guaranteed convergence for i.i.d. conditions.
Additionally, semantic information from representations is shared among nodes, reducing the need for common neural network architectures. Finally, extensive simulations {on MNIST, CIFAR-10 and CIFAR-100} confirm the effectiveness of the proposed algorithms in capturing global structural representations.
\end{abstract}pdf

\begin{IEEEkeywords}
Distributed learning, representation learning, semantic communication, maximal coding rate reduction.
\end{IEEEkeywords}

\section{Introduction} \label{Sec1}
The rapid increase in devices with expanding storage and hardware capabilities has facilitated the efficient collection and processing of large-scale datasets. With the remarkable success of deep neural networks, distributed learning has gained significant attention for handling large-scale collaborative scenarios \cite{chen2021distributed}. 
As the collected data samples embed more complicated structures, the application scenarios and the user needs become various, there arise more challenging processing tasks. These intricate tasks require devices to collaboratively extract information from distinct datasets across multiple sources while preserving the intrinsic structures of the data to best
facilitate subsequent tasks \cite{ma2022principles}. 

In conventional training algorithms for over-parameterized deep networks, particularly in classification tasks, the learned intermediate features often exhibit the "neural collapse" phenomenon \cite{zhu2021geometric, yaras2022neural}, as illustrated in Fig. \ref{fig_sec1_1}: The within-class features converge to their respective class means, while the means of different classes form a tightly structured frame. Empirical studies have observed this effect across various loss functions, including the widely used cross-entropy (CE) and mean squared error (MSE). Such feature representations compromise the intrinsic structure of individual data samples, reducing interpretability and limiting their effectiveness for more complex downstream tasks. 
Furthermore, the compositional capability \cite{elmoznino2024complexity}, which is crucial for enabling connected intelligence in multi-agent systems, may be lost if the diversity in data structure is not maintained. To address this challenge, the original manifold of distributed datasets must be mapped onto global low-dimensional subspaces that maintain diversity within each class while ensuring clear discrimination between classes.
\begin{figure}[!htp]
    \centering
    \includegraphics[width=0.5\textwidth]{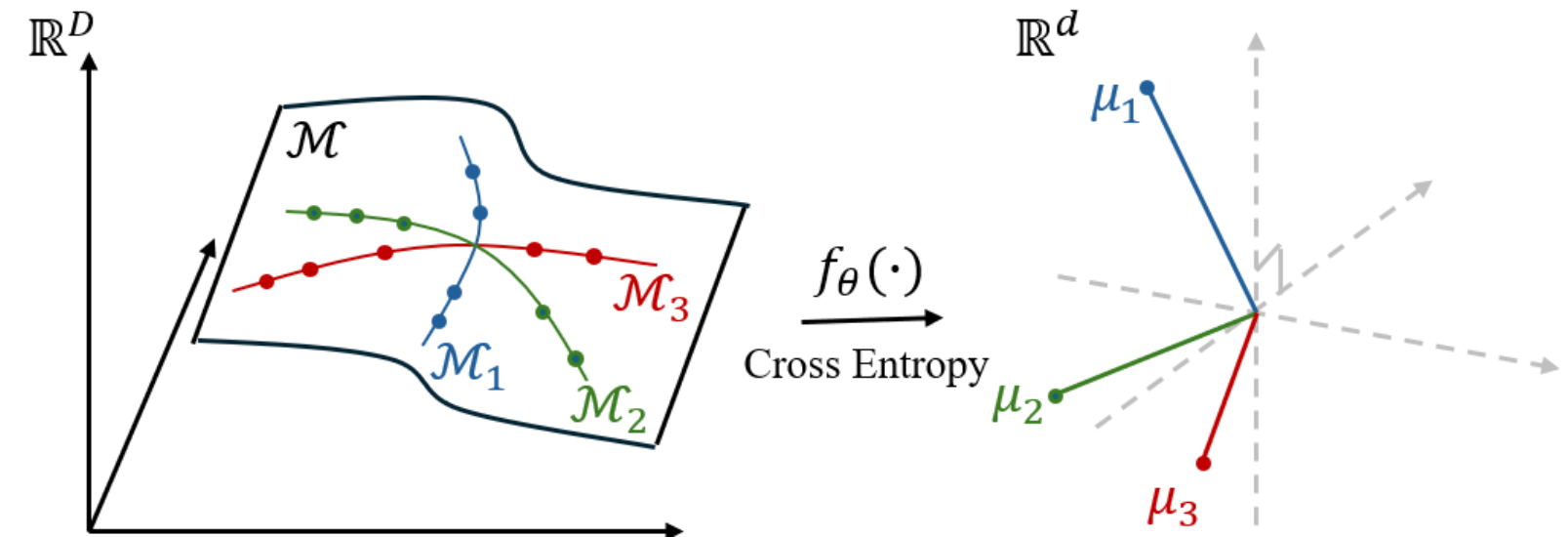}
    \caption{Illustration of the neural collapse phenomenon.}
    \label{fig_sec1_1}
    \vspace{-0.5cm}
\end{figure}

Meanwhile, traditional distributed learning algorithms, such as federated learning (FL) and decentralized stochastic gradient descent (SGD), rely on parameter transmission between nodes, imposing strict constraints on network architectures. However, in real-world scenarios, neural networks across different nodes may differ due to variations in computational power and storage capacity. To develop a more practical networked learning system, this work also aims to address system heterogeneity among nodes by enabling the transmission of semantic information instead of raw parameters. {Here the semantic information refers to the extracted representation information in the latent space.}

\subsection{Related Works}
\textbf{Distributed Learning. }Recently, distributed learning has made remarkable advances, including federated learning (FL) \cite{li2020federated}, which relies on a central processing node, and decentralized learning over networks \cite{lian2017can}. Numerous procedures and frameworks have been developed to address key challenges such as statistical heterogeneity \cite{smith2017federated, pu2021distributed, tian2023distributed, sattler2019robust, collins2021exploiting, tian2023distributedCS}, communication efficiency \cite{zhu2020one, mcmahan2017communication, li2021communication, li2021fedmask}, etc. However, most existing algorithms are task-specific, focusing on achieving consensus among model parameters, which often leads to the neural collapse phenomenon. As a result, clients struggle to extract expressive features necessary for more complex downstream processing, highlighting the need for distributed learning algorithms that generate meaningful representations.
Moreover, in all these methods, model parameters or gradients must be transmitted and processed, necessitating a shared neural network architecture among agents. Some recent approaches have explored heterogeneous system settings through split-learning \cite{he2020group, gupta2018distributed}, sub-model training \cite{diao2020heterofl, horvath2021fjord, deng2022tailorfl, jiang2022model, tian2025communication}, and factorization \cite{yao2021fedhm, mei2022resource}, where a central server has a large global model and conducts model partitioning, compression, or factorization. 
%In these cases, smaller models deployed on individual agents are derived from the global model. Additionally, the Dis-PFL protocol \cite{dai2022dispfl} introduced personalized sparse masks to activate model parameters within a common large neural network architecture.
However, all these approaches depend on aggregating model parameters and require a globally shared model structure to exploit correlations among varying small models.

\textbf{Distributed Representation Learning. }The effectiveness of deep learning heavily depends on the quality of data representations. As a result, the rapid advancement of deep learning has been closely linked to significant progress in representation learning \cite{payandeh2023deep, li2018survey, bengio2013representation}. This approach has been successfully applied across various domains, including image processing, natural language processing (NLP), speech recognition, and signal processing.
In multi-agent scenarios \cite{liu2020distributed}, additional challenges arise due to the distributed nature of datasets and the potential multi-view perspectives of the global data space. Specifically, representations learned by different agents may become misaligned due to randomness in the training process. Additionally, the limited data available to each agent can lead to variations in data distribution, causing inconsistencies in the representation spaces encoded by different local models. Considering these factors, distributed representation learning primarily aims to address alignment and composition challenges.
The authors in \cite{aguerri2019distributed} approach this problem from an information-theoretic perspective by extending Tishby’s centralized Information Bottleneck (IB) method to a distributed setting, focusing on extracting the most relevant information for a given task. 

{
More recently, some other works \cite{zhuang2021collaborative, lubana2022orchestra, liao2024rethinking, zhang2023federated} explore unsupervised or self-supervised federated representation learning using contrastive learning \cite{chen2020simclr, ghalkha2025sheafalign} or data augmentation techniques. 
Among these works, FedU proposed in \cite{zhuang2021collaborative} deals with heterogeneity in federated self-supervised learning by divergence-aware predictor update rule, while Orchestra \cite{lubana2022orchestra} utilizes local-global clustering.
FedU$^2$ in \cite{liao2024rethinking} consists of flexible uniform regularizer and efficient unified aggregator to avoid the representation collapse and promote unified representations among clients.
However, these methods require a central processor to coordinate the global information. Moreover, they rely on model aggregation through weighted averaging to construct a global model, which necessitates a shared neural network architecture among devices.}

\textbf{Learning diverse and discriminative representation. }The Maximal Coding Rate Reduction (MCR$^2$) principle, introduced in \cite{yu2020learning}, optimizes a coding reduction objective to expand the overall feature set while simultaneously compressing and linearizing features within each class. Leveraging MCR$^2$, the authors in \cite{yu2020learning} aimed to learn representations that are both diverse and discriminative for multi-class data. This approach maps real data distributions onto a feature space composed of multiple independent multi-dimensional linear subspaces, making it well-suited for both discriminative and generative tasks. Experimental results demonstrated that the learned representations effectively capture the structural information of the original datasets.
Building on the MCR$^2$ principle, the authors in \cite{dai2021closed} incorporated an autoencoder architecture and introduced a closed-loop transcription mechanism by formulating a two-player min-max game between the encoder and decoder. However, this approach heavily relies on a large amount of training datasets and is therefore impractical for distributed devices with limited data availability. The method proposed in \cite{tian2025compositional} utilizes the MCR$^2$ principle and subspace fusion to implement distributed multi-view perception.
\begin{figure}[t]
    \centering
    \includegraphics[width=0.5\textwidth]{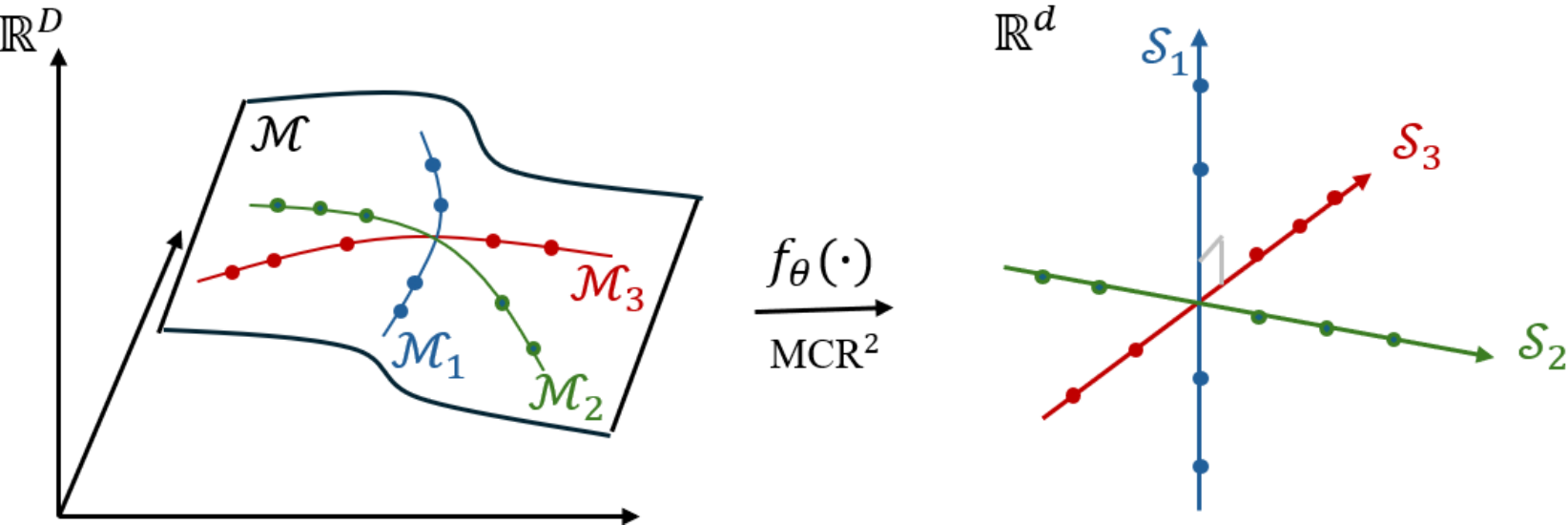}
    \caption{Illustration of the learned subspaces of MCR$^2$ principle.}
    \label{fig_sec1_2}
    \vspace{-0.5cm}
\end{figure}

\subsection{Contribution}
{Inspired by the MCR$^2$ principle, this work proposes two distributed representation learning algorithms grounded in rigorous theoretical derivations and guarantees. By extending MCR$^2$ to distributed scenarios, our approach allows agents to collaboratively capture global structural representations using only their limited local data, while cooperating with other agents. This framework jointly addresses the challenges of neural collapse, alignment, and composition in distributed representation learning.}
Through cooperation, the representation spaces within the same class are aligned across different agents, as illustrated on the right side of Fig. \ref{sysmodel}, ensuring that the global representations across the network reside within a unified space.
Unlike the approach in \cite{aguerri2019distributed}, which employs the Information Bottleneck (IB) method, MCR$^2$ treats data labels as side information, offering a distinct perspective. Additionally, MCR$^2$ provides a metric for measuring ``distance", making it conceptually similar to contrastive learning methods \cite{zhuang2021collaborative, zhang2023federated}. However, unlike contrastive learning, which focuses on sample-wise distance, MCR$^2$ considers the relationships between subsets of samples while incorporating label information as side information. Furthermore, our approach establishes consensus at the representation space level rather than relying on model parameter aggregation, enhancing its adaptability to diverse scenarios.

Our contribution can be summarized as follows.
\begin{itemize}
    \item{To achieve global representations that maintain diversity within classes and discriminability between classes, we apply the MCR$^2$ principle and extend it to distributed settings. Specifically, we reformulate its original optimization problem into decentralized sub-problems by incorporating consensus constraints on the statistical variance of the representation subspaces. For i.i.d. data distributions, the update rules are derived and simplified using augmented Lagrangian techniques.}
    \item{We also address the case of non-i.i.d. data distributions among agents by reformulating the optimization problem into a cluster-based form. Specifically, we propose clustering the agents and virtually replicating nodes locally so that data distributions become i.i.d. within each cluster. This allows the optimization problem to be decoupled into multiple parallel clusters, where within each cluster, the sub-problems are solved sequentially by different agents using the block coordinate descent method.}
    \item{For both i.i.d. and non-i.i.d. data distributions, we provide theoretical guarantees on the diverse and discriminative properties of the optimal solutions for the reformulated optimization problems. Additionally, we offer a convergence guarantee for the algorithm. Furthermore, in our proposed framework, devices share semantic information about their representations instead of model parameters, which is typically done in conventional distributed representation learning. This approach allows for the use of different neural network architectures among agents.}
\end{itemize}

The remainder of this paper is structured as follows. Section \ref{Sec2} introduces the system model and the MCR$^2$ principle. In Section \ref{Sec3}, we address the i.i.d. data distribution and reformulate the decomposable optimization problem, from which we develop the distributed representation learning algorithm for i.i.d. condition. Section \ref{Sec4} focuses on non-i.i.d. data distribution, where we formulate the cluster-based global optimization problem and design the corresponding distributed learning algorithm. We give the theoretical analysis on the properties of the learned representation and convergence of the algorithm in Section \ref{Sec5}.
The simulation results are presented in Section \ref{Sec6}, with the conclusion provided in Section \ref{Sec7}.

\section{Preliminary}\label{Sec2}
\subsection{System Model}
Consider a decentralized multi-agent communication network as shown in Fig. \ref{sysmodel}, which can be represented by an undirected graph $\mathcal{G}=(\mathcal{V},\mathcal{E})$, where $\mathcal{V}=\{1,\dots, N\}$ denotes the set of $N$ distributed agents/nodes and $\mathcal{E}=\{\varepsilon_{ij}\}_{i,j \in \mathcal{V}} $ represents the set of communication links between any two adjacent nodes. Let $\mathcal{N}_i$ denote the set of all neighboring nodes connected to node $i$ and we denote the number of nodes in $\mathcal{N}_i$ by $d_i=|\mathcal{N}_i|$. The adjacency matrix of $\mathcal{G}$ is denoted by $\textbf{A}$, where $\textbf{A}(i,j)=1$ if $\varepsilon_{ij}\in\mathcal{E}$ and $\textbf{A}(i,j)=0$, otherwise. 
Each node has its own dataset from the global data distribution, which is denoted by $\mathcal{D}_i=\{\mathcal{X}_i, \mathcal{Y}_i\}$, and we define $m_i$ as the number of training data samples in node $i$.

The output of the representation learning neural network in agent $i$ is denoted by $\bm{Z}_i\in\mathbb{R}^{d\times m_i}$, where $d$ is the dimension of the output feature and $m_i$ is the number of data samples in node $i$. 
The global representation can be defined as $\bm{Z}=[\bm{Z}_1, \cdots, \bm{Z}_N]\in\mathbb{R}^{d\times m}$, where $m=\sum_{i=1}^Nm_i$ denotes the total number of samples. The representation learning neural network (encoder) in node $i$ is denoted by $f_i$ parameterized by $\bm{\theta}_i$ and the output can be represented as $\bm{Z}_i=f_i(\bm{X}_i, \bm{\theta}_i)$.

\begin{figure*}[!htp]
    \centering
    \includegraphics[width=1.0\textwidth]{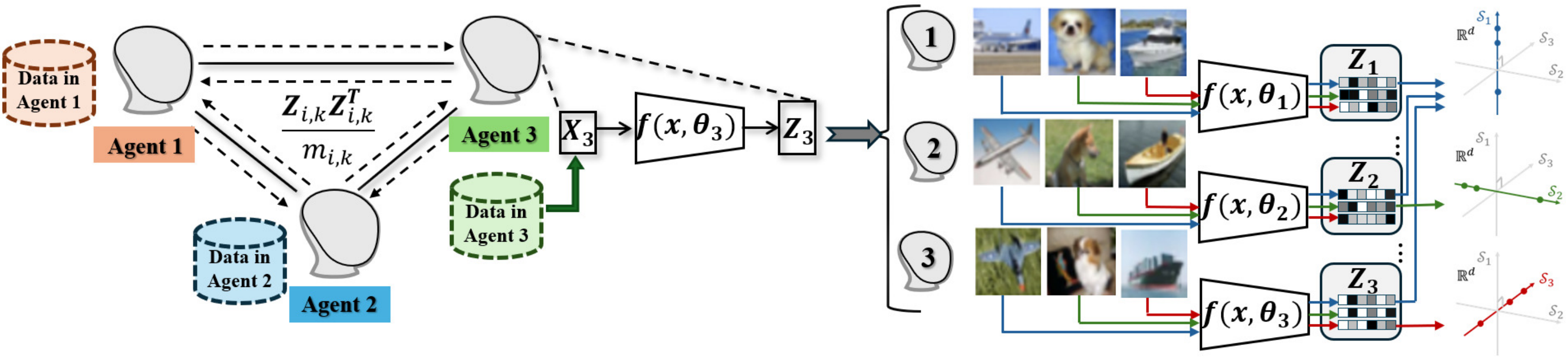}
    \caption{System model: The distributed nodes collaboratively learn the global representations by transmitting the semantic information. The right side shows the diverse and discriminative properties of the learned representations.}
    \label{sysmodel}
    \vspace{-0.5cm}
\end{figure*}

\subsection{Maximal Coding Rate Reduction}
The compactness of the learned features as a whole can be measured by the average coding length per sample when the sample size is large enough, i.e., the coding rate subject to the distortion \cite{ma2007segmentation, yu2020learning}, is given by.
\begin{equation}
    \label{eqR}
    R(\bm{Z}, \epsilon) = \frac{1}{2}\log\det(\bm{I}+\frac{d}{m\epsilon^2}\bm{Z}\bm{Z}^T).
\end{equation}
Considering that the generated features $\bm{Z}$ have multiple classes which may belong to different low-dimensional subspaces, the averaged number of bits per sample (the coding rate) is given in (\ref{eqRc}) \cite{ma2007segmentation}, where $\bm{\Pi}_k$ is a diagonal matrix whose diagonal entries indicate the membership of the samples in the multiple classes. In this sense, the label serves as side information.
\begin{equation}
    \label{eqRc}
    R^c(\bm{Z}, \epsilon|\bm{\Pi})\!=\!\sum_{k=1}^K\frac{tr(\bm{\Pi}_k)}{2m}\log\det(\bm{I}+\frac{d}{tr(\bm{\Pi}_k)\epsilon^2}\bm{Z}\bm{\Pi}_k\bm{Z}^T).
\end{equation}
To maximize the discrimination of the features among different classes, the whole space of $\bm{Z}$ must be as large as possible. On the other hand, within each class, the subspace should be of small volume to make the representation compact. Then, the loss function on the principle of Maximal Coding Rate Reduction can be expressed as follows \cite{yu2020learning}:
\begin{equation}
    \label{MCR2}
    \max\quad\triangle R(\bm{Z}, \bm{\Pi}, \epsilon)=R(\bm{Z}, \epsilon) - R^c(\bm{Z}, \epsilon|\bm{\Pi}).
\end{equation}
Given this objective function, our task is to decouple it into sub-problems for multiple distributed nodes, in order to jointly optimize their latent variables (feature space) $\mathcal{Z}=\{\bm{Z}_0, \dots, \bm{Z}_N\}$.
In the following, we aim to develop a distributed representation learning algorithm that can learn both diverse and discriminative representations across agents. This approach enhances the transparency of the learned feature representations, making them more interpretable and controllable for downstream tasks. We consider two different data distribution conditions among agents: i.i.d. and non-i.i.d., and we develop the corresponding algorithms for each scenario.

\section{Global Representation Learning for i.i.d. Data Distribution}\label{Sec3}
{In this section, we consider an i.i.d. data distribution. Firstly, the original optimization problem \eqref{MCR2} is equivalently reformulated so that the it can be decoupled among distributed agents. Building on augmented Lagrangian and primal–dual techniques, the original constrained optimization problem is reformulated into an unconstrained one, from which distributed update rules can be systematically derived.}

For the first term on the right side of (\ref{MCR2}), based on the concave property of the $\log\det$ function, we can write:
\begin{align}
    \label{eqR1}
    R(\bm{Z}, \epsilon) & = \frac{1}{2}\log\det(\bm{I}+\frac{d}{m\epsilon^2}\bm{Z}\bm{Z}^T)\\
    % & = \frac{1}{2}\log\det\Big[\bm{I}+\frac{d}{m\epsilon^2}(\bm{Z}_1\bm{Z}_1^T+\cdots+\bm{Z}_N\bm{Z}_N^T)\Big] \notag \\
    & = \frac{1}{2}\log\det\Big[\sum_{i=1}^N(\frac{m_i}{m}(\bm{I}+\frac{d}{m_i\epsilon^2}\bm{Z}_i\bm{Z}_i^T))\Big] \notag \\
    & \ge \sum_{i=1}^N\frac{m_i}{2m}\log\det(\bm{I}+\frac{d}{m_i\epsilon^2}\bm{Z}_i\bm{Z}_i^T),\notag
\end{align}
where the first equality comes from the definition of $\bm{Z}$ and the last inequality comes from Jensen's inequality given the concavity of $\log\det$ function. The equality holds if and only if $\bm{Z}_i\bm{Z}_i^T/m_i=\bm{Z}_j\bm{Z}_j^T/m_j$ for any two nodes $i,j\in\mathcal{V}$.

We define $\bm{Z}_{i,k}\in\mathbb{R}^{d\times m_{i,k}}$ as the features in node $i$ corresponding to the $k$-th class, i.e., $\bm{Z}_{i,k}$ is $\bm{Z}_{i}\bm{\Pi}_{i,k}$ with zero columns removed.
Drawing inspiration from the equality condition of Jensen’s inequality, we propose a procedure for decentralized representation learning. 
Specifically, we introduce a variance constraint (variance consensus) between the feature spaces of different nodes with respect to each class, i.e., ${\bm{Z}_{i,k}\bm{Z}_{i,k}^T}/{m_{i,k}} = {\bm{Z}_{j,k}\bm{Z}_{j,k}^T}/{m_{j,k}}$ for any class $k$ and arbitrary two nodes $i$ and $j$. Then, we derive the following decomposable result for the second term in (\ref{MCR2}):
\begin{align}
    \label{eqRc1}
    & R^c(\bm{Z}, \epsilon|\bm{\Pi}) = \sum_{k=1}^K\frac{tr(\bm{\Pi}_k)}{2m}\log\det(\bm{I}+\frac{d}{tr(\bm{\Pi}_k)\epsilon^2}\bm{Z}\bm{\Pi}_k\bm{Z}^T) \notag \\
    & = \sum_{k=1}^K\frac{m_{k}}{2m}\log\det(\bm{I}+\frac{d}{m_k\epsilon^2}\sum_{i=1}^N\bm{Z}_{i,k}\bm{Z}_{i,k}^T) \\
    & = \sum_{k=1}^K\sum_{i=1}^N\frac{m_{i,k}}{2m}\log\det(\bm{I}+\frac{d}{m_{i,k}\epsilon^2}\bm{Z}_{i,k}\bm{Z}_{i,k}^T). \notag
\end{align}
Given (\ref{eqR1}) and (\ref{eqRc1}), we formulate our optimization problem, using the lower bound of (\ref{eqR1}), as follows:
\begin{align}
    \label{opt1}
    & \max_{\bm{Z}}\quad\sum_{i=1}^N \Big[R_i(\bm{Z}_i)-R_i^c(\bm{Z}_i)\Big]\\
    & \text{s.t.}\quad \|\bm{Z}_{i,k}\|_F^2 = m_{i,k}, \forall i\in\mathcal{V}, 1\le k\le K, \notag \\
    & \quad\quad\frac{\bm{Z}_{i,k}\bm{Z}_{i,k}^T}{m_{i,k}}=\frac{\bm{Z}_{j,k}\bm{Z}_{j,k}^T}{m_{j,k}}, \forall i\in\mathcal{V}, j\in\mathcal{N}_i, 1\le k\le K,\notag
\end{align}
where
\begin{align}
    &R_i^c(\bm{Z}_i) = \sum_{k=1}^K\frac{m_{i,k}}{2m}\log\det(\bm{I}+\frac{d}{m_{i,k}\epsilon^2}\bm{Z}_{i,k}\bm{Z}_{i,k}^T), \\
    &R_i(\bm{Z}_i) = \frac{m_i}{2m}\log\det(\bm{I}+\frac{d}{m_i\epsilon^2}\bm{Z}_i\bm{Z}_i^T).
\end{align}
In (\ref{opt1}), the first constraint aims to ensure that the reduction magnitude is comparable across different representations. This is achieved by normalizing each feature to lie on the unit sphere \cite{yu2020learning}, which can be easily implemented by applying a normalization function to the output layer.
%the first constraint seeks to ensure the reduction amount is comparable across different representations. This can be achieved by normalizing each feature to lie on the unit sphere \cite{yu2020learning}, which can be easily implemented by adding a normalization function to the output layer. 
Furthermore, with the connected network topology, the second constraint between connected nodes is sufficient to ensure ${\bm{Z}_{i,k}\bm{Z}_{i,k}^T}/{m_{i,k}} = {\bm{Z}_{j,k}\bm{Z}_{j,k}^T}/{m_{j,k}}$ for any class $k$ and arbitrary two nodes $i,j$.

By introducing the intermediate variable $\bm{T}_{i,j,k}$, we equivalently reformulate problem (\ref{opt1}).
\begin{align}
    \label{opt1_2}
    & \min_{\bm{Z}}\quad \sum_{i=1}^NR_i^c(\bm{Z}_i) - R_i(\bm{Z}_i)\\
    & \text{s.t.}\quad \|\bm{Z}_{i,k}\|_F^2 = m_{i,k}, \forall i\in\mathcal{V}, 1\le k\le K, \notag \\
    & \quad\quad \frac{\bm{Z}_{i,k}\bm{Z}_{i,k}^T}{m_{i,k}}=\bm{T}_{i,j,k}, \forall i\in\mathcal{V}, \forall j\in\mathcal{N}_i, 1\le k\le K,\notag \\
    & \quad\quad \frac{\bm{Z}_{j,k}\bm{Z}_{j,k}^T}{m_{j,k}}=\bm{T}_{i,j,k}, \forall i\in\mathcal{V}, \forall j\in\mathcal{N}_i, 1\le k\le K,\notag 
\end{align}
In order to solve this optimization problem with consensus constraints, we utilize the augmented Lagrangian techniques, which can transform the optimization problem into an unconstrained one. 
By introducing the dual variables $\bm{U}_{i,j,k}, \bm{V}_{i,j,k}$, the problem can be transferred into an unconstrained one.
{Specifically, we utilize the similar dual-decomposition techniques as in the distributed alternating direction method of multipliers (ADMM) algorithm \cite{tian2020distributed}, and the updating rules can be derived as in Proposition \ref{p1}.
\newtheorem{proposition}{Proposition}
\begin{proposition}\label{p1}
    (Updating rules for i.i.d. data distribution.) Introduce $\bm{Y}_{i,j,k}\triangleq(\bm{U}_{i,j,k}+\bm{V}_{j,i,k})=2\bm{U}_{i,j,k}$, and $\bm{U}_{i,j,k}^{(0)}=\bm{V}_{j,i,k}^{(0)}=\bm{0}$. Define the step size for dual variables as $\rho$ and the penalty parameter as $\gamma$. Then the updating rules can be simplified into the following two steps, which can be conducted in parallel among the nodes:
\begin{subequations}
	\begin{equation}
		\bm{Y}_{i,j,k}^{(t)}=\bm{Y}_{i,j,k}^{(t-1)}+\rho[\frac{\bm{Z}_{i,k}^{(t-1)}{\bm{Z}_{i,k}^{(t-1)}}^T}{m_{i,k}}-\frac{\bm{Z}_{j,k}^{(t-1)}{\bm{Z}_{j,k}^{(t-1)}}^T}{m_{j,k}}],  \label{R1_2a}
	\end{equation}
	\begin{equation}
		\bm{Z}_i^{(t)} = \mathop{\arg\min}\mathcal{L}'(\bm{Y}_{i}^{(t)},\bm{Z}_i), \label{R1_2b}	
        \end{equation}
\end{subequations}
where  
\begin{align}
    \label{lagran1_2}
    & \mathcal{L}'(\bm{Y}_{i}^{(t)},\bm{Z}_i) = R_i^c(\bm{Z}_i) - R_i(\bm{Z}_i) + \\
    &\sum_{j\in\mathcal{N}_i}\sum_{k=1}^K\Big\{\text{tr}[{\bm{Y}_{i,j,k}^{(t)}}^T(\frac{\bm{Z}_{i,k}\bm{Z}_{i,k}^T}{m_{i,k}}-\frac{\bm{Z}_{j,k}^{(t-1)}{\bm{Z}_{j,k}^{(t-1)}}^T}{m_{j,k}})]\notag\\
    & + \gamma\big\|\frac{\bm{Z}_{i,k}\bm{Z}_{i,k}^T}{m_{i,k}}-\frac{\frac{\bm{Z}_{i,k}^{(t-1)}{\bm{Z}_{i,k}^{(t-1)}}^T}{m_{i,k}}+\frac{\bm{Z}_{j,k}^{(t-1)}{\bm{Z}_{j,k}^{(t-1)}}^T}{m_{j,k}}}{2}\big\|_F^2\Big\}.  \notag
\end{align}
\end{proposition}
\begin{proof}
    See Appendix A.
\end{proof}
}

Considering the fact that the optimization of $\bm{Z}_i$ is implemented by optimizing the neural network parameters $\bm{\theta}_i$, the updating rule (\ref{R1_2b}) is expressed as follows:
\begin{equation}
    \bm{\theta}_i^{(t)} = \mathop{\arg\min}_{\bm{\theta}_i}\mathcal{L}'(\bm{Y}_{i}^{(t)},f_i(\bm{\theta}_i, \bm{X}_i)). \label{R1_2c}	
\end{equation}
Due to the non-convexity of (\ref{lagran1_2}) and the inherent non-convexity of the neural network, a closed-form solution for (\ref{R1_2c}) cannot be obtained. However, it can be approximated through multiple steps of stochastic gradient descent, as is commonly done in neural networks. Therefore, the global structural representation learning algorithm for i.i.d. data distribution is summarized in Algorithm \ref{Alg1}. 
\begin{algorithm}
	\caption{\textbf{Distributed Representation Learning for i.i.d. Data Distribution}}\label{Alg1}
	\For{node $i=1,2,\dots, N$ in parallel}
	{
		\textbf{Initialize} the local parameters of encoder $\bm{\theta}_i$ and the dimension of the output feature $d$. $t=0, \bm{Y}_{i,j,k}^{(0)}=\bm{0}$ for all classes. \\
            %\textbf{Obtain} $\bm{Z}_{i,k}^{(0)}\in\mathbb{R}^{d\times m_{i,k}}$ for all classes with all of the training data samples.\\
            \textbf{Obtain} $\bm{V}_{i,k}^{(0)} = \bm{Z}_{i,k}^{(0)}{\bm{Z}_{i,k}^{(0)}}^T/m_{i,k}$ and \textbf{transmit} $\bm{V}_{i,k}^{(0)}$ to all neighboring nodes.\\
	}
	\While{not converge}
	{
		$t=t+1$\\
		\For{node $i=1,2,\dots, N$ in parallel}
		{
			\For{class $k=1,2,\dots, K$ in parallel}
                {
                    \textbf{Update} $\bm{Y}_{i,j,k}^{(t)}$ according to (\ref{R1_2a}).
                }
                
                \For{inner step $t'=1,\dots, T'$}
			{
				\textbf{Randomly select} a batch of data from local training data set $\bm{X}_s\in\mathcal{D}_i$.\\
				\textbf{Forward} and \textbf{Compute} the loss (\ref{lagran1_2}) with $\bm{Y}_{i}^{(t)}$, $\bm{V}_{i}^{(t-1)}$ and $\bm{V}_{j}^{(t-1)}, \forall j\in\mathcal{N}_i$.\\
				\textbf{Update} $\bm{\theta}_i$ with gradient descent.\\
			}
            }
            \For{node $i=1,2,\dots, N$ in parallel}
		{
			\textbf{Obtain} $\bm{Z}_{i,k}^{(t)}\in\mathbb{R}^{d\times m_{i,k}}$ for all classes with all of the training data samples.\\
                \textbf{Obtain} $\bm{V}_{i,k}^{(t)} = \bm{Z}_{i,k}^{(t)}{\bm{Z}_{i,k}^{(t)}}^T/m_{i,k}$ and \textbf{transmit} $\bm{V}_{i,k}^{(t)}$ to all neighboring nodes.\\
		}
	}
	\textbf{Output} the trained local encoders and the resultant representations $\bm{Z}$.
\end{algorithm}

\newtheorem{remark}{Remark}
{\begin{remark}
(Computational cost analysis.) In Algorithm \ref{Alg1}, in one iteration, the computation cost consists of three parts: the updating of $\bm{Y}_{i,j,k}^{(t)}$, the forward computation and backward propagation, the computation of $\bm{V}_{i,k}^{(t)}$. We define the number of model parameters in node $i$ as $P_i=|\bm{\theta}_i|$. Then the computational cost of the three steps are $\mathcal{O}(|\mathcal{N}_i|\cdot K\cdot d^2)$, $\mathcal{O}(T'P)$, $\mathcal{O}(d^2m_i)$ respectively. Take sum of them and we can obtain the computational cost in each iteration for node $i$ as $\mathcal{O}(|\mathcal{N}_i|Kd^2+T'P+d^2m_i)$. Typically we have $P_i\gg d^2$, so the neural network training plays the main role, among which the computation of $\log\det(\cdot)$ function costs the most. 

Compared to conventional algorithms such as Decentralized SGD (D-SGD) and its variants, the computational cost increases mainly with the computation of $\log\det(\cdot)$ function, and additional updating of $\bm{Y}_{i,j,k}^{(t)}$ and $\bm{V}_{i,k}^{(t)}$.
\end{remark}

\begin{remark}
(Communication cost analysis.) In each iteration, each agent transmits $\bm{V}_{i,k}^{(t)}\in\mathbb{R}^{d\times d}, \forall k=1,..., K$ to all its neighboring nodes. Thus the total communication cost is $\mathcal{O}(T\cdot|\mathcal{N}_i|\cdot K\cdot d^2)$. 

Compared with other algorithms requiring the exchange of model parameters of size $\mathcal{O}(P)$, the proposed Algorithm \ref{Alg1} only communicates second-order representation statistics with complexity $\mathcal{O}(Kd^2)$, which is independent of both the model size and the local data volume. In deep neural network architectures, this significantly reduces communication overhead with $Kd^2\ll P$ and makes the proposed method more scalable to large-scale and heterogeneous models.
\end{remark}
}

Theorem \ref{t1} in Section \ref{Sec5} guarantees the diverse and discriminative properties of the learned representations.
However, in non-i.i.d. data distribution, we do not have ${\bm{Z}_{i}\bm{Z}_{i}^T}/{m_{i}} = {\bm{Z}_{j}\bm{Z}_{j}^T}/{m_{j}}$ between different agents due to their different data distributions. 
As a result, the equality in (\ref{eqR1}) cannot be satisfied, and the resulting optimal solution does not possess the properties outlined in Theorem \ref{t1}.
In the following section, we will develop an algorithm to address this issue.

\section{GLOBAL REPRESENTATION LEARNING FOR Non-I.I.D. DATA DISTRIBUTION}\label{Sec4}
In non-i.i.d. conditions, the first term in (\ref{MCR2}) cannot be decomposed due to the fact that ${\bm{Z}_{i}\bm{Z}_{i}^T}/{m_{i}} \neq {\bm{Z}_{j}\bm{Z}_{j}^T}/{m_{j}}$. % This phenomenon is resulted from the label distribution skews among nodes. 
To address this issue, we propose clustering the agents over the network so that the data distributions between clusters are i.i.d.. These clusters and their respective agents can then be decomposed for parallel updating. The procedure is illustrated in Fig. \ref{fig_alg3}. 
In the first stage, we assume that the global information about the agents and their labels is known. 
To ensure i.i.d. data distributions among clusters, one node may belong to more than one cluster, as shown in Fig. \ref{fig_alg3}, making it difficult to conduct inter-cluster parallelization. To address this issue, we propose the virtual node replication, {whereby a heuristic clustering with Local Replication method as described in Algorithm \ref{Alg2}. For ease of representation, we define the set of labels in agent $i$ as $\mathcal{K}_i$ and assume that for any label $k\in\mathcal{K}$, there exists at least one agent $i\in\mathcal{V}$ such that $k\in\mathcal{K}_i$. Note that Algorithm~\ref{Alg2} prioritizes unclustered agents when constructing each cluster. Node replication is activated only when the remaining unclustered agents cannot cover all labels, which effectively limits redundant node usage.

\begin{algorithm}
	\caption{\textbf{Clustering with Local Replication}}\label{Alg2}
    \textbf{Input} the set of the labels $\mathcal{K}_i$ for all $i\in\mathcal{V}$.\\
    \textbf{Initialize} $s=1$, set of unclustered agents $\bar{\mathcal{V}}\leftarrow \mathcal{V}$. \\
	\While{$\bar{\mathcal{V}}\neq\emptyset$}
	{
		\textbf{Initialize} $\mathcal{V}^s=\emptyset$, $\mathcal{U}^s=\mathcal{K}$.\\
		\While{$\mathcal{U}^s\neq\emptyset$}
		{    
            \If{$\bar{\mathcal{V}}\neq\emptyset$}{
                \textbf{Choose} the agent $i^\star=\arg\max_{i\in\bar{\mathcal{V}}}|\mathcal{K}_i\cap \mathcal{U}^s|$ and \textbf{update} $\bar{\mathcal{V}}\leftarrow \bar{\mathcal{V}} \setminus \{i^\star\}$.\\
            }
            \Else{ 
                \textbf{Choose} the agent $i^\star=\arg\max_{i\in{\mathcal{V}}}|\mathcal{K}_i\cap \mathcal{U}^s|$\tcp*[l]{Replication is allowed}
            }
            \textbf{Update} $\mathcal{V}^s\leftarrow \mathcal{V}^s \cup \{i^\star\}$, $\mathcal{U}^s\leftarrow\mathcal{U}^s\setminus \mathcal{K}_{i^\star}$
        }
        $s=s+1$\\
	}
	\textbf{Output} $S=s-1$ clusters with node sets $\{\mathcal{V}^1,\ldots,\mathcal{V}^S\}$.
\end{algorithm}
}

Specifically, if agent $i$ exists in $1<S_i\le S$ clusters, we add $S_i$ virtual nodes in agent $i$ and these virtual nodes are initialized the same as agent $i$ with the same datasets and the same neural network model. Then, these $S_i$ virtual nodes, denoted by $\tilde{V}_i$, are assigned to the clusters for parallelized implementation.
The set of the agents together with the $\sum_{i\in\mathcal{V}}S_i$ virtual nodes is denoted by $\tilde{\mathcal{V}}$.
Correspondingly, the assigned (virtual) nodes' set for cluster $s$ is defined as $\tilde{\mathcal{V}}^s$.
To ensure the local consensus of the virtual nodes within one agent $i$, we add the consensus constraints of $\bm{Z}_{i'}\bm{Z}_{i'}^T$ the local virtual nodes $\tilde{V}_i$. With the same datasets in the node, this constraint can be easily satisfied by taking local average of the model parameters among $i'\in\tilde{V}_i$.

\begin{figure*}[t]
    \centering
    \includegraphics[width=1.0\textwidth]{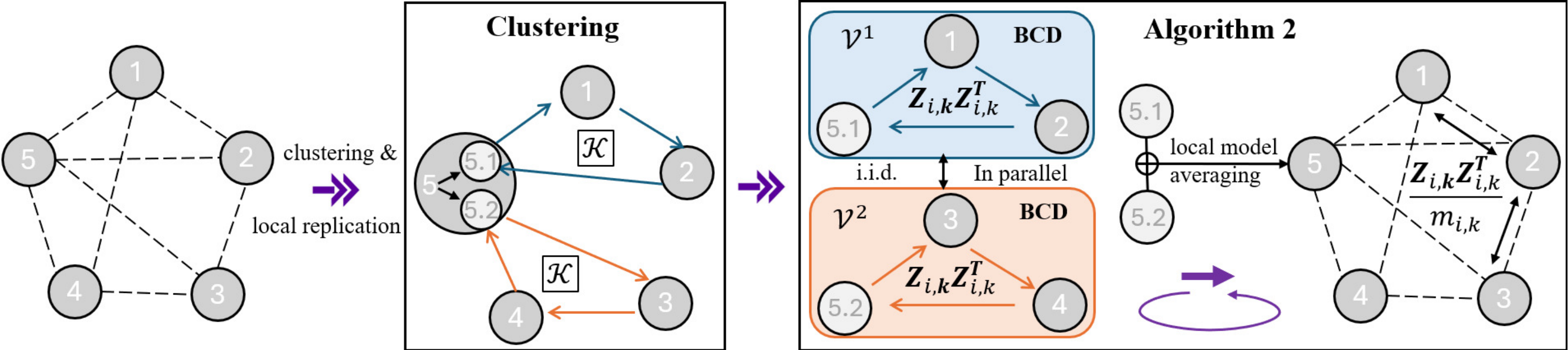}
    \caption{Illustration of the global representation learning algorithm under non-i.i.d. conditions: First, {$5$ agents are clustered into $2$ clusters. Through local replication in agent $5$, }the data distributions within each cluster become i.i.d. Then, in Algorithm \ref{Alg3}, the (virtual) nodes across different clusters can be updated in parallel, while the agents within each cluster update their models sequentially.}
    \label{fig_alg3}
    \vspace{-0.5cm}
\end{figure*}

After clustering and virtual replication, we could obtain the sets of nodes denoted by $\mathcal{V}^{s}$, corresponding to the $s$-th cluster.
Within $s$-th cluster, the set of nodes can be denoted by $\mathcal{V}^{s}$. 
The number of nodes in $\mathcal{V}^{s}$ is denoted by $N^s$ and the total number of data samples in the cluster is denoted by $m^{s}$. 
We have $\mathcal{V}^{1}\cup\cdots\cup\mathcal{V}^{S}=\mathcal{V}$ and $\cup\{\mathcal{K}_i, \forall i\in\mathcal{V}^{s}\}=\mathcal{K}, \forall s \in \{1,\dots, S\}$, where $\cup\{\cdot\}$ denotes the union of all entries in the sets.
Then, (\ref{eqR1}) can be reformulated as:
\begin{align}
    \label{eqR2_2}
    R(\bm{Z}, \epsilon) & = \frac{1}{2}\log\det(\bm{I}+\frac{d}{m\epsilon^2}\bm{Z}\bm{Z}^T)\\
    % & = \frac{1}{2}\log\det\Big[\bm{I}+\frac{d}{m\epsilon^2}\sum_{s=1}^S\sum_{i\in\tilde{\mathcal{V}}^s}\frac{1}{S_i}\bm{Z}_{i}\bm{Z}_{i}^T\Big] \notag \\
    % & = \frac{1}{2}\log\det\Big[\bm{I}+\frac{d}{(\sum_{s=1}^Sm^s)\epsilon^2}\sum_{s=1}^S\sum_{i\in\tilde{\mathcal{V}}^s}\frac{1}{S_i}\bm{Z}_{i}\bm{Z}_{i}^T\Big] \notag \\
    & = \frac{1}{2}\log\det\Big[\sum_{s=1}^S\frac{m^s}{m}(\bm{I}+\frac{d}{m^s\epsilon^2}\sum_{i\in\tilde{\mathcal{V}}^s}\frac{1}{S_i}\bm{Z}_{i}\bm{Z}_{i}^T\Big] \notag \\
    & \ge \sum_{s=1}^S\frac{m^s}{2m}\log\det(\bm{I}+\frac{d}{m^s\epsilon^2}\sum_{i\in\tilde{\mathcal{V}}^s}\frac{1}{S_i}\bm{Z}_{i}\bm{Z}_{i}^T), \notag
\end{align}
where the equality holds if $\frac{1}{m^s}\sum_{i\in\tilde{\mathcal{V}}^s}\frac{1}{S_i}\bm{Z}_{i}\bm{Z}_{i}^T$ is equal among different clusters. $m^s$ is defined as $m^s=\sum_{i\in\tilde{\mathcal{V}}^s}\frac{m_i}{S_i}$, making sure $\sum_{s=1}^Sm^s=m$.
The first equality follows from the consensus constraints of $\bm{Z}_{i'}\bm{Z}_{i'}^T, i'\in\tilde{V}_i$, which can be ensured by taking the average of model parameters among the virtual nodes within each agent.

In the non-i.i.d. condition, we still can make consensus constraints on the variance of the subspaces w.r.t. each class.
We define $\mathcal{V}_k$ as the set of original agents which have the label $k$, and $\tilde{\mathcal{V}}_k$ as the set of nodes with the virtual nodes included. 
We add the constraints ${\bm{Z}_{i,k}\bm{Z}_{i,k}^T}/{m_{i,k}} = {\bm{Z}_{j,k}\bm{Z}_{j,k}^T}/{m_{j,k}}, i,j\in\tilde{\mathcal{V}}_k, \forall k \in \{1,\dots, K\}$. Then, the second term in (\ref{MCR2}) can be reformulated as:
\begin{align}
    \label{eqRc2}
    & R^c(\bm{Z}, \epsilon|\bm{\Pi}) = \sum_{k=1}^K\frac{tr(\bm{\Pi}_k)}{2m}\log\det(\bm{I}+\frac{d}{tr(\bm{\Pi}_k)\epsilon^2}\bm{Z}\bm{\Pi}_k\bm{Z}^T) \notag \\
    % & = \sum_{k=1}^K\frac{m_{k}}{2m}\log\det(\bm{I}+\frac{d}{m_k\epsilon^2}\sum_{i\in\mathcal{V}_k}\bm{Z}_{i,k}\bm{Z}_{i,k}^T) \\
    %& = \sum_{k=1}^K\frac{m_{k}}{2m}\log\det(\bm{I}+\frac{d}{m_k\epsilon^2}\sum_{i\in\tilde{\mathcal{V}}_k}\frac{1}{S_i}\bm{Z}_{i,k}\bm{Z}_{i,k}^T) \notag\\
    & =\sum_{k=1}^K\frac{m_{k}}{2m}\log\det(\bm{I}+\frac{d}{m_k\epsilon^2} \sum_{s=1}^S\sum_{i\in\tilde{\mathcal{V}}_k^s}\frac{1}{S_i}\bm{Z}_{i,k}\bm{Z}_{i,k}^T) \\
    & = \sum_{s=1}^S\sum_{k=1}^K\frac{m_{k}^s}{2mS_i}\log\det(\bm{I}+\frac{d}{m_{k}^s\epsilon^2}\sum_{i\in\tilde{\mathcal{V}}_k^s}\frac{1}{S_i}\bm{Z}_{i,k}\bm{Z}_{i,k}^T), \notag
\end{align}
where the last equality follows from the consensus constraints.
By defining:
$$
R^c_s(\bm{Z}^s) = \sum_{k=1}^K\frac{m_{k}^s}{2mS_i}\log\det(\bm{I}+\frac{d}{m_{k}^s\epsilon^2}\sum_{i\in\tilde{\mathcal{V}}_k^s}\frac{1}{S_i}\bm{Z}_{i,k}\bm{Z}_{i,k}^T),
$$
the global optimization problem can be reformulated as:
\begin{align}
    \label{opt2}
    & \max_{\bm{Z}}\sum_{s=1}^S\Big[\frac{m^s}{2m}\log\det(\bm{I}+\frac{d}{m^s\epsilon^2}\sum_{i\in\tilde{\mathcal{V}}^s}\frac{1}{S_i}\bm{Z}_{i}\bm{Z}_{i}^T)-R^c_s(\bm{Z}^s)\Big] \notag \\
    & \text{s.t.}\quad \|\bm{Z}_{i,k}\|_F^2 = m_{i,k}, \forall i\in\tilde{\mathcal{V}}, 1\le k\le K,\\
    & \quad\quad \frac{\bm{Z}_{i,k}\bm{Z}_{i,k}^T}{m_{i,k}}=\frac{\bm{Z}_{j,k}\bm{Z}_{j,k}^T}{m_{j,k}}, \forall i,j\in\tilde{\mathcal{V}}_k, 1\le k\le K,\notag \\
    & \quad\quad \bm{Z}_{i'}\bm{Z}_{i'}^T =\bm{Z}_{i}\bm{Z}_{i}^T, \forall i'\in\tilde{V}_i, i\in\mathcal{V},\notag
\end{align}

In (\ref{opt2}), the first constraint can be satisfied by adding the normalization function on the output layer, as illustrated in the previous section. The third constraint can be satisfied by model parameter averaging among the local virtual nodes in each agent. 
As discussed above, this is due to the fact that these local virtual nodes share the same dataset and neural network architecture. By performing local averaging, consensus on the representations can be easily reached. 
{To solve the formulated optimization problem \eqref{opt2}, the updating rules can be derived as in Proposition \ref{p2}.
\begin{proposition}\label{p2}
    (Updating rules for non-i.i.d. data distribution.) Introduce $\bm{Y}_{i,j,k}$ and initialize it as $\bm{Y}_{i,j,k}^{(0)}=\bm{0}$. Define the step size for dual variables as $\rho$ and the penalty parameter as $\gamma$.  Then the updating rules can be simplified into the following two steps, which can be conducted in parallel among the clusters.
\begin{subequations}
	\begin{equation}
		\bm{Y}_{i,j,k}^{(t)}=\bm{Y}_{i,j,k}^{(t-1)}+\rho[\frac{\bm{Z}_{i,k}^{(t-1)}{\bm{Z}_{i,k}^{(t-1)}}^T}{m_{i,k}}-\frac{\bm{Z}_{j,k}^{(t-1)}{\bm{Z}_{j,k}^{(t-1)}}^T}{m_{j,k}}],  \label{R2_1a}
	\end{equation}
	\begin{align}
    (\bm{Z}_{i})^{(t)} = \mathop{\arg\min}\mathcal{L}^s(&\bm{Y}^{(t)},(\bm{Z}_{i_1})^{(t)},\cdots,(\bm{Z}_{i}),\cdots,\notag\\
    &(\bm{Z}_{N^s})^{(t-1)}), \label{R2_3}	
    \end{align}
\end{subequations}
 Within each cluster, the update of \eqref{R2_3} is conducted sequentially and  
\begin{align}
&\mathcal{L}^s=-\frac{m^s}{2m}\log\det(\bm{I}+\frac{d}{m^s\epsilon^2}\sum_{i\in\tilde{\mathcal{V}}^s}\frac{1}{S_i}\bm{Z}_{i}\bm{Z}_{i}^T)+ R_s^c(\bm{Z}^s)\notag\\
& + \sum_{i\in\tilde{\mathcal{V}}^{s}}\sum_{j\in\tilde{\mathcal{V}}_k/\{i\}}\sum_{k\in\mathcal{K}_i}  \Big\{ \text{tr}[{\bm{Y}_{i,j,k}^{(t)}}^T(\frac{\bm{Z}_{i,k}\bm{Z}_{i,k}^T}{m_{i,k}}-\frac{\bm{Z}_{j,k}^{(t-1)}{\bm{Z}_{j,k}^{(t-1)}}^T}{m_{j,k}})] \notag\\
& + \gamma\big\|\frac{\bm{Z}_{i,k}\bm{Z}_{i,k}^T}{m_{i,k}}-\frac{\frac{\bm{Z}_{i,k}^{(t-1)}{\bm{Z}_{i,k}^{(t-1)}}^T}{m_{i,k}}+\frac{\bm{Z}_{j,k}^{(t-1)}{\bm{Z}_{j,k}^{(t-1)}}^T}{m_{j,k}}}{2}\big\|_F^2\Big\}. \label{R2_4}
\end{align}
\end{proposition}
\begin{proof}
    See Appendix B.
\end{proof}

In \eqref{R2_3}, the updating of primal variables are decoupled into distributed nodes through the blocked coordinate descent procedure (BCD) \cite{beck2013convergence, hong2017iteration}.
Specifically, in each coordinate descent step, the node receives the updated results from previous ones and then updates its local variables (model parameters).
Under this sequential way, the problem can be solved as shown in the right side of Fig. \ref{fig_alg3}.}

However, experiments with the local loss function (\ref{R2_4}) reveal that it is challenging for two nodes within the same cluster to achieve the same subspace for class k, if they both have the label $k$. 
This phenomenon results from the iterative nature of the process. In the experiments, the iterative updating process clearly shows that the expansion term (\ref{eqR2_2}) decreases quickly at the beginning, while the compression term (\ref{eqRc2}) decreases much more slowly and later, as depicted in Fig. \ref{fig_exp1} in Section \ref{Sec6}. This suggests that the expansion term takes priority during the update. In the local loss function (\ref{R2_4}), when two nodes share the same class $k$, the initial expansion term expands the space first, thus making it difficult for the two nodes to converge to the same subspace. 
To address this issue, we propose replacing the information on the common class $k$ from other nodes with local one, using the constraints ${\bm{Z}_{i,k}\bm{Z}_{i,k}^T}/{m_{i,k}} = {\bm{Z}_{j,k}\bm{Z}_{j,k}^T}/{m_{j,k}}, i,j\in\tilde{\mathcal{V}}_k$ for all $k=1,\dots, K$.
Specifically, if node $i$ and $j$ have the same label $k$ in cluster $s$, then in both (\ref{eqR2_2}) and (\ref{eqRc2}), $\frac{1}{S_j}\bm{Z}_{j,k}\bm{Z}_{j,k}^T$ is replaced by $\frac{m_{j,k}}{S_jm_{i,k}}\bm{Z}_{i,k}\bm{Z}_{i,k}^T$ in the updating of node $i$.

Thus, the global structural representation learning algorithm for non-i.i.d. data distributions can be summarized in Algorithm \ref{Alg3}. 
\begin{algorithm}
	\caption{\textbf{Distributed Representation Learning for non-i.i.d. Data Distribution}}\label{Alg3}
	\For{node $i\in\tilde{\mathcal{V}}$ in parallel}
	{
		\textbf{Initialize} the local parameters of encoder $\bm{\theta}_i$ and the dimension of the output feature $d$. $t=0, \bm{Y}_{i,j,k}^{(0)}=\bm{0}$ for $k\in\mathcal{K}_i$. \\
            \textbf{Obtain} $\bm{Z}_{i,k}^{(0)}\in\mathbb{R}^{d\times m_{i,k}}$ for $k\in\mathcal{K}_i$ with all of the training data samples.\\
            \textbf{Obtain} $\bm{V}_{i,k}^{(0)} = \bm{Z}_{i,k}^{(0)}{\bm{Z}_{i,k}^{(0)}}^T/m_{i,k}$ for $k\in\mathcal{K}_i$ and \textbf{transmit} $\bm{V}_{i,k}^{(0)}$ to the nodes with label $k$.\\
            \textbf{Transmit} $\bm{Z}_{i}^{(0)}{\bm{Z}_{i}^{(0)}}^T/S_i$ and corresponding $m_i/S_i$ to the nodes in the same cluster.\\
	}
	\While{not converge}
	{
		$t=t+1$\\
		\For{node $i\in\tilde{\mathcal{V}}$ in parallel}
		{
                \For{class $k=1,2,\dots, K$ in parallel}
		      {
			     \textbf{Update} $\bm{Y}_{i,j,k}^{(t)}$ according to (\ref{R2_1a}).
                }
            }
            
            \For{clusters $s=1,\dots, S$ in parallel}{
                \For{node $i\in\tilde{\mathcal{V}}^s$ in sequential}{
                    \For{inner step $t'=1,\dots, T'$}
    			{
    				\textbf{Update} $\bm{\theta}_i$ with gradient descent.\\
    			}
                    \textbf{Obtain} $\bm{Z}_{i,k}^{(t)}\in\mathbb{R}^{d\times m_{i,k}}$ for $k\in\mathcal{K}_i$ with all of the training data samples.\\
                    \textbf{Obtain} $\bm{Z}_{i}^{(t)}{\bm{Z}_{i}^{(t)}}^T/S_i$ and \textbf{transmit} to the latter nodes within the cluster.\\
                }
            }
            \For{node $i\in\mathcal{V}$ in parallel}{
                \If{$S_i>1$}{
                    Take average of the model parameters $\bm{\theta}_{i'}$ among the virtual nodes $i'\in\tilde{\mathcal{V}}_i$.\\
                    Assign the averaged model parameters $\overline{\bm{\theta}}_{i'}$ into the virtual nodes $i'\in\tilde{\mathcal{V}}_i$.\\
                }
            
            }
            \For{node $i\in\tilde{\mathcal{V}}$ in parallel}
		{
                    \textbf{Obtain} $\bm{V}_{i,k}^{(t)} = \bm{Z}_{i,k}^{(t)}{\bm{Z}_{i,k}^{(t)}}^T/m_{i,k}$ for $k\in\mathcal{K}_i$ and \textbf{transmit} $\bm{V}_{i,k}^{(t)}$ to the nodes with label $k$.\\
                }
	}
	\textbf{Output} the trained local encoders.
\end{algorithm}

{
\begin{remark}(Computational cost analysis.) In Algorithm \ref{Alg3}, in one iteration, the computation cost is consist of four parts: the updating of $\bm{Y}_{i,j,k}^{(t)}$, the forward computation and backward propagation, the computation of $\bm{V}_{i,k}^{(t)}$, the local averaging when $S_i>1$. We define the number of model parameters in node $i$ as $P_i=|\bm{\theta}_i|$. Then the computational cost of the fours steps are $\mathcal{O}(|\mathcal{N}_i^{(k)}|\cdot K_i\cdot d^2)$, $\mathcal{O}(T'P)$, $\mathcal{O}(d^2m_i)$, $\mathcal{O}(S_iP_i)$ respectively. Take sum of them and we can obtain the computational cost in each iteration for node $i$ as $\mathcal{O}(|\mathcal{N}_i|Kd^2+T'P+d^2m_i+S_iP_i)$. Typically we have $P_i\gg d^2$ and $T'\gg S_i$, so the training of the neural network takes the main role, among which the computation of $\log\det(\cdot)$ function costs the most. 

Compared to conventional algorithms such as D-SGD and its variants, the computational cost increases with the computation of $\log\det(\cdot)$ function, additional updating of $\bm{Y}_{i,j,k}^{(t)}$, $\bm{V}_{i,k}^{(t)}$ and the local averaging.
\end{remark}

\begin{remark}
(Communication cost analysis.) In each iteration, there are two different kinds of communication. The first one is the communication within each cluster as in Line 16 of Algorithm \ref{Alg3}, which is $\mathcal{O}(|\tilde{\mathcal{V}}^s|\cdot d^2)$. The second one is the label-wise communication as in Line 22 of Algorithm \ref{Alg3}, which is $\mathcal{O}(|\mathcal{K}_i|\cdot d^2)$. The local model averaging does not introduce any communication cost. Then for each (virtual) node, the communication cost is $\mathcal{O}(d^2 +|\mathcal{K}_i|\cdot d^2)=\mathcal{O}(|\mathcal{K}_i|\cdot d^2)$.

For non-i.i.d. data distributions, the proposed algorithm further reduces communication overhead by only exchanging class-conditional second-order representation statistics. The per-round communication complexity scales as $\mathcal{O}(|\mathcal{K}_i|\cdot d^2)$, which is significantly smaller than parameter-based methods with complexity $\mathcal{O}(P)$, especially when each node contains only a small subset of classes. 
\end{remark}}

%It is important to note that both the computational and communication costs may increase somewhat compared to the i.i.d. conditions. 
Ideally, to achieve the most efficient algorithm, the value of $S$ should be as large as possible to maximize parallelization, while the number of agents, $N^s$, should be minimized to reduce the sequential updating time in each cluster. However, increasing the value of $S$ may require more local replication, which leads to higher local computation costs and greater demands on local capacities. Therefore, there is a trade-off between the number of clusters and the capabilities of the local agents. Additionally, considering potential multi-hop communications, the pre-training cluster and replication algorithm should also take into account the network topology to minimize communication costs. However, optimizing this aspect is not the primary focus of this work and will be addressed in future research.

{
\begin{remark}
While the proposed framework demonstrates strong scalability and deployment potential, several practical considerations warrant further investigation. For example, the quadratic dependence on the representation dimension $d$ may become non-negligible if extremely high-dimensional embeddings are used. Exploring low-rank approximations or adaptive dimension selection could further enhance scalability. Additionally, {as discussed in Remarks 2 and 4, the communication cost increases with the dimension of embeddings and the number of classes. In real-world large-scale distributed systems, some compression techniques can be applied to further reduce the communication cost.}
\end{remark}}

\section{Theoretical Analysis}\label{Sec5}
For the i.i.d. conditions, we consider Assumption \ref{as1} on the number of data samples as follows.
\newtheorem{assumption}{Assumption}
\begin{assumption}\label{as1}
    For the number of data samples in each node $i$, we have $\frac{m_{i,k}}{m_i}=\frac{m_{j,k}}{m_j}$ for any $k\in\mathcal{K}$ and $j\in\mathcal{N}_i$.
\end{assumption}

Note that in practical implementations, Assumption \ref{as1} can be satisfied easily by data augmentation techniques in each agent, to ensure the same proportion.
Given Assumption \ref{as1}, we could have the following Theorem \ref{t1} for the optimal solution of the optimization problem (\ref{opt1}) in i.i.d. conditions.
\newtheorem{theorem}{Theorem}
\begin{theorem}\label{t1}
    (Property Guarantees) Under the large ambient dimension $d\ge\sum_{k=1}^Kd_k$, high coding precision $\epsilon^4<\min_{k\in\mathcal{K}}\{\frac{m_k}{m}\cdot\frac{d^2}{d_k^2}\}$ and Assumption \ref{as1}, if we constrain $\text{rank}(\bm{Z}_k)\le d_k$, then the optimal solution $\bm{Z}^*$ of (\ref{opt1}) among all nodes has the following properties.
    \begin{enumerate}
        \item{(Between-class discriminative within node) $(\bm{Z}_{i,k1}^*)^T\bm{Z}_{i,k2}^*=\bm{0}$ for different classes $k1$ and $k2$ in each node $i\in\mathcal{V}$, i.e., $\bm{Z}_{i,k1}^*$ and $\bm{Z}_{i,k2}^*$ lie in orthogonal subspaces.}
        \item{(Between-class discriminative among nodes) $(\bm{Z}_{i,k1}^*)^T\bm{Z}_{j,k2}^*=\bm{0}$ for different classes $k1$ and $k2$ between any two nodes $i$ and $j$, i.e., $\bm{Z}_{i,k1}^*$ and $\bm{Z}_{j,k2}^*$ lie in orthogonal subspaces.}
        \item{(Within-class diverse among nodes) For the output features in the same class among all nodes, each subspace achieves its maximal dimension where $\text{rank}(\bm{Z}_k^*)=d_k$. And the largest $d_k-1$ singular values of $\bm{Z}_k^*$ are equal.}
    \end{enumerate}
\end{theorem}
\begin{proof}
    See Appendix C.
\end{proof}

Theorem \ref{t1} guarantees the properties of the optimal solution of (\ref{opt1}). After the collaborative training, the embeddings of the distributed data are promoted into multiple independent subspaces, where the features are distributed isotropically within the subspaces.
To provide the convergence guarantee of Algorithm \ref{Alg1}, we additionally make the following assumptions for the local learning steps.

\begin{assumption}\label{as2}
    (Smoothness) In the first iteration, given the dual variable $\bm{Y}_i^{(0)}$, the local objective functions $\mathcal{L}_i(\bm{Y}_i^{(0)}, \bm{\theta}_i)$ are $L_0$-smooth for $\bm{\theta}_i$ of each node $i\in\mathcal{V}$.
\end{assumption}

\begin{assumption}\label{as3}
    (Local Gradient) The local gradient estimator is unbiased in the $t'$-th inner gradient descent step. Furthermore, there exists scalar $\kappa>0$ such that the variance of the local gradient estimator is bounded by $\kappa$.
\end{assumption}

Consequently, the convergence guarantees of Algorithm \ref{Alg1} can be provided under Theorem \ref{t2}. Theorem \ref{t2} characterizes the convergence of the model parameters given a small penalty parameter $\gamma$ and updating step size $\rho$ of the dual variable. However, it should be noted that Theorem \ref{t2} does not necessarily guarantee the algorithm converges to the optimal solution satisfying K.K.T. conditions. This is mainly due to the non-convexity of the neural network and the SGD-based updating rule of the primal variable.
\begin{theorem}\label{t2}
    (Convergence Guarantees) Define $L=max_{t\in[T]}L_t$. Provided that the learning rate $\eta<\frac{1}{4T'L}$ and $(4\rho+\gamma)$ is small enough such that there exist some scalar $c'$ where $(4\rho+\gamma)\le c'\eta^2$, under Assumption \ref{as2} and \ref{as3}, the iterates of Algorithm \ref{Alg1} satisfy the following inequality:
    $$
    \min_{t\in[T]} \mathbb{E}[\|\nabla \mathcal{L}(\bm{Y}^{(t+1)},\theta^{(t)})\|_2^2]\le \frac{\mathcal{L}_0-\mathcal{L}_*}{cTT'\eta}+\frac{\eta NE}{c},
    $$
    where $0<c<\frac{1}{2}-8\eta^2T'^2L^2$, $E=\frac{T'L}{2}(1+4\eta T'L)\kappa^2+c'K|\mathcal{N}_i|$. If we properly choose $\eta=\frac{c''}{\sqrt{TT'}}$ such that $\eta<\frac{1}{4T'L}$, then $L$ is a limited scalar and the convergence rate for Algorithm \ref{Alg1} is $\mathcal{O}(\frac{1}{\sqrt{TT'}})$.
\end{theorem}
\begin{proof}
    See Appendix D.
\end{proof}

{
\begin{remark}
    The conditions in Theorem \ref{t2} are standard in distributed learning, which can be satisfied in practice through appropriate hyperparameter selection. In our implementation, mini-batch training with Adam provides unbiased stochastic gradients at the sampling level with bounded variance due to the finite dataset, while the objective function is locally smooth for neural network training. The hyperparameters are chosen within stable ranges ensuring that the required bounds on $\eta$ and $(4\rho+\eta)$ are met. The observed stable convergence behavior in our experiments further confirms that these theoretical conditions can be satisfied in practice.
\end{remark}}

For the non-i.i.d. conditions, we make the following Assumption \ref{as4} on the number of data samples among different clusters as follows. 
In the $s$-th cluster, the number of samples corresponding to class $k$ is denoted by $m^s_k$. Note that $m^s_k$ may consist of data samples from multiple nodes.

\begin{assumption}\label{as4}
    Among any different clusters $s$ and $s'$, the proportion corresponding to the $k$-th class satisfies $\frac{m_{k}^{s}}{m^{s}}=\frac{m_{k}^{s'}}{m^{s'}}$ for any $k\in\mathcal{K}$, where ${m_{k}^{s}}=\sum_{i\in\tilde{\mathcal{V}}^s_k}\frac{m_{i,k}}{S_i}$ and $m^{s}=\sum_{k\in\mathcal{K}}{m_{k}^{s}}$.
\end{assumption}

Note that Assumption \ref{as4} can also be satisfied easily by data augmentation techniques in each node.
Given Assumption \ref{as4}, we could have the following Theorem \ref{t3} for the optimization problem (\ref{opt2}) in non-i.i.d. conditions.
\begin{theorem}\label{t3}
    
    Under the large ambient dimension $d\ge\sum_{k=1}^Kd_k$, high coding precision $\epsilon^4<\min_{k\in\mathcal{K}}\{\frac{m_k}{m}\cdot\frac{d^2}{d_k^2}\}$ and Assumption \ref{as2}, if we further constrain $\text{rank}(\bm{Z}_k)\le d_k$, then the optimal solution $\bm{Z}^*$ of (\ref{opt2}) among all nodes has the same properties as those in Theorem \ref{t1}.
    
\end{theorem}
\begin{proof}
    Under Assumption \ref{as2} and given the constraints in (\ref{opt2}), it is easily to derive that $\frac{1}{m_s}\sum_{k=1}^K\bm{Z}_{i_k^s,k}\bm{Z}_{i_k^s,k}^T$ is equal among different clusters for any $k$. Then the equality in (\ref{eqR2_2}) holds and the optimization problem (\ref{opt2}) is equivalent to the original one. Then Theorem \ref{t2} can be proved.
\end{proof}

\section{Numerical Experiments}\label{Sec6}
In this section, we evaluate the performance of the proposed algorithms under i.i.d. and non-i.i.d. settings. Specifically, in each setting, we consider two datasets MNIST and CIFAR-10, which are typical 10-classification problems. 
For the MNIST dataset, the model architectures among nodes are identical. 
On the other hand, in the CIFAR-10 dataset, the models among nodes are different, to better show the proposed representation learning algorithms do not require the same architecture.
{
In each experiment, we present the loss curves to illustrate the convergence behavior of the proposed algorithms, providing complementary empirical support for Theorem \ref{t2}. To further validate the properties established in Theorems \ref{t1} and \ref{t3}, we analyze the learned representations at convergence and compare them in terms of cosine similarity and singular value spectra. Finally, in the last subsection, we evaluate several performance metrics of the proposed algorithms against pre-trained representation learning models.

We consider $7$ different baseline methods for comparison: three conventional methods including Centralized MCR$^2$ with centralized datasets, D-SGD, Independent MCR$^2$ where each node independently updates its local model with MCR$^2$ principle; four state-of-the-art federated representation learning algorithms including FedU$^2$ \cite{liao2024rethinking}, FedSimCLR \cite{chen2020simclr}, FedU \cite{zhuang2021collaborative} and Orchestra \cite{lubana2022orchestra}. In these SOTA federated representation learning algorithms, we assume there is one central server in order to show the learned representations' properties. If not specified, the experimental settings of these baseline methods are aligned with the original papers.}

\subsection{Experiments with i.i.d. Setting}
In this subsection, we show the convergence curves of the proposed algorithm and properties of the learned representations in i.i.d. data distributions. 
The datasets are evenly split among the nodes.
The dimension of the representation is set to $128$.
The communication network topology is randomly generated using the Erdos\_Renyi graph model, with the number of nodes $N=10$ and a connectivity probability $p=0.5$. 

\begin{figure}[!htp]
    \centering
    \includegraphics[width=0.5\textwidth]{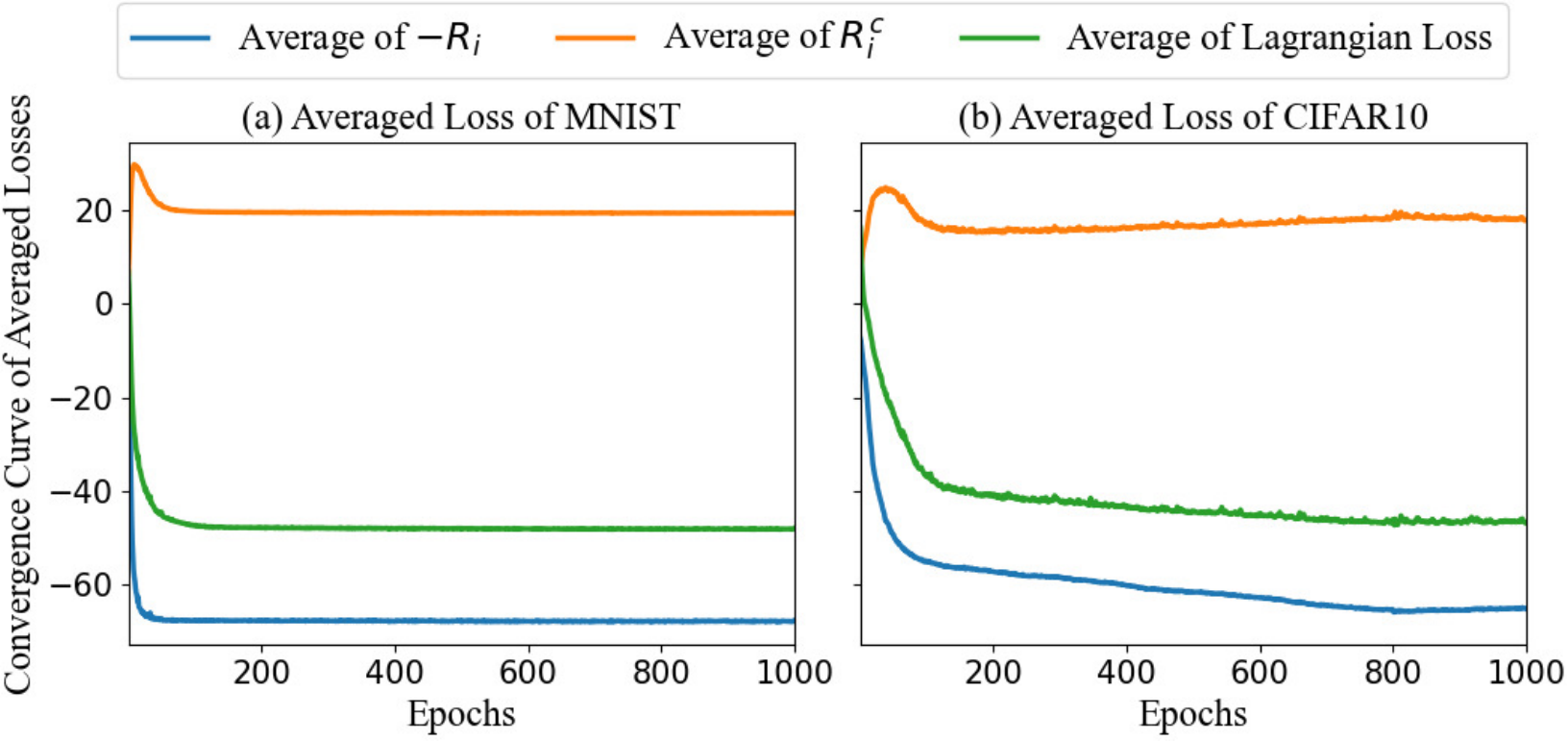}
    \caption{Convergence curves of the averaged loss.}
    \label{fig_exp1}
    \vspace{-0.3cm}
\end{figure}

\textbf{MNIST}. 
We first consider the same model architecture among nodes and the MNIST dataset.
The neural network in each node consists of a $4$ convolutional layers, followed by one flattened layer and one linear layer to map the output to the size of $128$. The normalization operation is added after the output to satisfy the first constraint in (\ref{opt1_2}).
We set the initial learning rate to $0.1$ with a weight decay of $10^{-5}$ and use the Adam optimizer. The mini-batch size is set to $1000$ and the precision parameter $\epsilon^2=0.5$.
In each iteration, to approximate the optimal solution of (\ref{R1_2c}), each node conducts $5$ local epochs. The step size for $\bm{Y}$'s update is set to $\rho=0.1$ and the penalty parameter $\gamma$ is set to $1.0$. 
The convergence curves of the averaged training loss for $1000$ iterations are shown in Fig. \ref{fig_exp1}(a), where the orange and blue curves represent the first and second term in (\ref{opt1_2}). The green curve shows the reduction of the total loss in each node.

It can be observed in Fig. \ref{fig_exp1}(a) that after an initial stage, $-R_i$ and $R_i^c$ both decrease, indicating that in each node, the learned features $\bm{Z}_i$ are expanding as a whole while each class $Z_{i,k}$ is being compressed. Meanwhile, the total loss of the augmented Lagrangian function decreases with the additional two terms in (\ref{lagran1_2}), which describes the correlation among the representations from different nodes. 

Then we plot the cosine similarity of the learned representations and compare the results with the baseline methods. For D-SGD, we take the outputs of the second last layer as the learned representations.
{For the SOTA methods based on federated learning framework, we use the global model in the server as the encoder to output the representations. The experimental settings are same as those in the original papers.
The results can be shown in Fig. \ref{fig_exp2}, visualizing the pairwise cosine similarity matrix of the learned representations, where samples are sorted by class label. Block structures along the diagonal indicate intra-class similarity, and black lines denote class boundaries.}

\begin{figure*}[htbp]
\centering
\subfigure[Proposed Algorithm 1]{
\begin{minipage}[t]{0.25\textwidth}
\centering
\includegraphics[width=\textwidth]{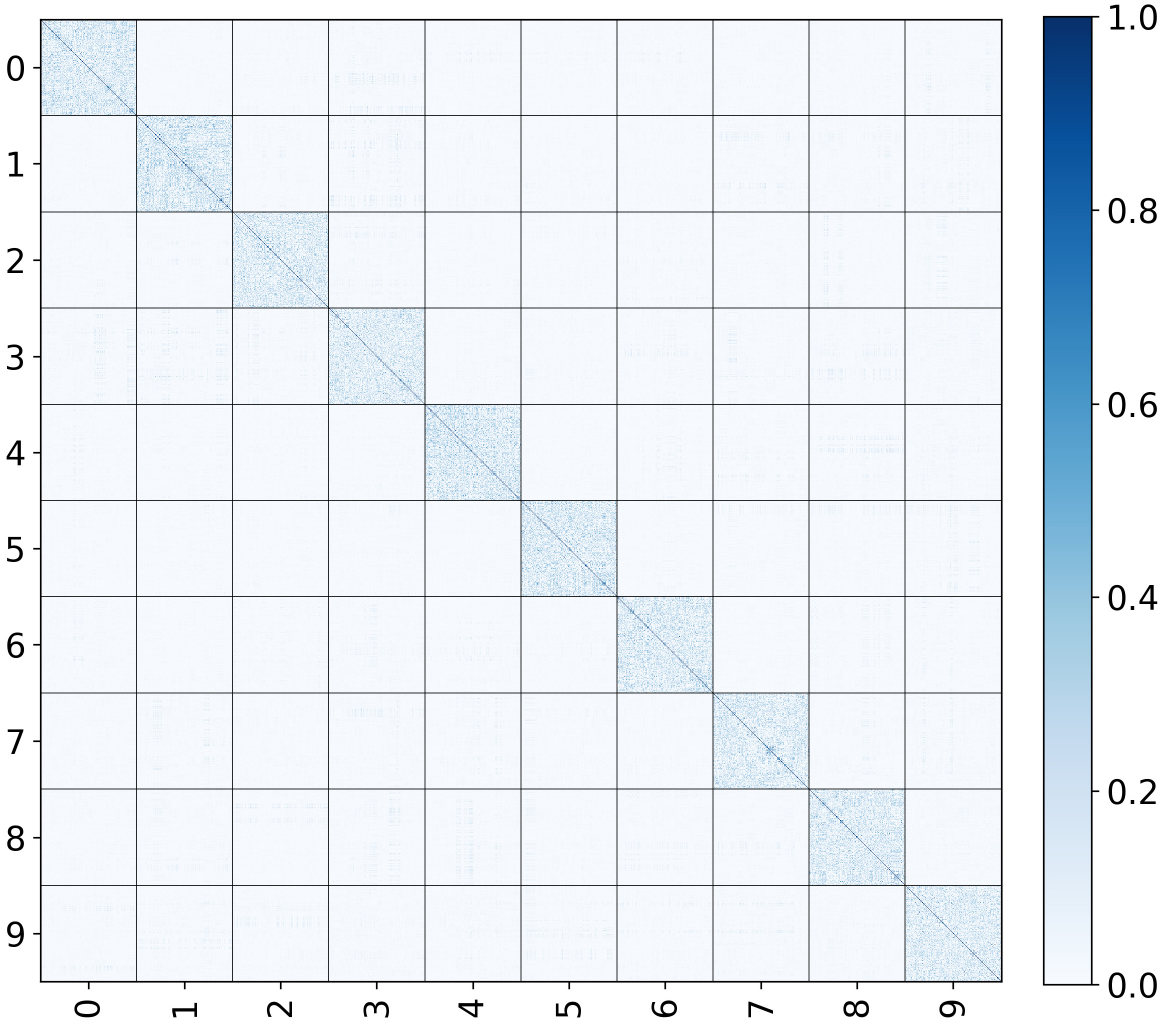}
%\caption{fig1}
\end{minipage}%
}%
\subfigure[Centralized MCR$^2$]{
\begin{minipage}[t]{0.25\textwidth}
\centering
\includegraphics[width=\textwidth]{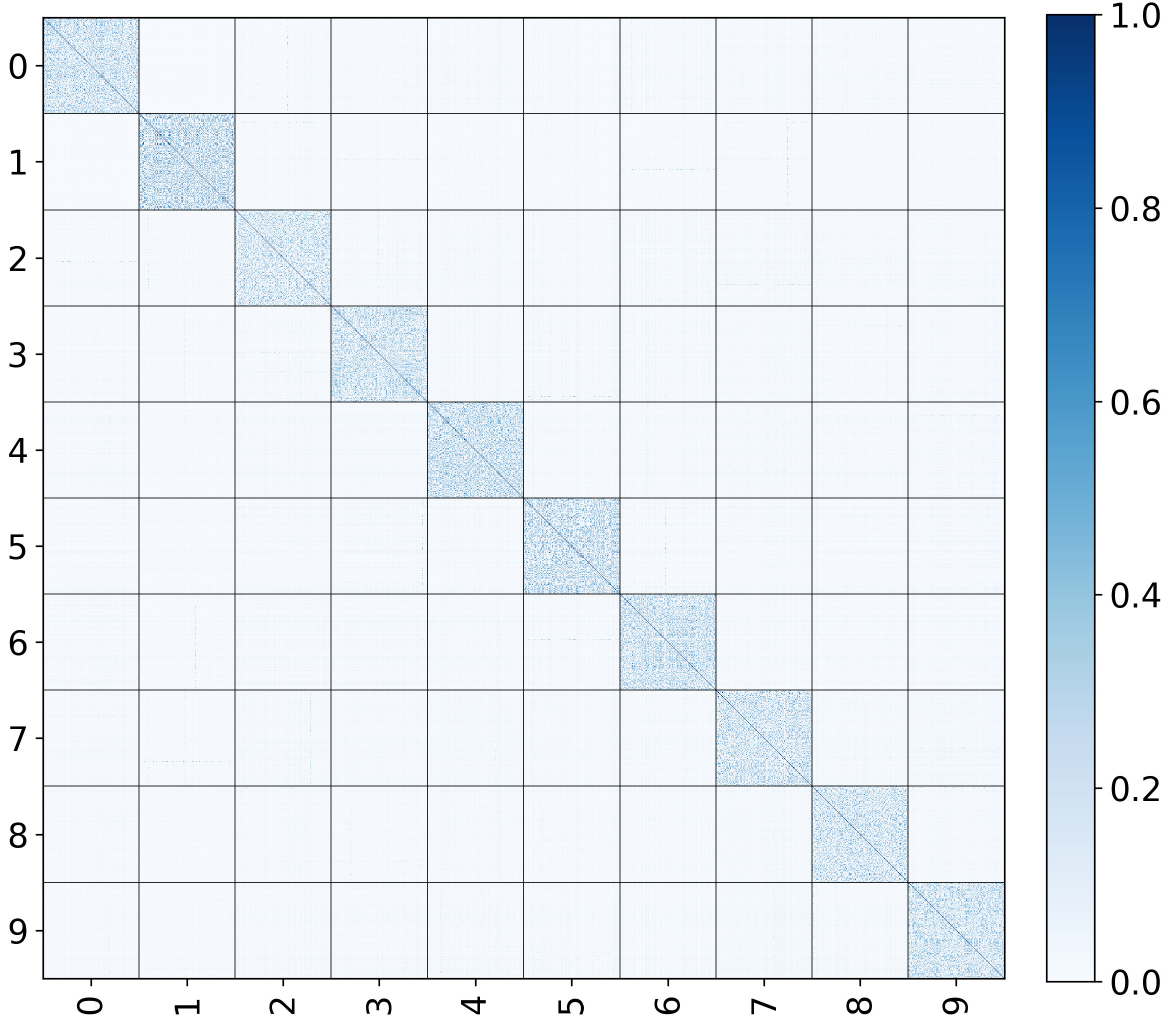}
%\caption{fig2}
\end{minipage}%
}%
\subfigure[D-SGD]{
\begin{minipage}[t]{0.25\textwidth}
\centering
\includegraphics[width=\textwidth]{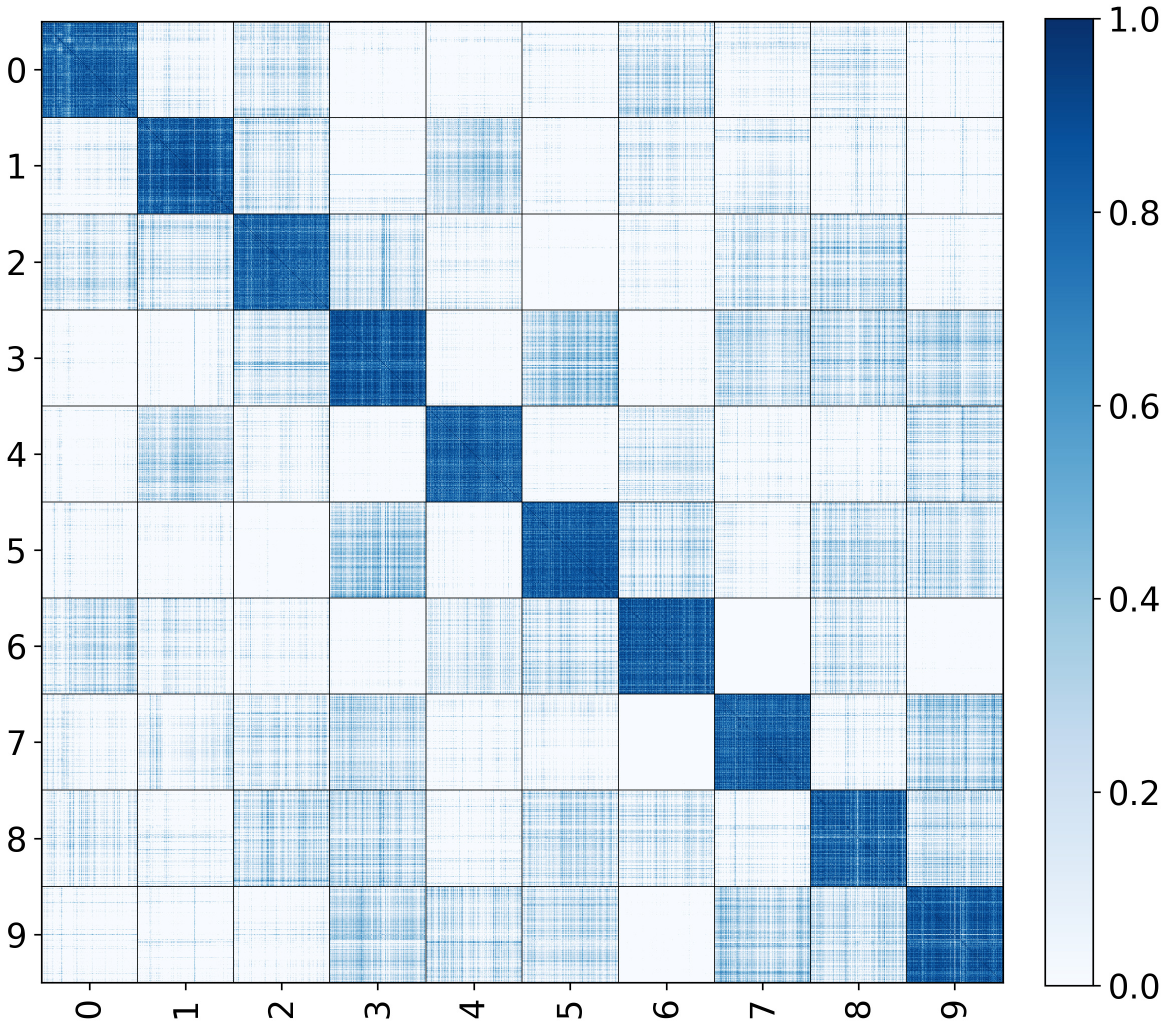}
%\caption{fig2}
\end{minipage}
}%
\subfigure[Independent MCR$^2$]{
\begin{minipage}[t]{0.25\textwidth}
\centering
\includegraphics[width=\textwidth]{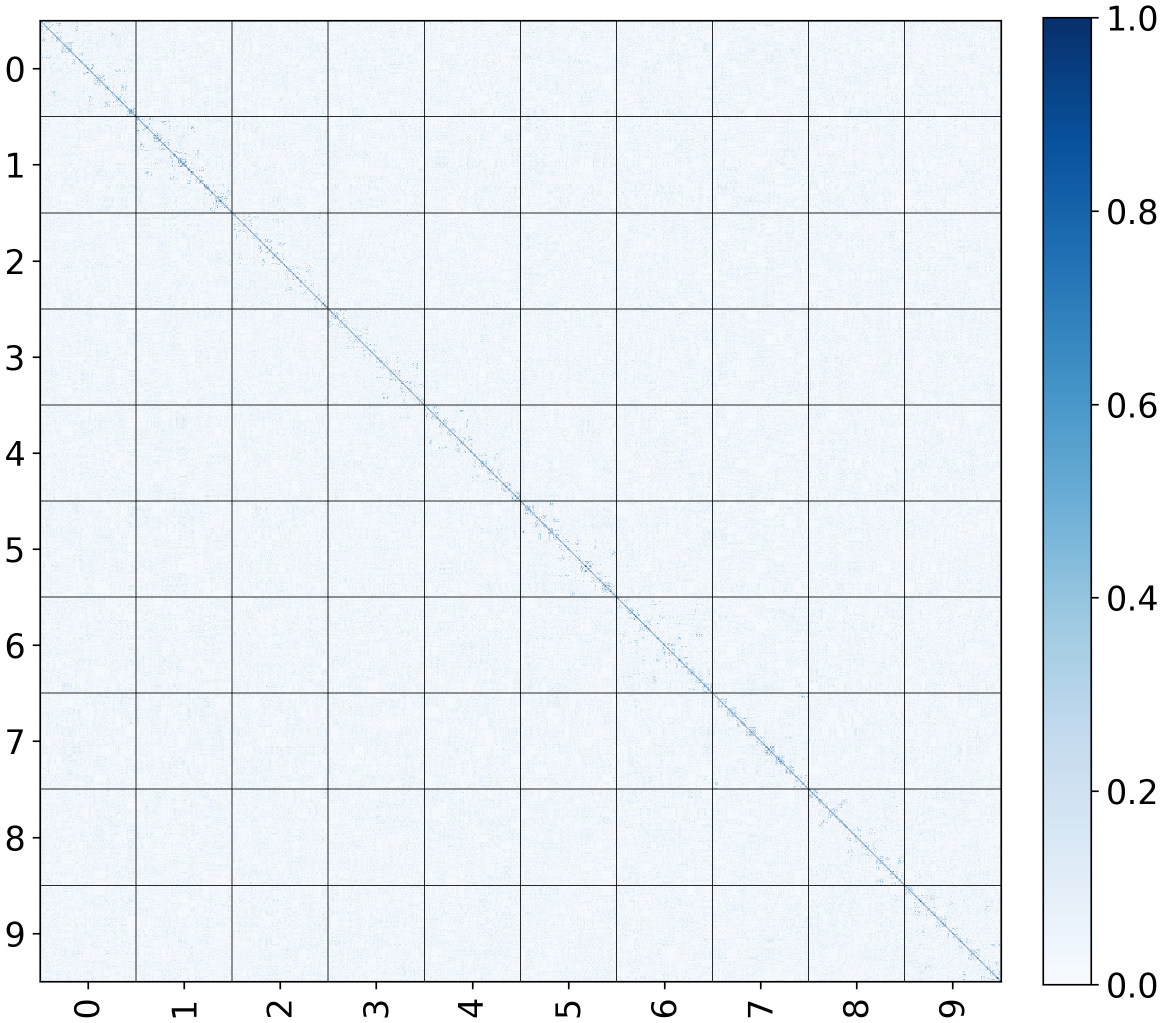}
%\caption{fig2}
\end{minipage}
}%

\subfigure[FedU$^2$]{
\begin{minipage}[t]{0.25\textwidth}
\centering
\includegraphics[width=\textwidth]{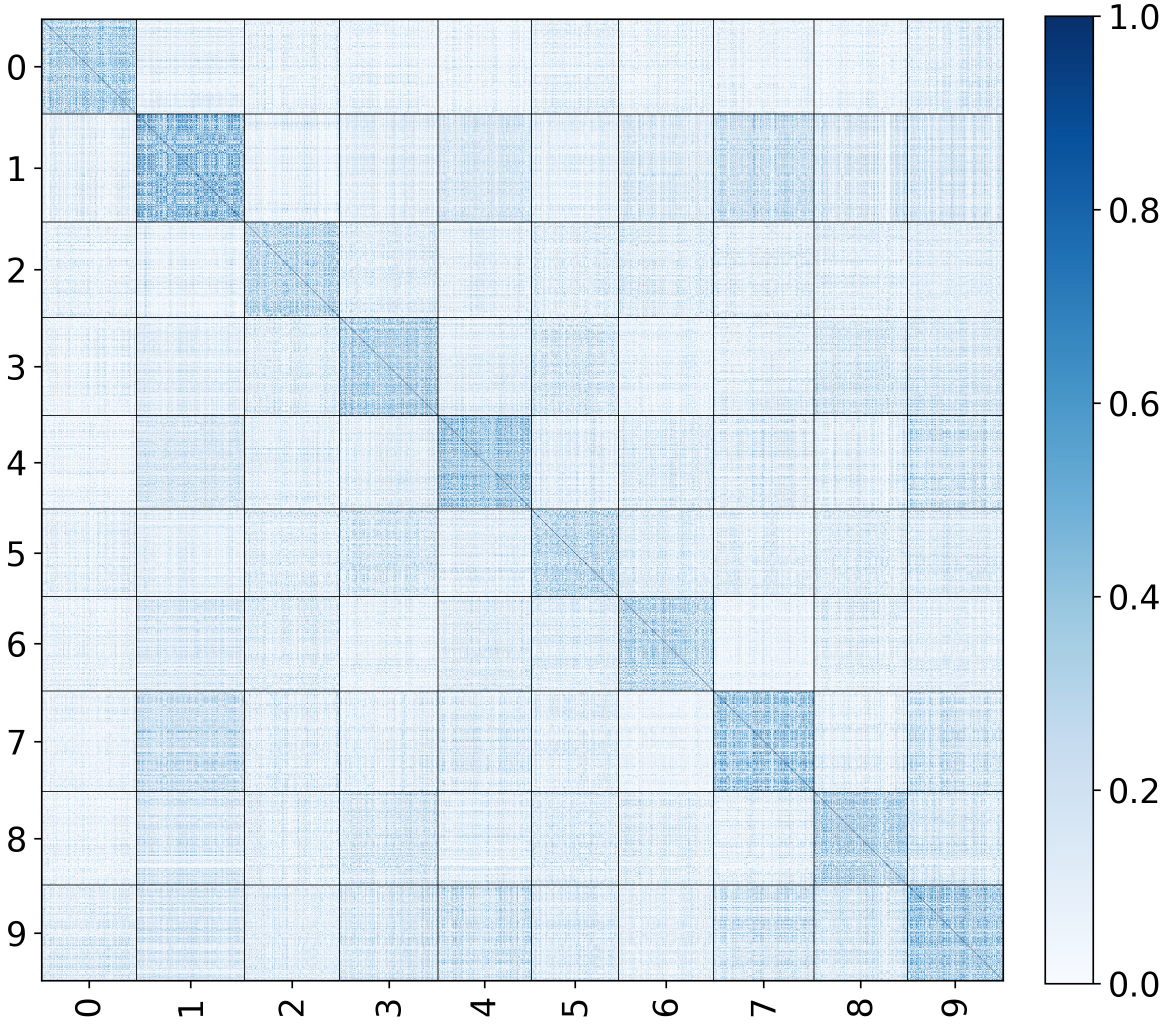}
%\caption{fig1}
\end{minipage}%
}%
\subfigure[FedSimCLR]{
\begin{minipage}[t]{0.25\textwidth}
\centering
\includegraphics[width=\textwidth]{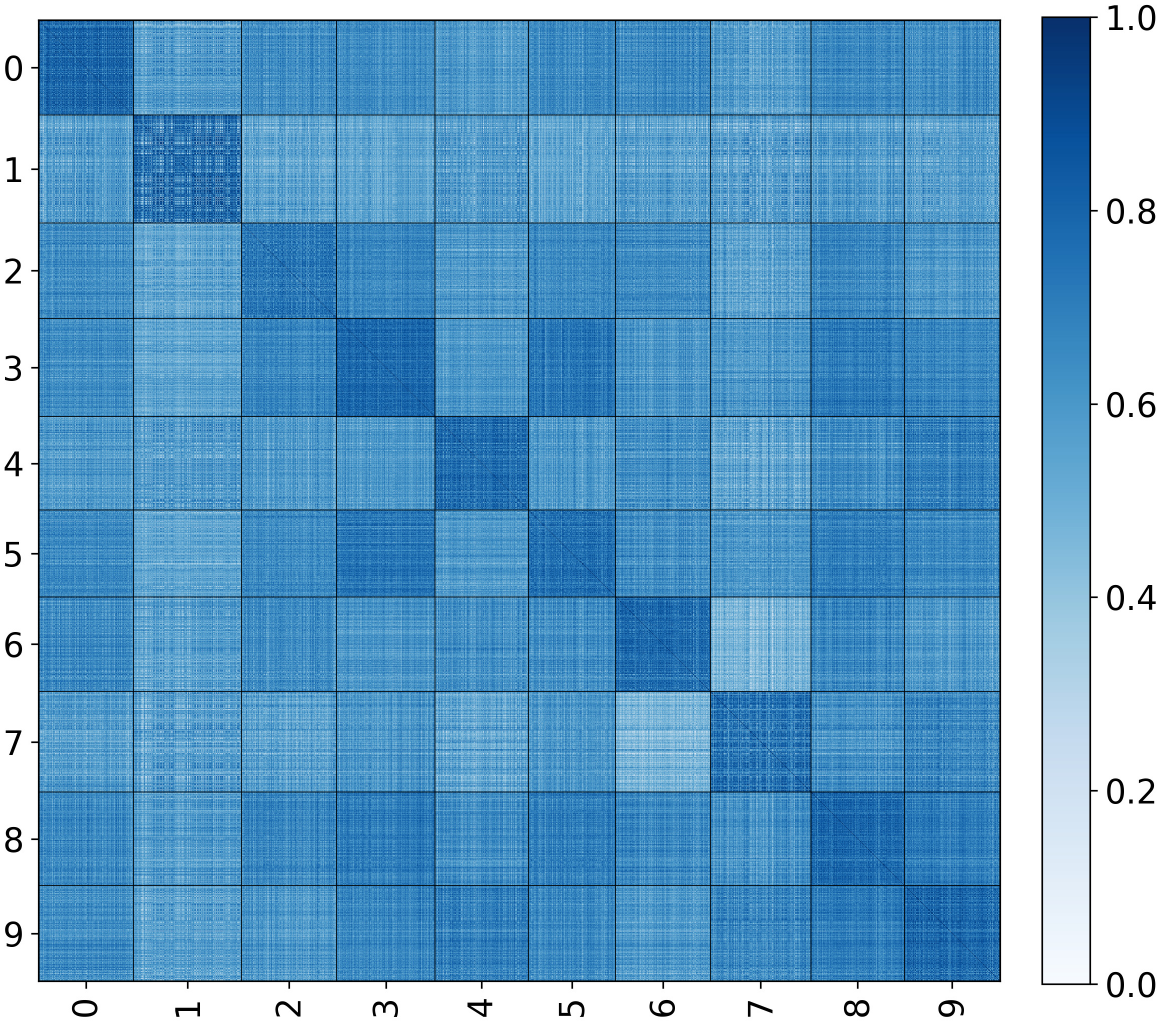}
%\caption{fig2}
\end{minipage}%
}%
\subfigure[FedU]{
\begin{minipage}[t]{0.25\textwidth}
\centering
\includegraphics[width=\textwidth]{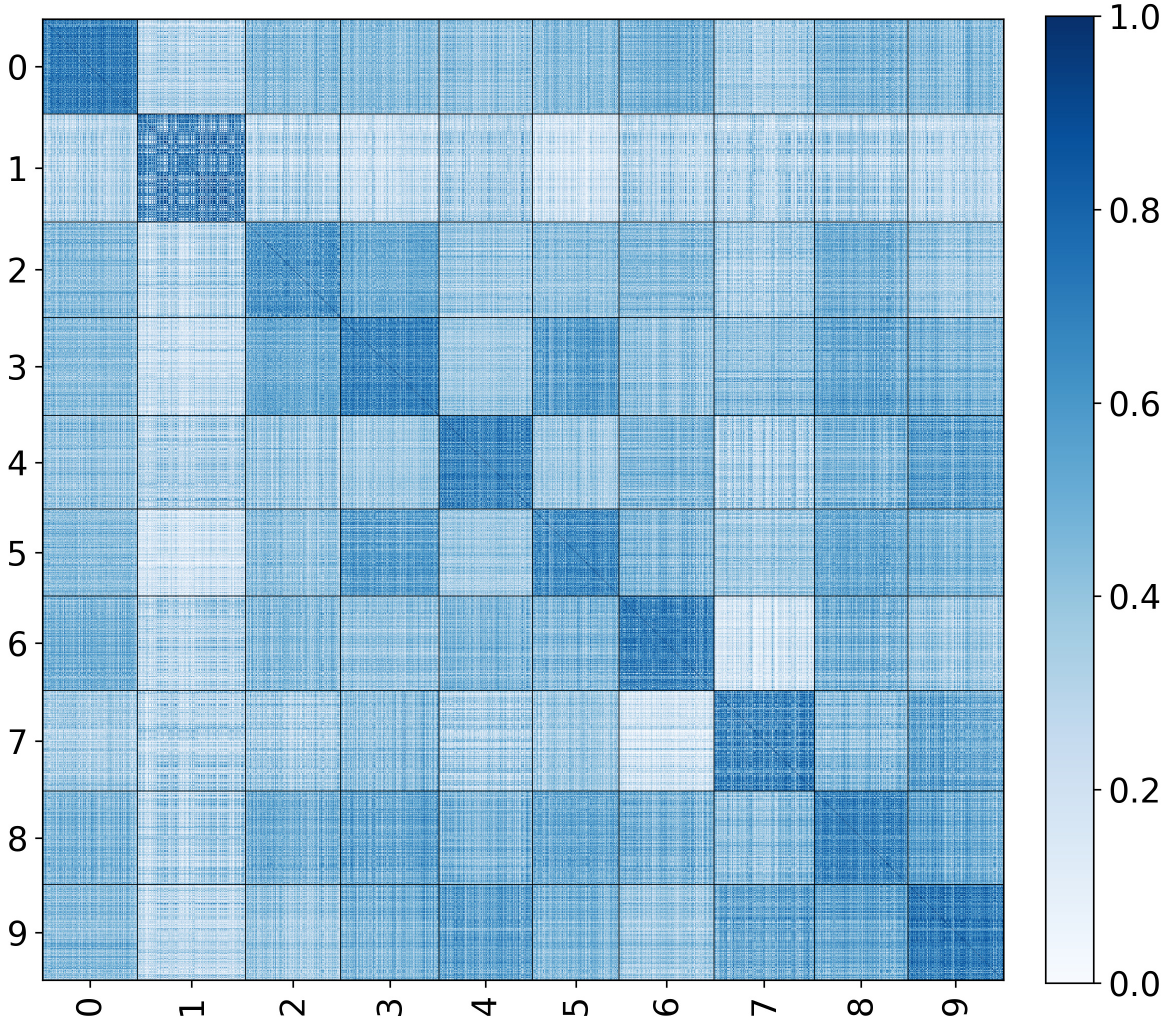}
%\caption{fig2}
\end{minipage}
}%
\subfigure[Orchestra]{
\begin{minipage}[t]{0.25\textwidth}
\centering
\includegraphics[width=\textwidth]{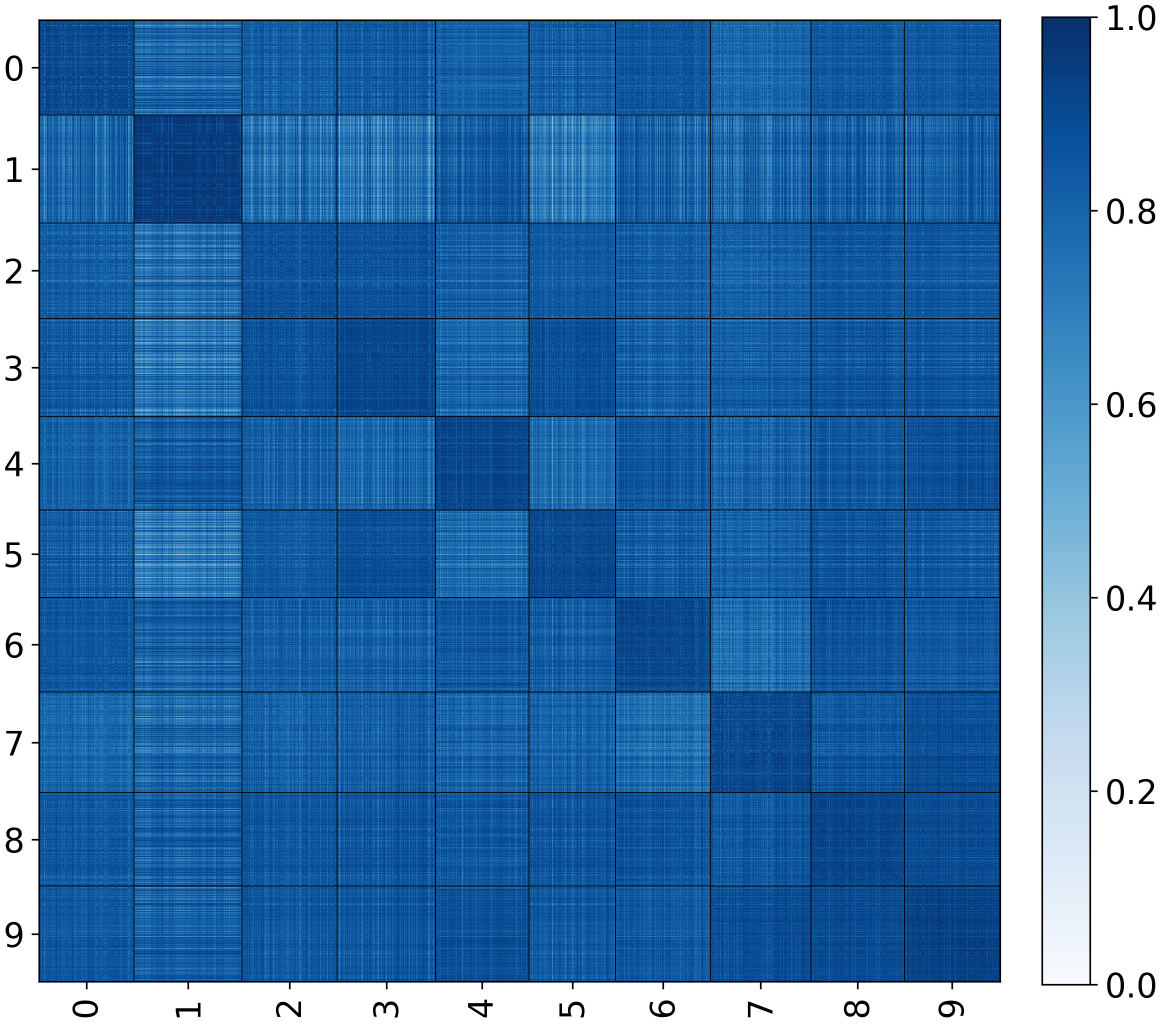}
%\caption{fig2}
\end{minipage}
}%
\vspace{-0.2cm}
\centering
\caption{Cosine similarity between learned representations for MNIST under i.i.d. data distribution by using different algorithms.}\label{fig_exp2}
\vspace{-0.6cm}
\end{figure*}

As shown in Fig. \ref{fig_exp2}(a), the representation of different classes are nearly orthogonal when using the proposed distributed Algorithm \ref{Alg1}. In comparison with Fig. \ref{fig_exp2}(b), the equivalence of the optimization problems—between our formulation and the original one-can be verified by the similar resulting structure. In Fig. \ref{fig_exp2}(c), D-SGD demonstrates performance comparable to centralized conventional methods with a cross-entropy loss function. Here, it is evident that the representations within each class are strongly correlated, as indicated by the dark blue color blocks, reflecting the neural collapse phenomenon. This shows that features within each class converge to the means with zero variability, and the features across different classes remain correlated. In Fig. \ref{fig_exp2}(d), independent local updating with a limited number of training data samples results in poor performance in the learned representations, highlighting the need for communication and message sharing between nodes, thus validating the effectiveness of our proposed framework.
{In Fig. \ref{fig_exp2}(e), the learned representations of FedU$^2$ show that the neural collapse problem is mitigated, while the cross-label correlation is still severe from those colored non-diagonal blocks. Finally, Figs. \ref{fig_exp2}(f-h) show that those self-supervised methods still suffer from neural collapse problem and the resultant representations lack the discriminative properties.}

Finally, we compare the singular values of global representations, for both overall data and each individual class. The results are shown in Fig. \ref{fig_exp3}. 
We find that the representations learned by our proposed algorithm are much more diverse than those learned by D-SGD, with much higher dimensional spaces/subspaces, both for the overall data space and the subspaces for each class.

\begin{figure}[htbp]
\centering
\subfigure[Proposed Algorithm \ref{Alg1}]{
\begin{minipage}[t]{0.24\textwidth}
\centering
\includegraphics[width=\textwidth]{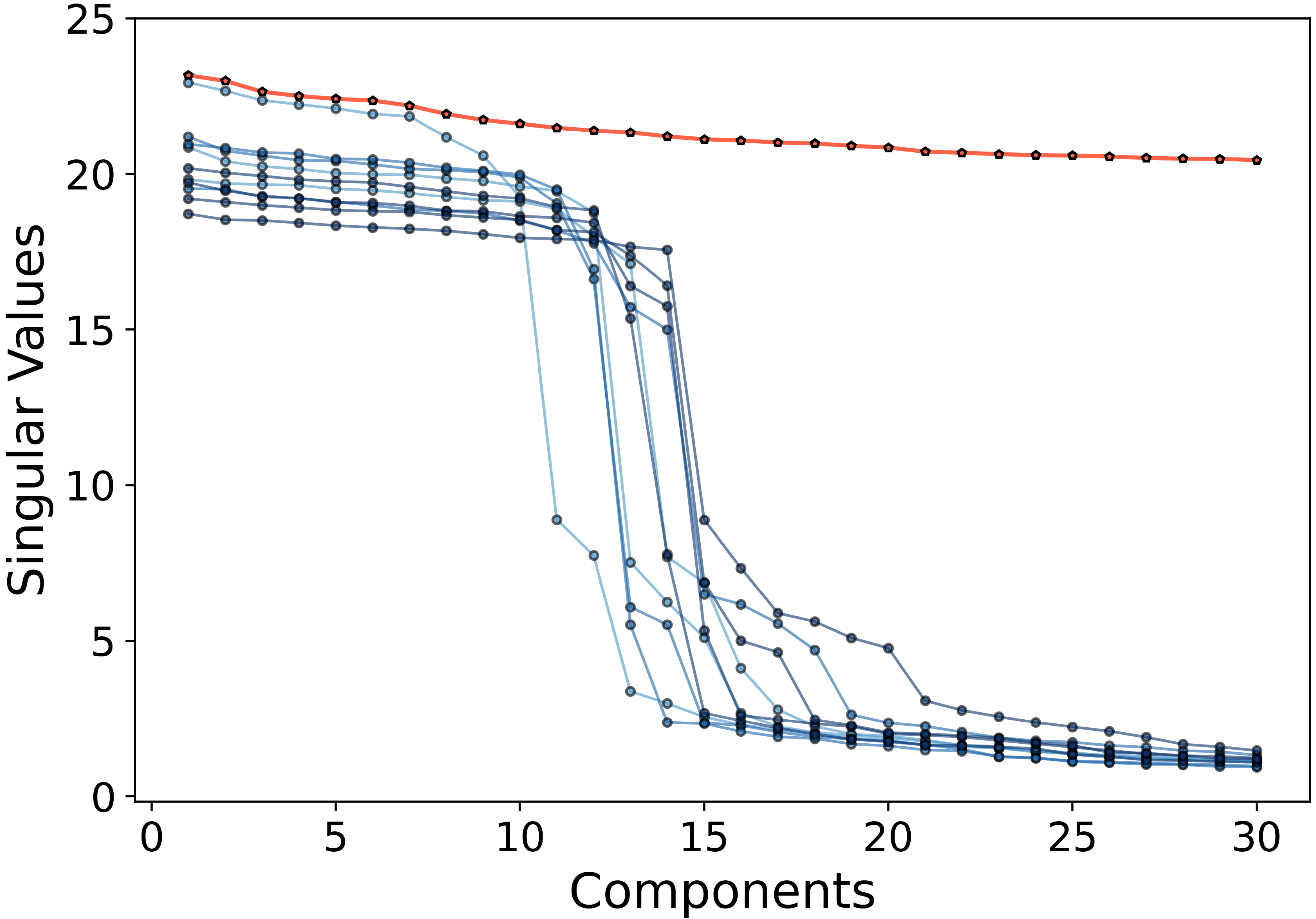}
\end{minipage}%
}%
\subfigure[D-SGD]{
\begin{minipage}[t]{0.24\textwidth}
\centering
\includegraphics[width=\textwidth]{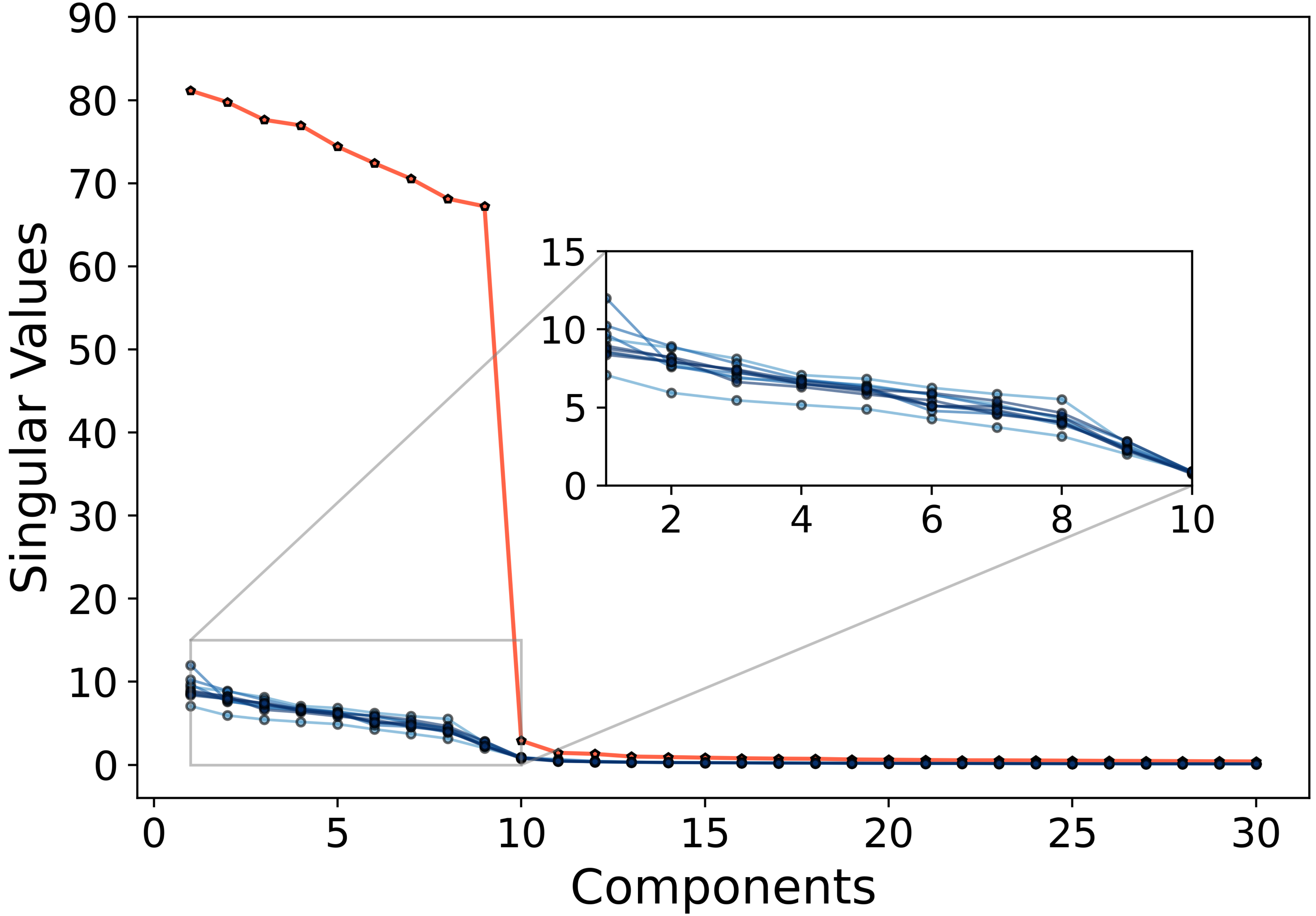}
%\caption{fig2}
\end{minipage}
}%
\vspace{-0.2cm}
\caption{Singular value comparison for i.i.d. MNIST dataset: (red) overall data; (blue) individual class}\label{fig_exp3}
\end{figure}

\textbf{CIFAR10.} 
Secondly, we show the experiments on CIFAR10 datasets. To verify the superiority of our algorithm in the heterogeneous model conditions, we consider nodes with different neural network architectures. 
Specifically, we use the same communication network topology. The $10$ nodes have the following neural network architecture respectively: ResNet18, ResNet34, VGG11, VGG16, ResNet18, ResNet34, VGG11, VGG16, ResNet18, ResNet34. 
The last linear layer of each model is replaced by a two-layer fully connected network with ELU activation function, such that the output dimension is $128$. Moreover, the normalization function is added to the outputs to satisfy the first constraint in the optimization problem (\ref{opt1}).
We set the initial learning rate to $0.01$ with a weight decay of $10^{-5}$ and use the Adam optimizer.
The mini-batch size is set to $1000$ and the precision parameter $\epsilon=0.5$.
In each iteration, to approximate the optimal solution of (\ref{R1_2c}), each node conducts $5$ local epochs. $\rho$ is set to $0.1$ and $\gamma=2.0$. 

\begin{figure*}[htbp]
\centering
\subfigure[Proposed Algorithm 1]{
\begin{minipage}[t]{0.25\textwidth}
\centering
\includegraphics[width=\textwidth]{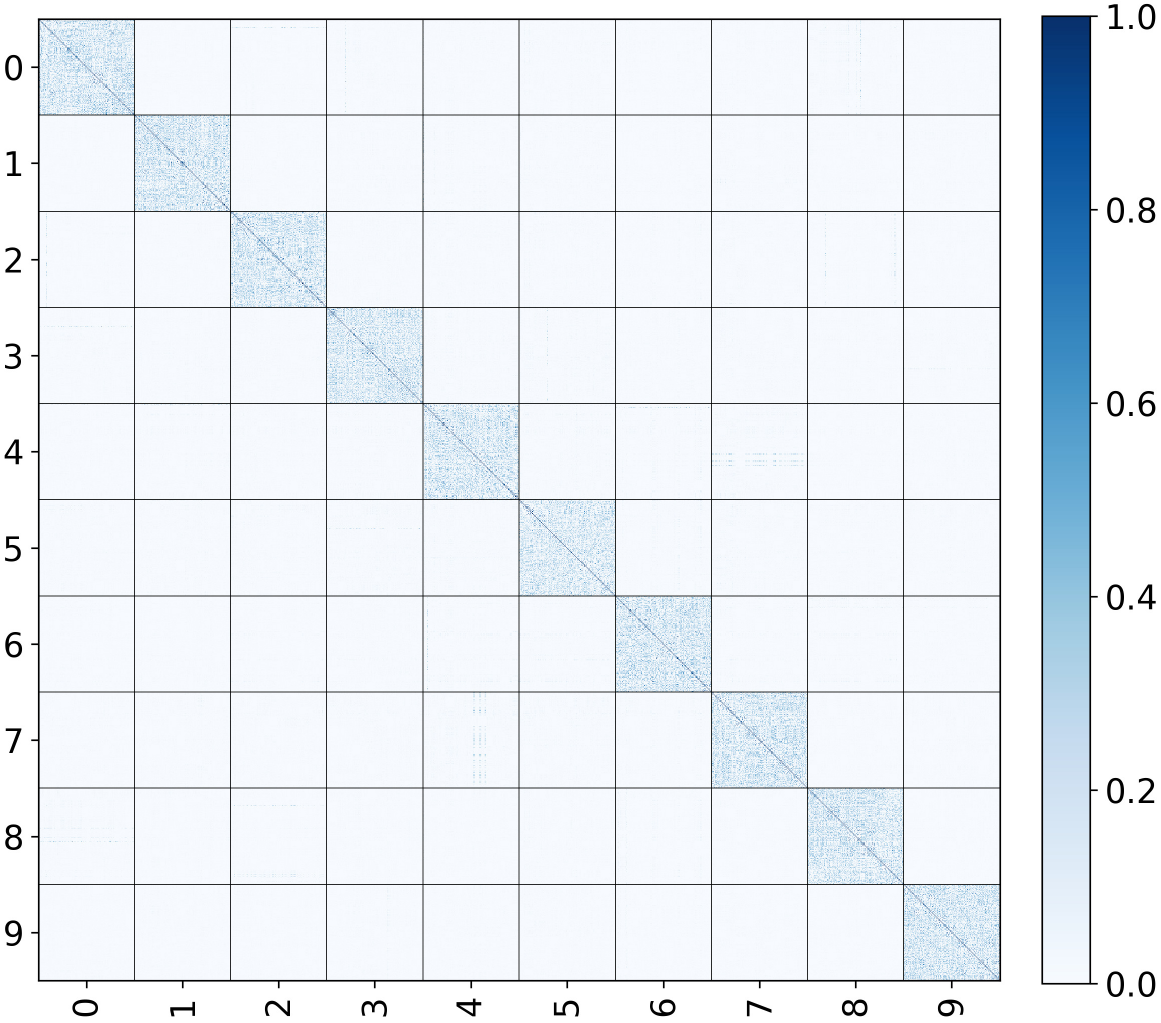}
%\caption{fig1}
\end{minipage}%
}%
\subfigure[Centralized MCR$^2$]{
\begin{minipage}[t]{0.25\textwidth}
\centering
\includegraphics[width=\textwidth]{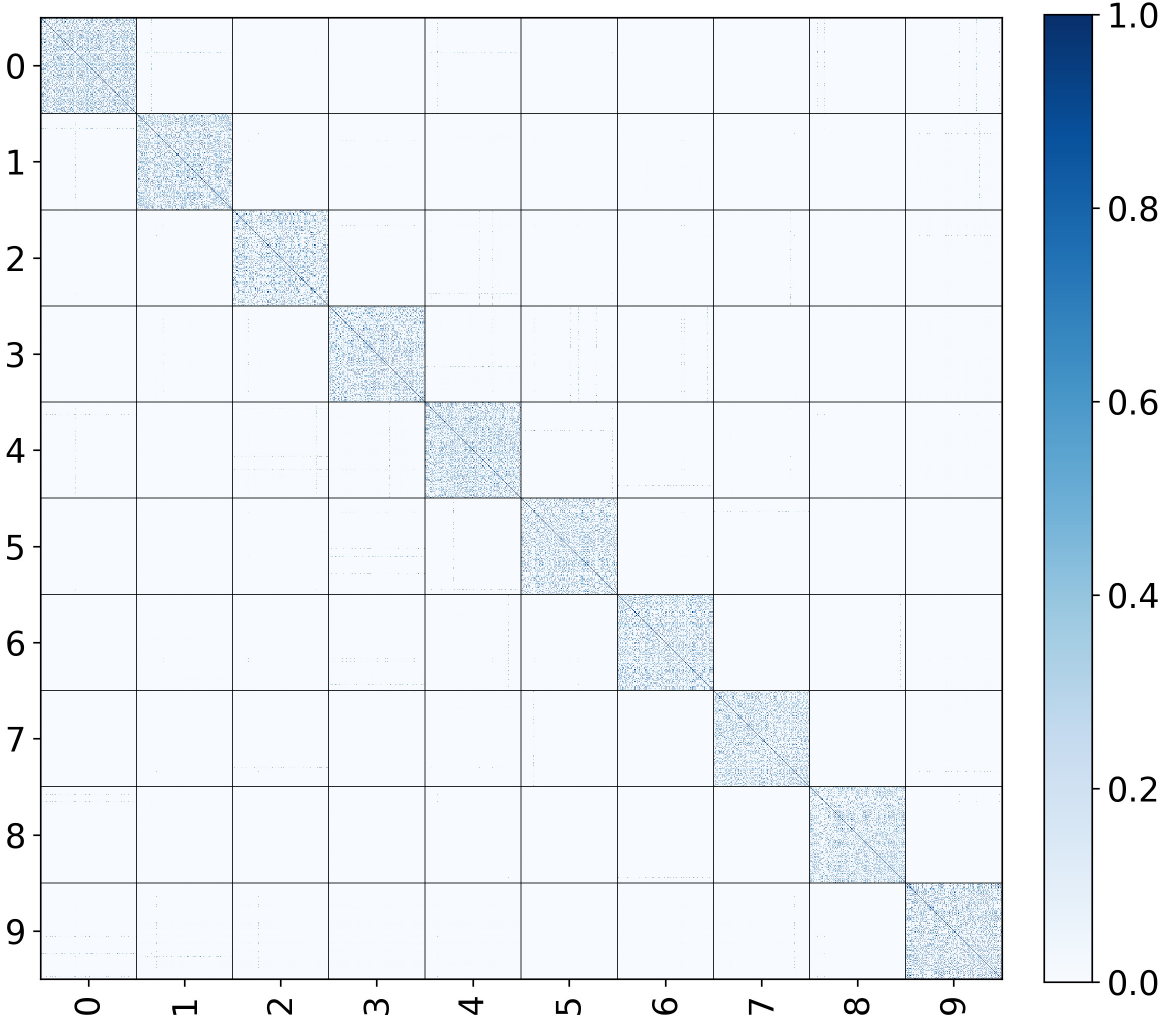}
%\caption{fig2}
\end{minipage}%
}%
\subfigure[D-SGD]{
\begin{minipage}[t]{0.25\textwidth}
\centering
\includegraphics[width=\textwidth]{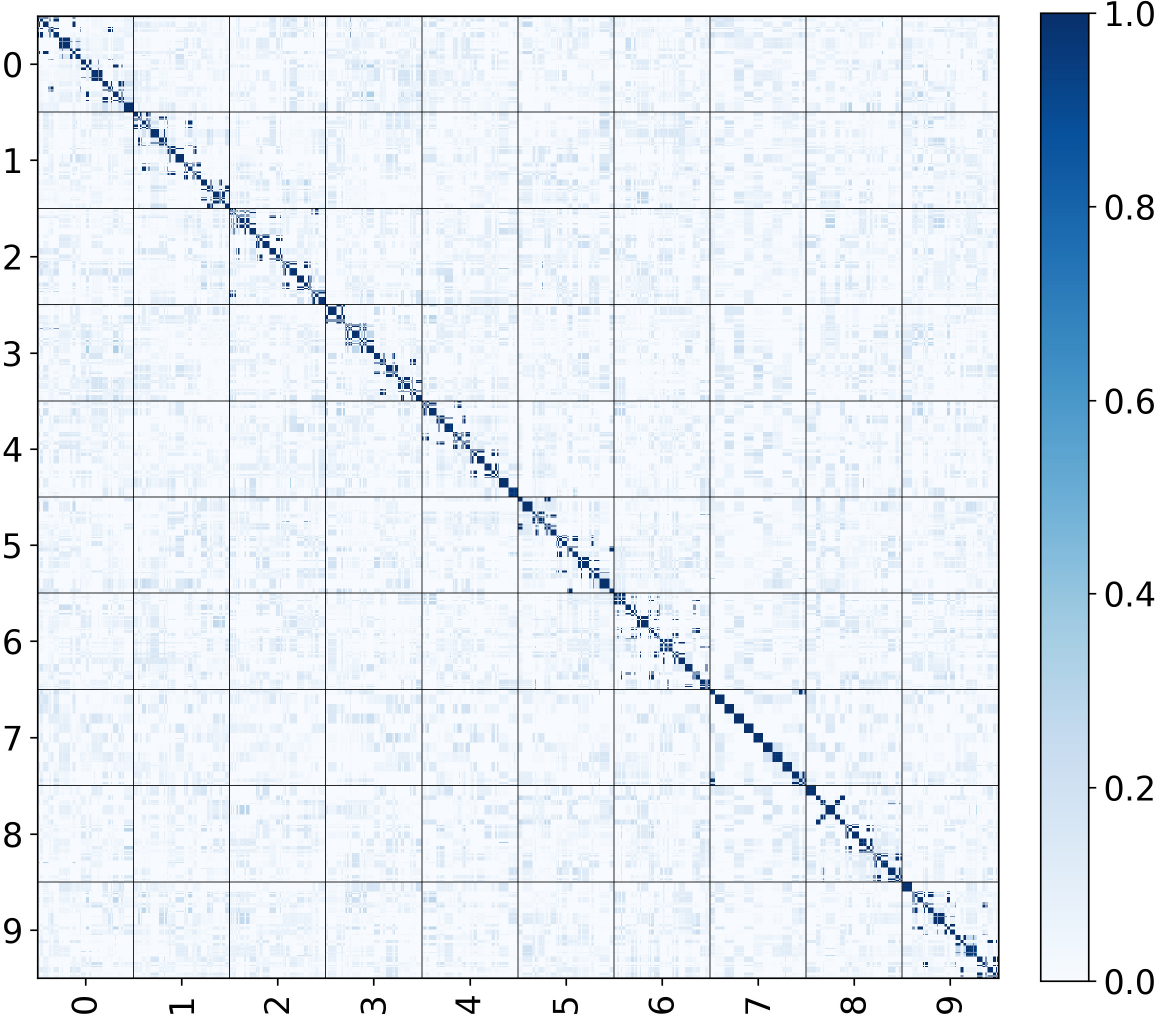}
%\caption{fig2}
\end{minipage}
}%
\subfigure[Independent MCR$^2$]{
\begin{minipage}[t]{0.25\textwidth}
\centering
\includegraphics[width=\textwidth]{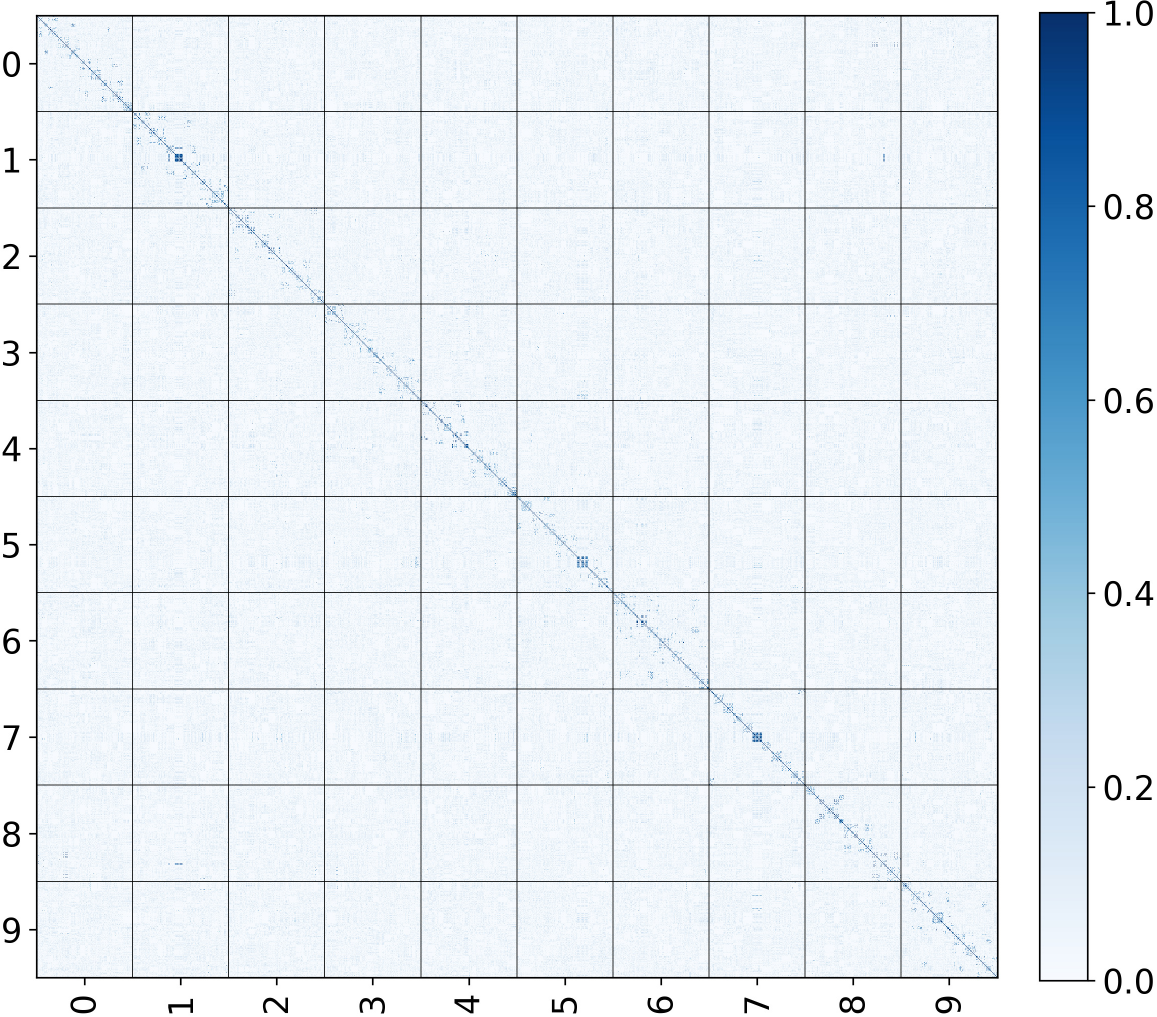}
%\caption{fig2}
\end{minipage}
}%

\subfigure[FedU$^2$]{
\begin{minipage}[t]{0.25\textwidth}
\centering
\includegraphics[width=\textwidth]{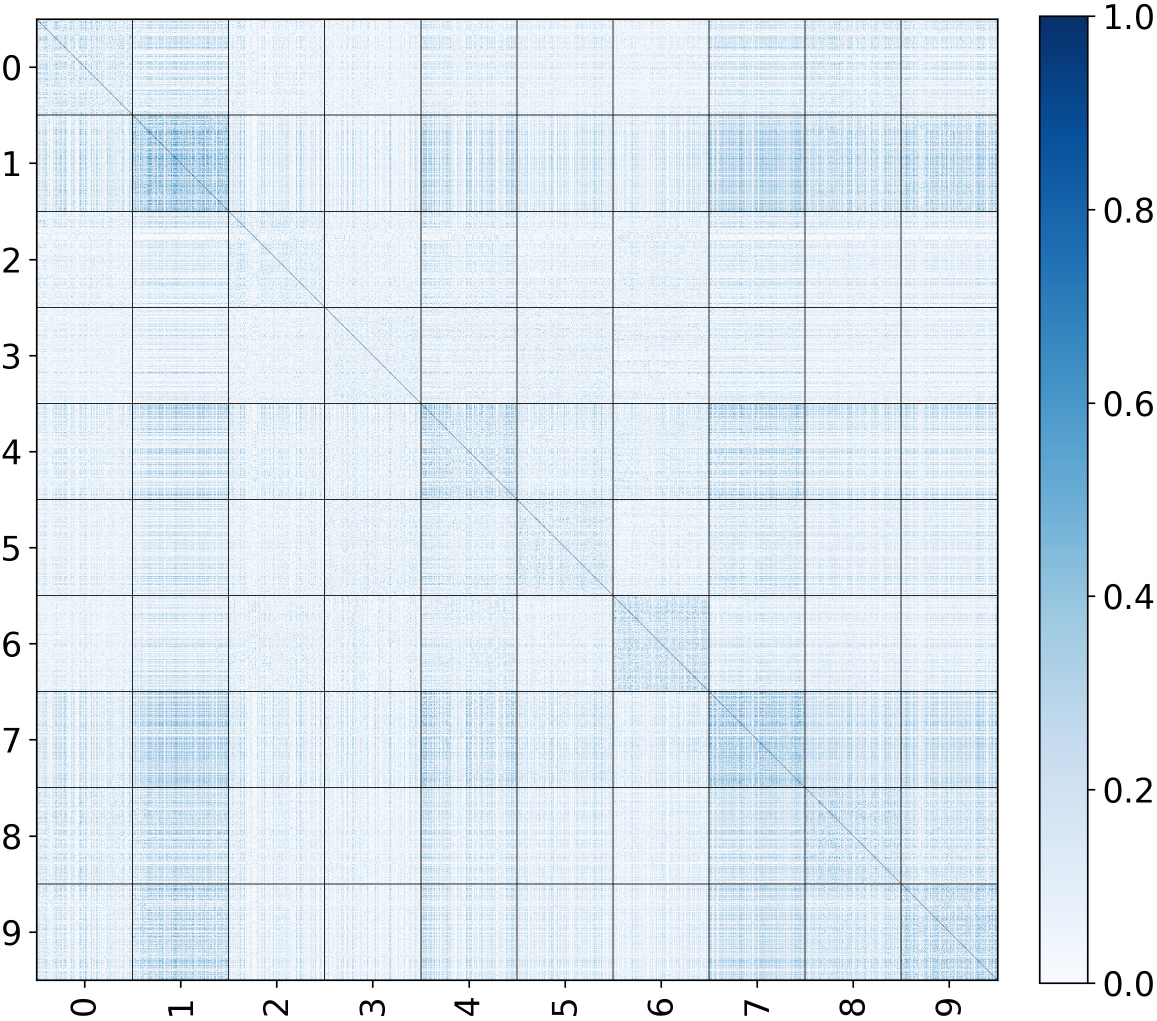}
%\caption{fig1}
\end{minipage}%
}%
\subfigure[FedSimCLR]{
\begin{minipage}[t]{0.25\textwidth}
\centering
\includegraphics[width=\textwidth]{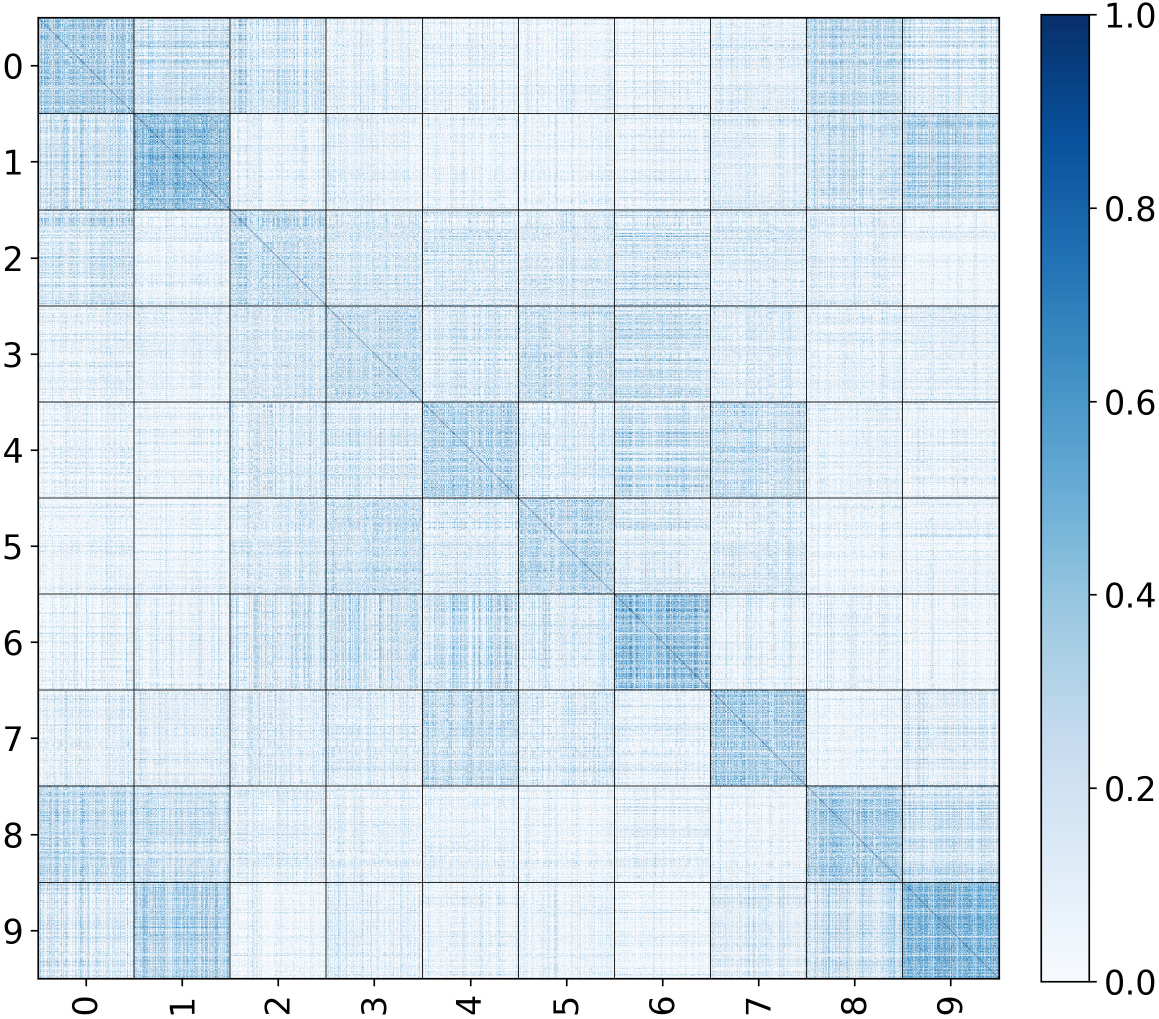}
%\caption{fig2}
\end{minipage}%
}%
\subfigure[FedU]{
\begin{minipage}[t]{0.25\textwidth}
\centering
\includegraphics[width=\textwidth]{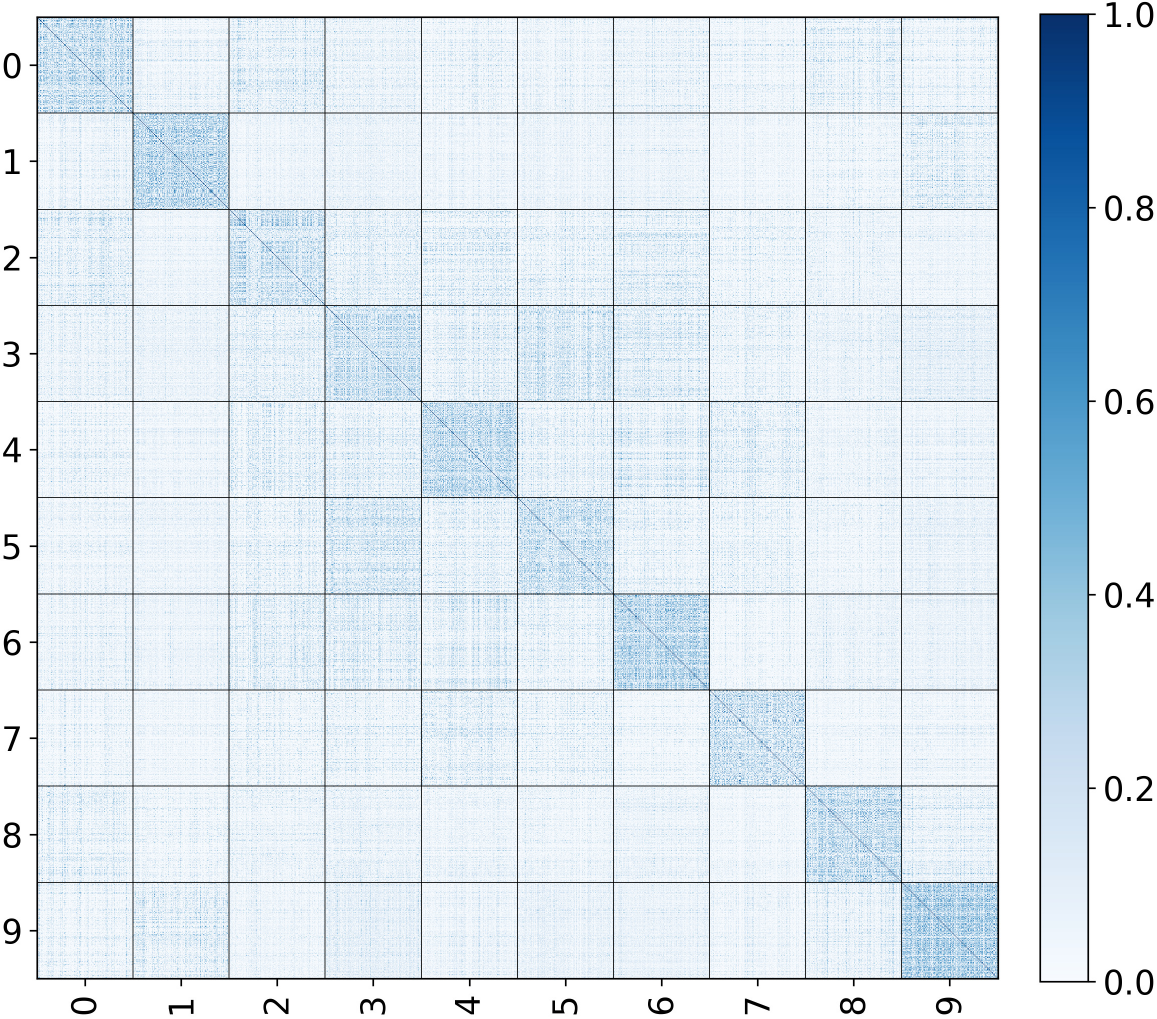}
%\caption{fig2}
\end{minipage}
}%
\subfigure[Orchestra]{
\begin{minipage}[t]{0.25\textwidth}
\centering
\includegraphics[width=\textwidth]{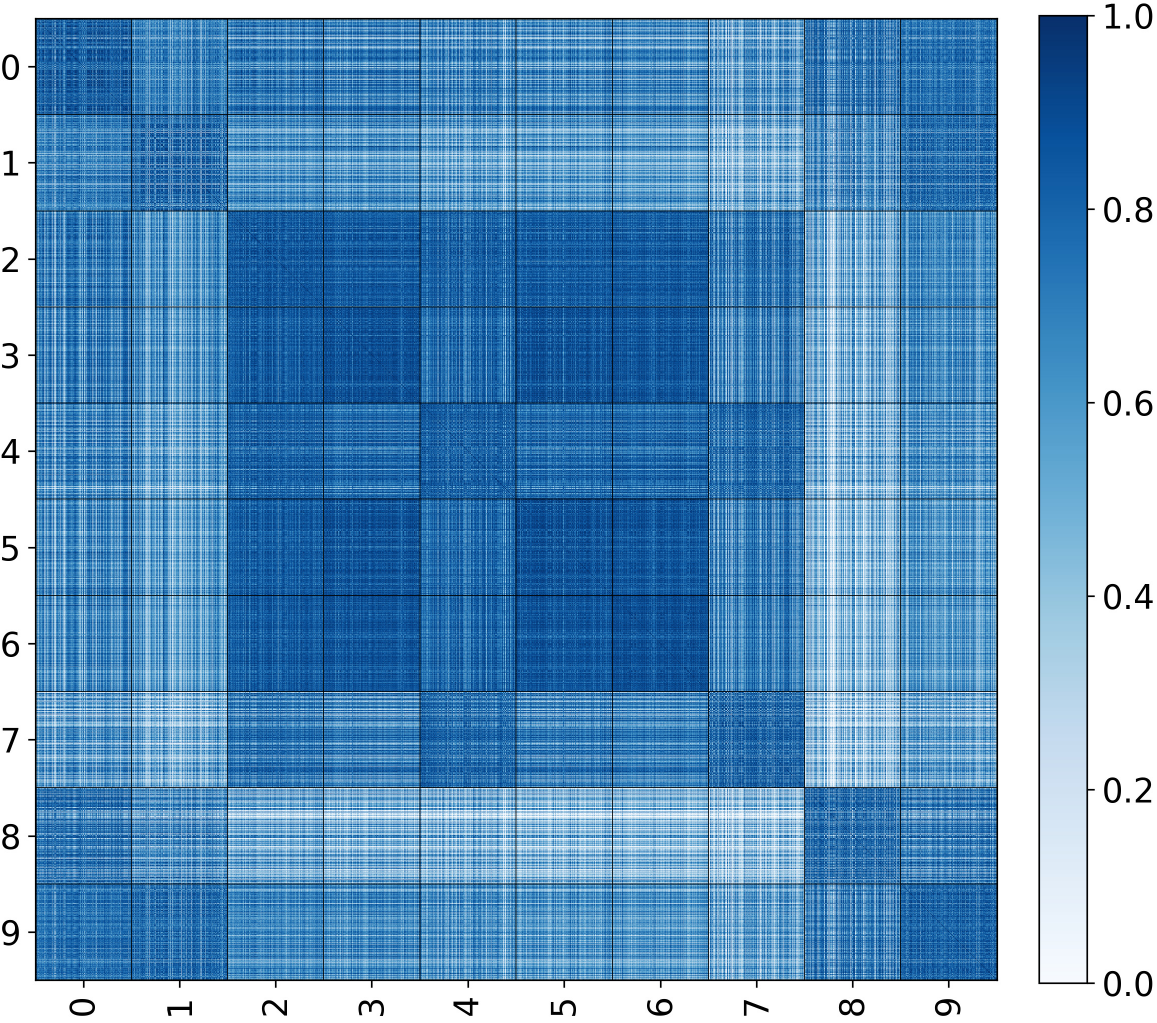}
%\caption{fig2}
\end{minipage}
}%
\vspace{-0.2cm}
\centering
\caption{Cosine similarity between learned representations for CIFAR10 under i.i.d. data distribution using different algorithms.}\label{fig_exp5}
\vspace{-0.6cm}
\end{figure*}

The convergence curves of the averaged training loss over $1000$ iterations are presented in Fig. \ref{fig_exp1}(b), demonstrating the convergence property of the proposed Algorithm \ref{Alg1}.
Additionally, we show the cosine similarity of the learned representations and compare the results with other methods. 
Regarding the heterogeneous model architectures among nodes, D-SGD is implemented by sharing and averaging the model parameters between agents that have the same architecture. 
{For those federated representation learning algorithms relying on model aggregation, we fix their neural network architectures as ResNet18 among agents.}
The results are shown in Fig. \ref{fig_exp5}. The similarity between Fig. \ref{fig_exp5}(a) and (b) confirms the effectiveness of our proposed method and the properties outlined in Theorem \ref{t1} under model heterogeneity conditions.
Furthermore, in Fig. \ref{fig_exp5}(c), the representations within each class no longer exhibit correlations due to the heterogeneous model architecture and partial averaging among agents. 
The result in Fig. \ref{fig_exp5}(d) further highlights the effectiveness of nodes' collaboration.
{The comparison between Fig. \ref{fig_exp5}(a) and Figs. \ref{fig_exp5}(e-h) exhibits the discriminative and diverse properties of the proposed algorithm and the MCR$^2$ principle, compared with those self-supervised federated representation learning algorithms.}

% Finally, we also show the singular values of the representations over the network, for both overall data and each individual class, as in Fig. \ref{fig_exp6}. It can be observed here that our features are not only orthogonal but also of much higher dimension. \begin{figure}[!htp] \centering\includegraphics[width=0.4\textwidth]{./pics/IID_CIFAR10/exp_IID_CIFAR_PCA_ColMCR.pdf}\caption{Singular values for i.i.d. CIFAR10 dataset: (red) overall data; (blue) individual class}\label{fig_exp6} \vspace{-0.5cm} \end{figure}

\subsection{Experiments with non-i.i.d. Setting}
In this part, we simulate the performance of Algorithm \ref{Alg3} in the non-i.i.d. data distribution conditions. 
Specifically, we first simulate the conditions where the nodes can be fully clustered without local replication, to show the effectiveness of the parallel cluster and inner BCD. Then we consider more complicated scenarios where the nodes need to be locally replicated for inter-cluster i.i.d. data distribution.

\begin{figure}[htbp]
    \centering
    \includegraphics[width=0.5\textwidth]{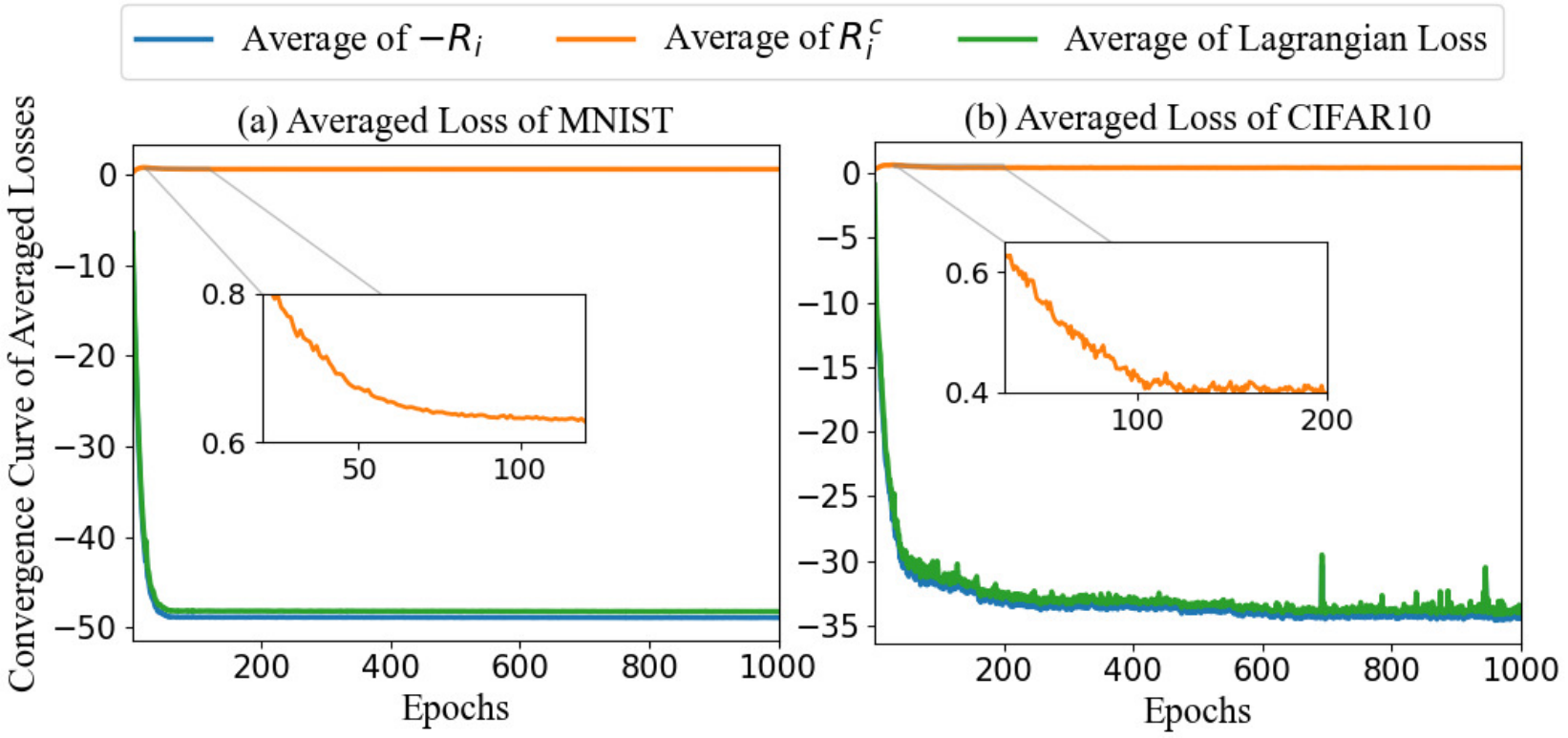}
    \caption{Convergence curves of the averaged loss.}
    \label{fig_exp7}
    \vspace{-0.5cm}
\end{figure}

In the first scenario, we consider $4$ nodes in the network with local labels $[1,2,3,4,5],[6,7$, $8,9,0], [0,3,5,7,9], [1,2,4,6,8]$ respectively. Based on clustering, we could obtain two clusters in the network, i.e., cluster $1$ with nodes $0,1$ and cluster $2$ with nodes $2,3$.
%We first randomly assign $4$ different classes into $5$ nodes and these nodes have the classes $[0,3,5,6], [0,5,8,7],[1,3,8,9],[2,4,6,7],[1,2,4,9]$ respectively. Based on the clustering Algorithm \ref{Alg2}, we could obtain two clusters in the network, i.e., cluster $1$ with node $0,1,4, (\mathcal{K}_0^1=\{0,3,5,6\}, \mathcal{K}_1^1=\{0,5,7,8\}, \mathcal{K}_4^1=\{1,2,4,9\})$ and cluster $2$ with node $5,2,3, (\mathcal{K}_5^2=\{0,5,7,8\}, \mathcal{K}_2^2=\{1,3,8,9\}, \mathcal{K}_3^2=\{2,4,6,7\})$. Here node $1$ and $5$ are the replicated virtual nodes in node $1$.
Similar to those in the previous subsection, we first simulate the results in the MNIST dataset, where the nodes have the common neural network architecture. 
We set $\gamma=2.0$, the initial learning rate $0.01$ and other settings are the same as those in Section \ref{Sec5}-A.
The convergence curves of Algorithm \ref{Alg3} are shown in Fig. \ref{fig_exp7}(a) and the cosine similarities of the global representation are shown in Fig. \ref{fig_exp8}. It can be observed in Fig. \ref{fig_exp8}, that through the cooperation among nodes, the within-class correlation is built while the between-class orthogonality is enhanced. % {And the comparison between Fig. \ref{fig_exp7}(a) and (e-h) further show the structural properties of the proposed algorithm.}

\begin{figure}[htbp]
\centering
\subfigure[Proposed Algorithm \ref{Alg3}]{
\begin{minipage}[t]{0.24\textwidth}
\centering
\includegraphics[width=\textwidth]{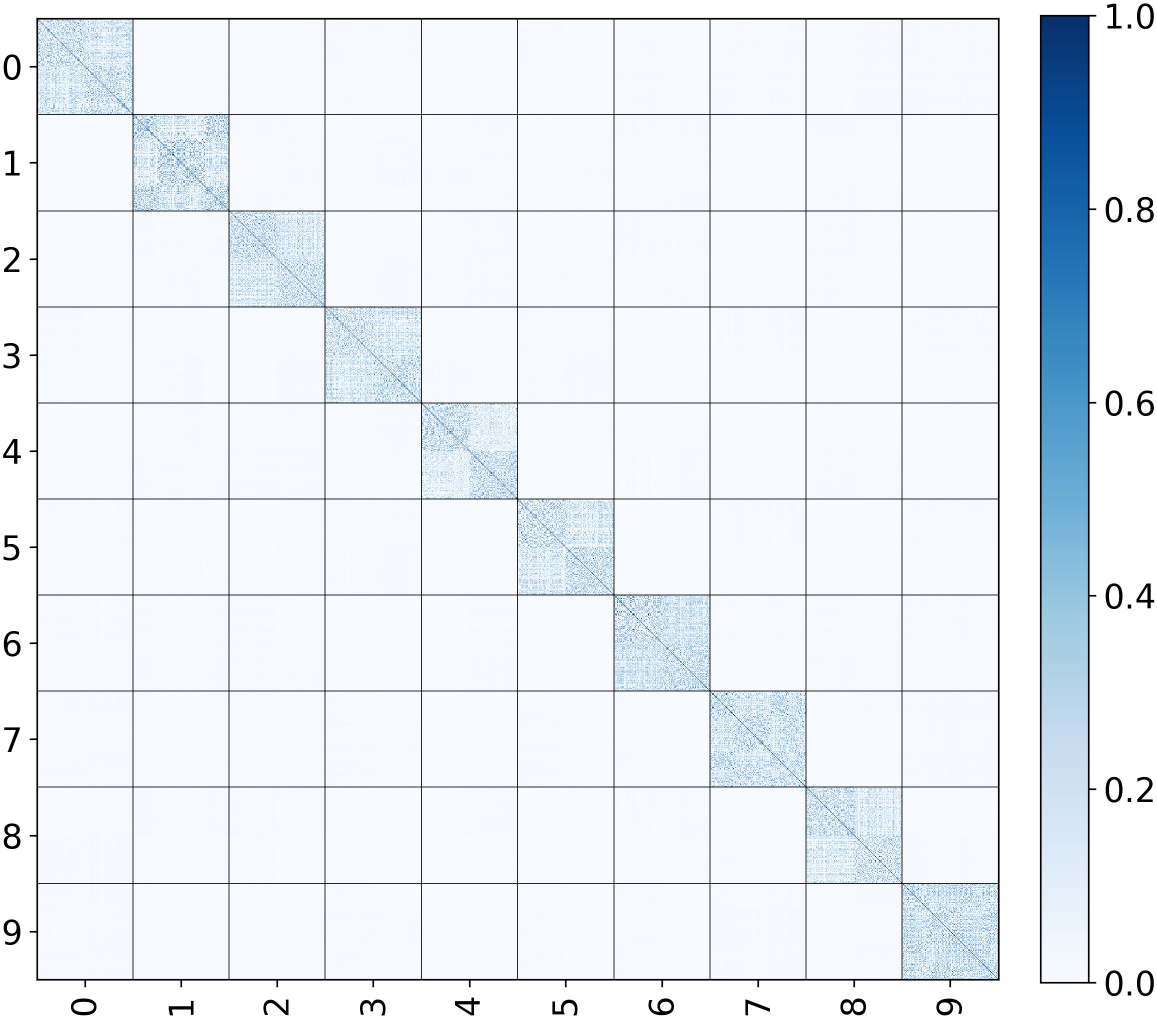}
%\caption{fig1}
\end{minipage}%
}%
\subfigure[Independent MCR$^2$]{
\begin{minipage}[t]{0.24\textwidth}
\centering
\includegraphics[width=\textwidth]{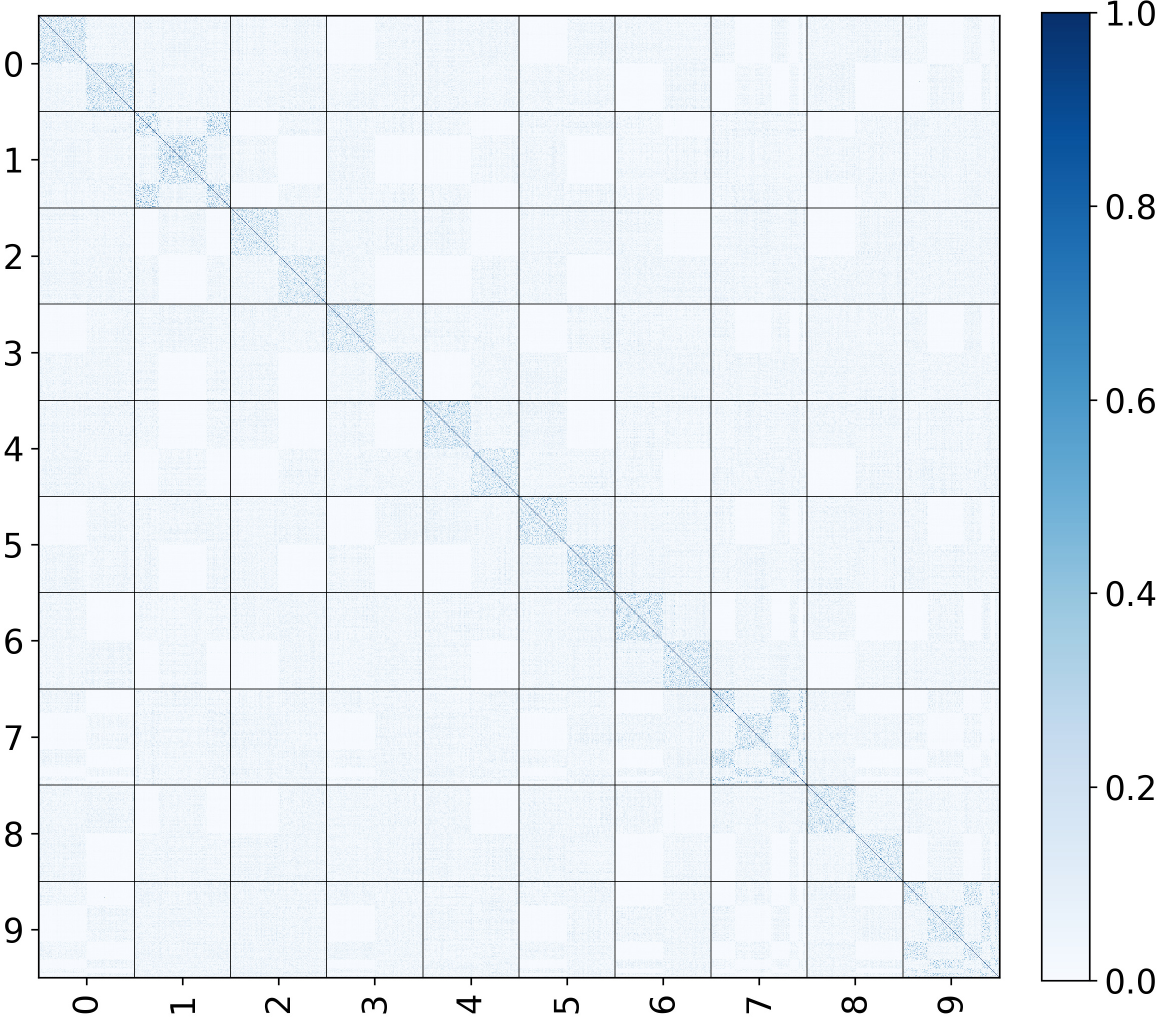}
%\caption{fig2}
\end{minipage}
}%
\vspace{-0.2cm}
\centering
\caption{Cosine similarity between learned representations for MNIST under non-i.i.d. data distribution.}\label{fig_exp8}
\vspace{-0.3cm}
\end{figure}

For the CIFAR-10 dataset and four nodes with ResNet18, VGG11, VGG16, and ResNet34 model architectures, the convergence curves are shown in Fig. \ref{fig_exp7}(b) and the comparison of the learned representations are shown in Fig. \ref{fig_exp9}.

\begin{figure}[htbp]
\centering
\subfigure[Proposed Algorithm \ref{Alg3}]{
\begin{minipage}[t]{0.24\textwidth}
\centering
\includegraphics[width=\textwidth]{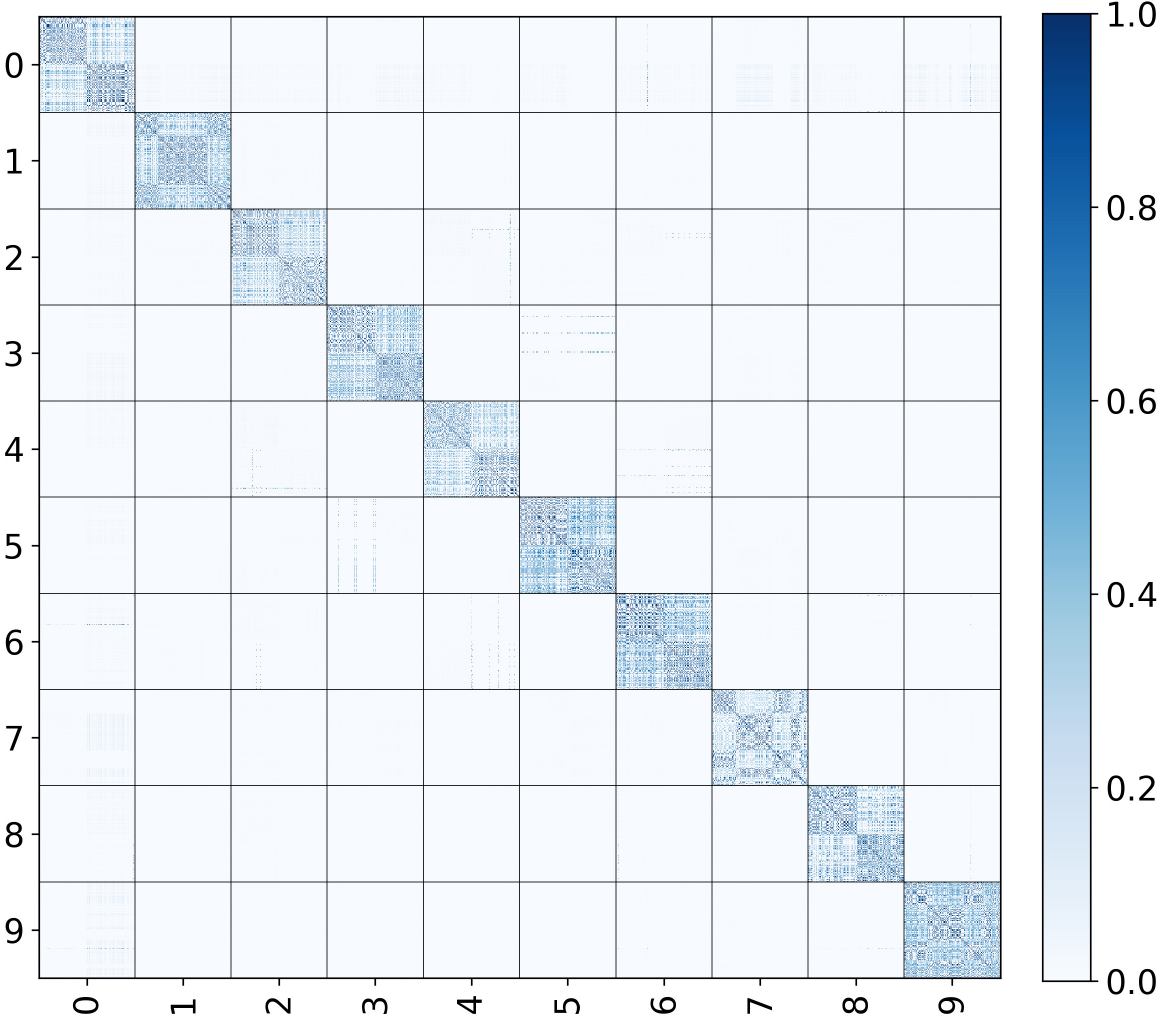}
%\caption{fig1}
\end{minipage}%
}%
\subfigure[Independent MCR$^2$]{
\begin{minipage}[t]{0.24\textwidth}
\centering
\includegraphics[width=\textwidth]{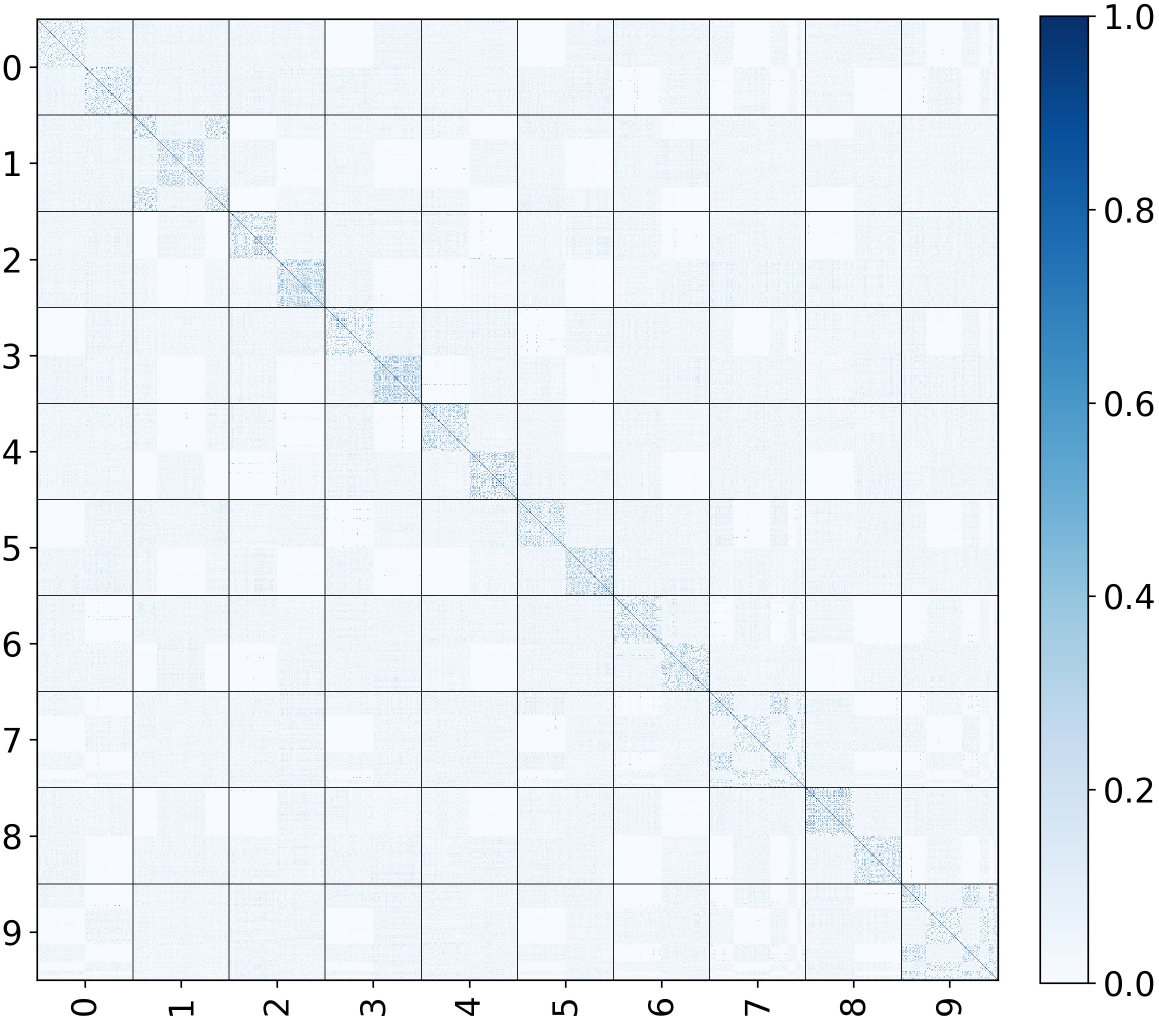}
%\caption{fig2}
\end{minipage}
}%
\vspace{-0.2cm}
\centering
\caption{Cosine similarity between learned representations for CIFAR10 under non-i.i.d. data distribution.}\label{fig_exp9}
\vspace{-0.3cm}
\end{figure}
Then, we consider more complicated conditions with $5$ nodes in the network, each with $4$ random labels, which are $[1,3,5,6], [0,5,7,8], [1,3,8,9], [2,4,6,7],[0,2,4,9]$ respectively. Then, based on Algorithm \ref{Alg2}, we could have two clusters: cluster $1$ with nodes $\{0,1,4\}$ and cluster $2$ with nodes $\{1,2,3\}$, with the local virtual replication in node $1$. The cosine similarity of learned representations can be shown in Fig. \ref{fig_exp10} for MNIST and CIFAR-10 datasets. In CIFAR-10 datasets, the nodes have ResNet18, VGG11, VGG16, ResNet34, VGG11 model architectures, respectively.

\begin{figure}[htbp]
\centering
\subfigure[MNIST]{
\begin{minipage}[t]{0.24\textwidth}
\centering
\includegraphics[width=\textwidth]{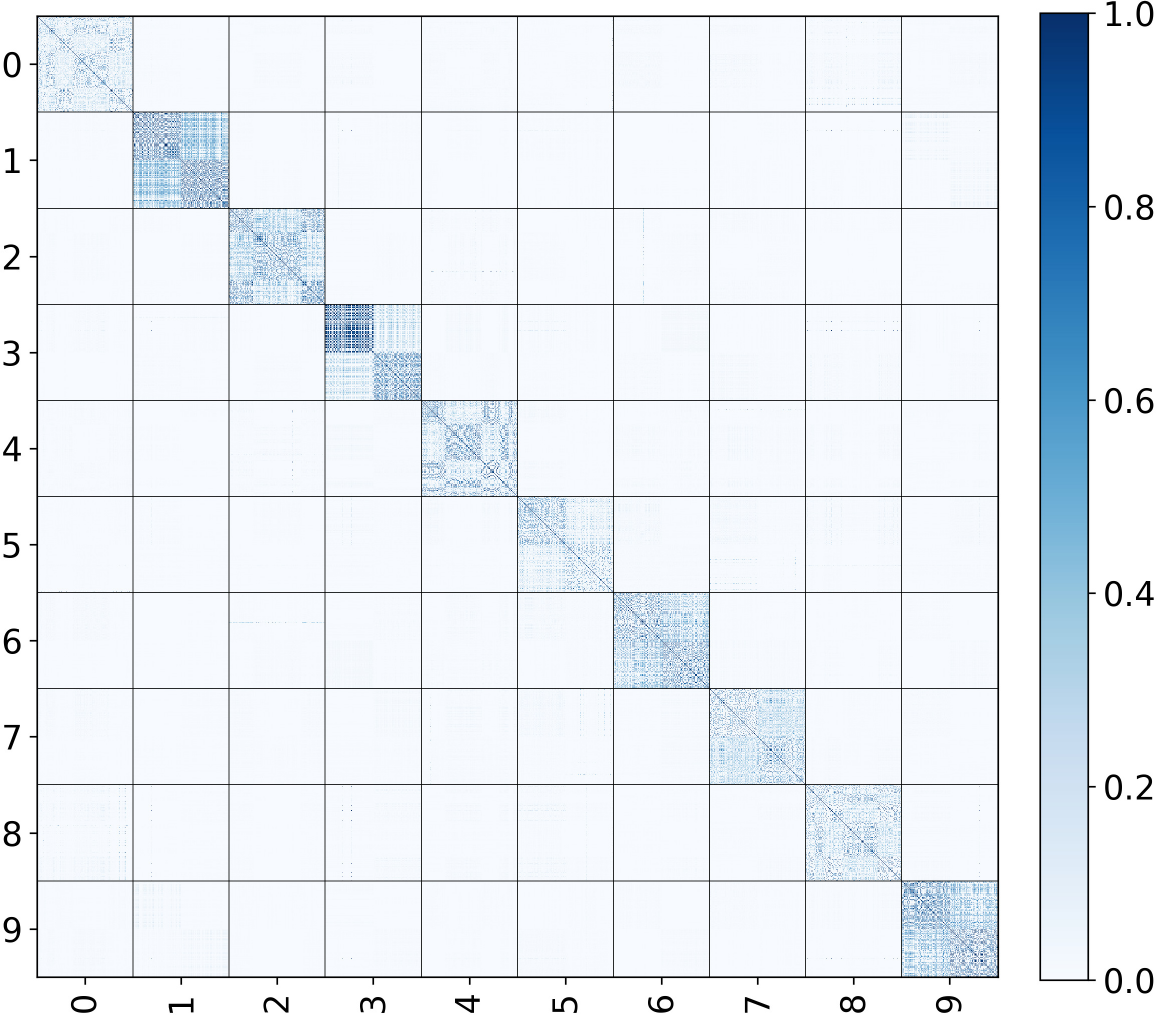}
%\caption{fig1}
\end{minipage}%
}%
\subfigure[CIFAR10]{
\begin{minipage}[t]{0.24\textwidth}
\centering
\includegraphics[width=\textwidth]{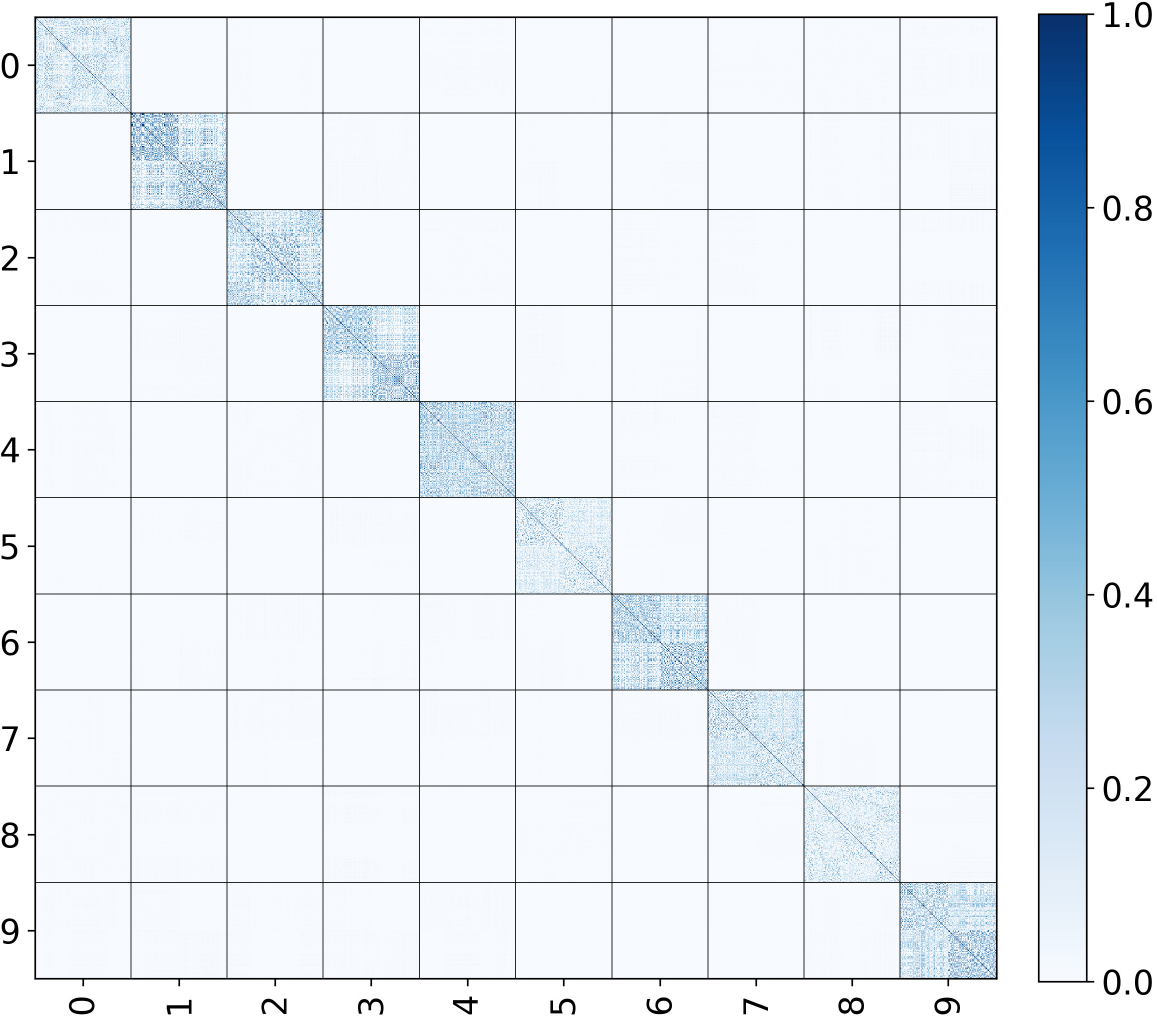}
%\caption{fig2}
\end{minipage}
}%
\vspace{-0.2cm}
\centering
\caption{Cosine similarity of learned representations for MNIST and CIFAR-10 in non-i.d.d. condition under virtual node replication.}\label{fig_exp10}
\vspace{-0.3cm}
\end{figure}

\subsection{Metrics Comparison}
In this subsection, we evaluate {and compare different metrics of the learned representation neural networks between our proposed algorithms and those baseline methods.} 
It is important to note that the proposed methods train each local model to map its local datasets into a shared global representation space. 
This allows each local model to capture a partial view of the global datasets, making the trained representation model personalized for each local perspective. For classification tasks, we propose averaging the output representations across nodes for each test data sample, which can be considered the global representation of the test data. Based on this global representation, classification is performed using a natural nearest subspace classifier, similar to the approach in \cite{yu2020learning}, which uses the subspaces of global representations from the training data samples.
Specifically, for the representations of training data samples corresponding to class $k$, denoted by $\bm{Z}_k$, let $\bm{\mu}_k\in\mathbb{R}^d$ be its mean and $\bm{\Phi}_k\in\mathbb{R}^{d\times r_k}$ be the first $r_k$ principle components for $\bm{Z}_k$.
Define $z' = \frac{1}{N}\sum_{i\in\mathcal{V}}f_i(x',\bm{\theta}_i)$, then the predicted label of a test data sample $\bm{x}'$ is given by
\begin{align}
    k' = {\arg\min}_{k\in\mathcal{K}}\|(\bm{I}-\bm{\Phi}_k\bm{\Phi}_k^T)(z'-\bm{\mu}_k)\|_2^2,
\end{align}

{In addition to testing accuracy, to comprehensively evaluate the learned representations, we design a set of geometry-aware metrics that characterize neural collapse behavior, representation properties, and cross-node semantic alignment in distributed settings.
Firstly, to assess class-level geometry and neural collapse, we compute the Class Mean Cosine Similarity Matrix, whose off-diagonal statistics quantify the angular relationships between class prototypes (mean/std). In addition, we measure the Within-Class Covariance Collapse Ratio (WCCR), which captures the proportion of variance retained within classes relative to the total variance, serving as a direct indicator of whether representations exhibit within-class collapse.
To further evaluate discriminative structure without relying on classifier parameters, we report the Inter-Class to Intra-Class Distance Ratio (IIDR), which jointly reflects class separation and intra-class compactness under non-collapsed representations.
Beyond single-model geometry, we investigate representation alignment across clients. We employ Centered Kernel Alignment (CKA) to measure subspace-level similarity between embeddings learned by different clients, revealing whether distributed training enforces excessive representational collapse. 
These metrics provide a holistic view of representation quality, allowing us to disentangle neural collapse mitigation, semantic alignment, and representational diversity in distributed representation learning.}

The comparisons of these metrics between the proposed algorithms and baseline methods in i.i.d. condition are shown in Table \ref{exp_table}, where testing accuracy is the averaged results over 5 trails along with 95\% confidence intervals. {We additionally compare the results under CIFAR100 dataset.}
From Table \ref{exp_table}, considering MNIST, Algorithm \ref{Alg1} achieves competitive accuracy while significantly improving structural metrics.
On CIFAR10, Algorithm \ref{Alg1} achieves the best classification performance and maintains strong representation compactness.
Across both datasets, the off-diagonal statistics (mean/std) exhibit a relatively small mean off-diagonal similarity with low variance, indicating that class prototypes are nearly orthogonal and evenly distributed in the embedding space, which is consistent with Theorem \ref{t1}.
Moreover, a consistent trend can be observed that Algorithm 1 achieves the highest WCCR and the lowest IIDR. This suggests that the proposed method explicitly encourages representation compactness and reduces cross-client feature divergence. In contrast, several baseline methods show higher IIDR and reduced WCCR, indicating potential representation misalignment or instability.
% The structural metrics (WCCR and IIDR) correlate strongly with the final classification performance on CIFAR10, further validating the effectiveness of the proposed regularization mechanism.
Last but not least, although some methods including FedU$^2$ and FedSimCLR exhibit competitive CKA values, they suffer from higher IIDR or degraded ACC, implying a trade-off between feature similarity and discriminative power. In contrast, Algorithm \ref{Alg1} maintains both discriminative accuracy and structural coherence.

{On the other hand, Table 1 indicates that both the proposed algorithm and the centralized MCR$^2$ principle exhibit limited performance on the CIFAR100 dataset. This can be attributed to the large number of classes, which results in fewer samples per class. In the loss function (3) and those derived in Section \uppercase\expandafter{\romannumeral4}, the accuracy of coding rate estimation strongly depends on the sample size. Furthermore, the analysis of computational and communication costs in Section \uppercase\expandafter{\romannumeral4} shows that these costs increase significantly with larger $K$. Therefore, the MCR$^2$ principle is better suited to scenarios with a smaller number of classes.}

%It could be observed that the proposed algorithm does not highly improve the accuracy compared with conventional D-SGD. Indeed, through the task-specific D-SGD, the learned representations within one class converge to some consensus point and naturally enhance the classification performance. On the other hand, our proposed methods yield a representation that lies in structural subspaces, which can be used for solving further downstream tasks.
\begin{table}[!htp]
\centering
\caption{Comparison Among Different Algorithms.}
\label{exp_table}
\small
\setlength{\tabcolsep}{3pt}
\begin{tabular}{llllll}
\hline
Method & ACC(\%) & mean/std & WCCR & IIDR & CKA  \\ 
\hline

\multicolumn{6}{c}{\textbf{MNIST}} \\ 
\hline
Algorithm \ref{Alg1}   & 0.9762$\pm$0.10\% & \textbf{0.07/0.07} & \textbf{0.9941} & \textbf{0.1159} & 0.0245 \\
CenMCR       & 0.9838$\pm$0.12\% & 0.09/0.03 & 0.9235 & 0.4292 & -\\
D-SGD         & \textbf{0.9851$\pm$0.02\%} & -0.08/0.25 & 0.2100 & 3.0249 & 0.0013 \\
FedU$^2$      & 0.9560$\pm$0.21\% & 0.29/0.13 & 0.8224 & 0.6939 & 0.0060 \\
FedSimCLR     & 0.9743$\pm$0.09\% & 0.81/0.06 & 0.6556 & 1.0855 & 0.0028  \\
FedU          & 0.9773$\pm$0.12\% & 0.60/0.13 & 0.6189 & 1.1770 & 0.0029 \\
Orchestra     & 0.9813$\pm$0.21\% & 0.91/0.05 & 0.5556 & 1.3454 & 0.0013\\
\hline

\multicolumn{6}{c}{\textbf{CIFAR10}} \\ 
\hline
Algorithm \ref{Alg1}   & \textbf{71.76$\pm$0.06\%} & \textbf{0.00/0.02} & \textbf{0.9914} & \textbf{0.1399} & 0.0208 \\
CenMCR       & 71.08$\pm$0.02\% & 0.00/0.01 & 0.9153 & 0.4541 & - \\
D-SGD         & 68.26$\pm$0.10\% & -0.11/0.09 & 0.9044 & 0.4848 & 0.0018 \\
FedU$^2$      & 41.18$\pm$0.20\%& 0.92/0.05 & 0.9028 & 0.4900 & 0.0027 \\
FedSimCLR     & 49.03$\pm$0.41\% & 0.71/0.15 & 0.8179 & 0.6901 & 0.0022 \\
FedU          & 68.08$\pm$0.21\% & 0.53/0.13 & 0.8689 & 0.5747 & 0.0068 \\
Orchestra     & 64.02$\pm$0.19\% & 0.64/0.23 & 0.5539 & 1.3275 & 0.0011 \\
\hline

\multicolumn{6}{c}{\textbf{CIFAR100}} \\ 
\hline
Algorithm \ref{Alg1}   & 27.61$\pm$0.16\% & \textbf{0.00/0.02} & \textbf{0.9967} & 0.3834 & 0.0300 \\
CenMCR       & 27.84$\pm$0.22\% & 0.01/0.03 & 0.5915 & 0.7728 & - \\
D-SGD         & 30.20$\pm$0.10\% & 0.75/0.10 & 0.9010 & \textbf{0.1442} & 0.0087 \\
FedU$^2$      & 30.13$\pm$0.18\%& 0.16/0.21 & 0.8833 & 0.5186 & 0.0081 \\
FedSimCLR     & 28.90$\pm$0.21\% & 0.49/0.20 & 0.7857 & 0.7371 & 0.0037 \\
FedU          & \textbf{31.82$\pm$0.20\%} & 0.24/0.23 & 0.8103 & 0.6836 & 0.0057 \\
Orchestra     & 30.36$\pm$0.15\% & 0.95/0.03 & 0.5035 & 0.8093 & 0.0009 \\
\hline
\end{tabular}
\end{table}
\vspace{-0.5cm}

\section{Conclusion}\label{Sec7}
In this work, we proposed a novel distributed semantic-based representation learning framework, which can jointly capture the global structural representations in a multi-node system. 
%which can extract diverse and discriminative features across distributed data spaces. 
Both i.i.d. and non-i.i.d data distributions were studied, where we reformulated the global optimization problem into equivalent decomposable ones. 
With the derivation of updating rules based on augmented Lagrangian function and dual-decomposition procedure, the sub-problems can be solved iteratively among distributed nodes. 
The properties of the optimal solutions and convergence of the proposed algorithm were theoretically given. Finally, we conducted extensive simulations to validate the superiority of the proposed algorithms in terms of structural representations.

The diverse and discriminative properties of the global representations are promising for solving more complicated downstream tasks, such as distributed multi-view problems and semantic communication. {Its performance in dynamic network topologies and generalization to other non-vision tasks need to be further investigated.}
{Another topic exploiting the representations is to apply it to generative tasks. Thirdly, the proposed distributed representation learning algorithm without transmission of model parameters may face the generalization challenge with limited local datasets, calling for a more comprehensive analysis of its generalization capability.
}

\vspace{-0.3cm}
\section*{Appendix}
{
\subsection{Derivation of Proposition \ref{p1}}
The augmented Lagrangian function for (\ref{opt1_2}) can be formulated as follows:
\begin{align}
    \label{lagran1}
    & \mathcal{L}(\bm{Z}, \bm{T}, \bm{U}, \bm{V}) = \sum_{i=1}^N\Large\{R_i^c(\bm{Z}_i) - R_i(\bm{Z}_i) + \\
    & \sum_{j\in\mathcal{N}_i}\sum_{k=1}^K\Big\{\text{tr}[\bm{U}_{i,j,k}^T(\frac{\bm{Z}_{i,k}\bm{Z}_{i,k}^T}{m_{i,k}}-\bm{T}_{i,j,k})] + \notag \\
    & \text{tr}[\bm{V}_{i,j,k}^T(\frac{\bm{Z}_{j,k}\bm{Z}_{j,k}^T}{m_{j,k}}-\bm{T}_{i,j,k})]\Big\} + \notag \\
    & \frac{\gamma}{2}\sum_{j\in\mathcal{N}_i}\sum_{k=1}^K\Big[\|\frac{\bm{Z}_{i,k}\bm{Z}_{i,k}^T}{m_{i,k}}-\bm{T}_{i,j,k}\|_F^2+\|\frac{\bm{Z}_{j,k}\bm{Z}_{j,k}^T}{m_{j,k}}-\bm{T}_{i,j,k}\|_F^2\Big]\Large\},\notag 
\end{align}
where $\gamma$ is a penalty parameter. To solve the problem, we use the dual gradient ascent and Gauss-Seidel strategy \cite{matamoros2015distributed, hong2017linear, tian2023distributed2}.
Then, the update rules of the dual variables $\{\bm{U}_{i,j,k}, \bm{V}_{i,j,k}\}$, the primal variables $\bm{Z}_i$ and the intermediate variable $\bm{T}_{i,j,k}$ can be derived alternatively, for all $i\in\mathcal{V}, j\in\mathcal{N}_i$ and $1\le k\le K$, as follows:
\begin{subequations}
	\begin{equation}
		\bm{U}_{i,j,k}^{(t)}=\bm{U}_{i,j,k}^{(t-1)}+\rho[\frac{\bm{Z}_{i,k}^{(t-1)}{\bm{Z}_{i,k}^{(t-1)}}^T}{m_{i,k}}-\bm{T}_{i,j,k}^{(t-1)})],  \label{R1_1a}
	\end{equation}
	\begin{equation}
		\bm{V}_{i,j,k}^{(t)}=\bm{V}_{i,j,k}^{(t-1)}+\rho[\frac{\bm{Z}_{j,k}^{(t-1)}{\bm{Z}_{j,k}^{(t-1)}}^T}{m_{j,k}}-\bm{T}_{i,j,k}^{(t-1)})], \label{R1_1b}	\end{equation}
	\begin{equation}
			\bm{Z}^{(t)} = \mathop{\arg\min}\mathcal{L}(\bm{U}^{(t)}, \bm{V}^{(t)},\bm{Z},\bm{T}^{(t-1)}), \label{R1_1c}
	\end{equation}
        \begin{equation}
        \begin{aligned}
			&\bm{T}_{i,j,k}^{(t)}=\arg\min\Big{\{}-\text{tr}[{(\bm{U}_{i,j,k}^{(t)}+\bm{V}_{i,j,k}^{(t)})}^T\bm{T}_{i,j,k}] + \\
            & + \frac{\gamma}{2}\Big[\|\frac{\bm{Z}_{i,k}^{(t)}{\bm{Z}_{i,k}^{(t)}}^T}{m_{i,k}}-\bm{T}_{i,j,k}\|_F^2+\|\frac{\bm{Z}_{j,k}^{(t)}{\bm{Z}_{j,k}^{(t)}}^T}{m_{j,k}}-\bm{T}_{i,j,k}\|_F^2\Big]\Big{\}}. \label{R1_1d} 
        \end{aligned}
        \end{equation}
\end{subequations}
Here, $\rho$ is the step size for the gradient ascent of the dual variables.
From (\ref{R1_1d}), the closed form solution of $\bm{T}_{i,j,k}^{(t)}$ is:
\begin{equation}
\bm{T}_{i,j,k}^{(t)}=\frac{1}{2\gamma}\big[\bm{U}_{i,j,k}^{(t)}+\bm{V}_{i,j,k}^{(t)}\big]+\frac{1}{2}\big[\frac{\bm{Z}_{i,k}^{(t)}{\bm{Z}_{i,k}^{(t)}}^T}{m_{i,k}}+\frac{\bm{Z}_{j,k}^{(t)}{\bm{Z}_{j,k}^{(t)}}^T}{m_{j,k}}\big]. 
\label{R1_2}
\end{equation}
Under the initialization $\bm{U}_{i,j,k}^{(0)}+\bm{V}_{i,j,k}^{(0)}=\bm{0}$, the update rules (\ref{R1_1a}) and (\ref{R1_1b}) and the closed form of $\bm{T}_{i,j,k}^{(t)}$, it can be easily deduced that 
\begin{align}
\bm{U}_{i,j,k}^{(t)}+\bm{V}_{i,j,k}^{(t)}=\bm{0}, \quad \forall t\ge 1,    
\end{align}
which means the updating steps (\ref{R1_1a}) and (\ref{R1_1b}) can be simplified into one step (\ref{R1_1a}).
Then $\bm{T}_{i,j,k}^{(t)}$ in (\ref{R1_2}) can be readily expressed by: 
\begin{equation}
\label{R1_T}
\bm{T}_{i,j,k}^{(t)}=\frac{1}{2}\big[\frac{\bm{Z}_{i,k}^{(t)}{\bm{Z}_{i,k}^{(t)}}^T}{m_{i,k}}+\frac{\bm{Z}_{j,k}^{(t)}{\bm{Z}_{j,k}^{(t)}}^T}{m_{j,k}}\big].
\end{equation}
From (\ref{R1_T}), it can be observed $\bm{T}_{i,j,k}^{(t)} = \bm{T}_{j,i,k}^{(t)}$
We additionally initialized $\bm{U}_{i,j,k}^{(0)}+\bm{U}_{j,i,k}^{(0)}=\bm{0}$. Then according to (\ref{R1_1a}) and (\ref{R1_T}), we could have $\bm{U}_{i,j,k}^{(t)}+\bm{U}_{j,i,k}^{(t)}=\bm{0}$ for all iterations $t$. And we have $\bm{U}_{i,j,k}^{(t)} = \bm{V}_{j,i,k}^{(t)}$.
We define 
$$\bm{Y}_{i,j,k}\triangleq(\bm{U}_{i,j,k}+\bm{V}_{j,i,k})=2\bm{U}_{i,j,k},$$ and initialize $\bm{U}_{i,j,k}^{(0)}=\bm{V}_{j,i,k}^{(0)}=\bm{0}$.
Then, $\bm{Y}_{i,j,k}$ is initialized as $\bm{Y}_{i,j,k}^{(0)}\triangleq(\bm{U}_{i,j,k}^{(0)}+\bm{V}_{j,i,k}^{(0)})=\bm{0}$.
Considering the closed form of $\bm{T}_{i,j,k}^{(t)}$ in (\ref{R1_T}), the update rules can be simplified into the two steps in \eqref{R1_2a} and \eqref{R1_2b} as in Proposition \ref{p1}.

\subsection{Derivation of Proposition \ref{p2}}
The global augmented Lagrangian function can then be formulated as follows:
\begin{align}
    \label{lagran2_1}
    & \mathcal{L} = \sum_{s=1}^S\Big[R^c_s(\bm{Z}^s)-\frac{m^s}{2m}\log\det(\bm{I}+\frac{d}{m^s\epsilon^2}\sum_{i\in\tilde{\mathcal{V}}^s}\frac{1}{S_i}\bm{Z}_{i}\bm{Z}_{i}^T)\Big] \notag\\
    & + \sum_{i\in\tilde{\mathcal{V}}}\sum_{j\in\tilde{\mathcal{V}}_k/\{i\}}\sum_{k\in\mathcal{K}_i}\Big\{\text{tr}[\bm{U}_{i,j,k}^T(\frac{\bm{Z}_{i,k}\bm{Z}_{i,k}^T}{m_{i,k}}-\bm{T}_{i,j,k})] + \notag \\
    &\quad\text{tr}[\bm{V}_{i,j,k}^T(\frac{\bm{Z}_{j,k}\bm{Z}_{j,k}^T}{m_{j,k}}-\bm{T}_{i,j,k})]  + \\
    &\quad\frac{\gamma}{2}\Big[\|\frac{\bm{Z}_{i,k}\bm{Z}_{i,k}^T}{m_{i,k}}-\bm{T}_{i,j,k}\|_F^2+
    \|\frac{\bm{Z}_{j,k}\bm{Z}_{j,k}^T}{m_{j,k}}-\bm{T}_{i,j,k}\|_F^2\Big]\Big\}.\notag
\end{align}
Similar to the definition in Section \ref{Sec3}, we define $\bm{Y}_{i,j,k}\triangleq(\bm{U}_{i,j,k}+\bm{V}_{j,i,k})=2\bm{U}_{i,j,k}$, and $\bm{U}_{i,j,k}^{(0)}=\bm{V}_{j,i,k}^{(0)}=\bm{0}$.
Then, $\bm{Y}_{i,j,k}$ is initialized as $\bm{Y}_{i,j,k}^{(0)}=\bm{0}$.
Following similar derivations in Section \ref{Sec3}, we could get the global updating rules for all $i\in\tilde{\mathcal{V}}$ and $k=1,\dots, K$ as follows:
\begin{subequations}
	\begin{equation}
		\bm{Y}_{i,j,k}^{(t)}=\bm{Y}_{i,j,k}^{(t-1)}+\rho[\frac{\bm{Z}_{i,k}^{(t-1)}{\bm{Z}_{i,k}^{(t-1)}}^T}{m_{i,k}}-\frac{\bm{Z}_{j,k}^{(t-1)}{\bm{Z}_{j,k}^{(t-1)}}^T}{m_{j,k}}],  \label{R2_1a_ap}
	\end{equation}
	\begin{equation}
		\bm{Z}^{(t)} = \mathop{\arg\min}\mathcal{L}(\bm{Y}^{(t)},\bm{Z}). \label{R2_1b_ap}	
        \end{equation}
\end{subequations}
%where \begin{align} \label{lagran2_2} & \mathcal{L}(\bm{Y}^{(t)},\bm{Z}) = \sum_{s=1}^S\Big[-\frac{m^s}{2m}\log\det(\bm{I}+\frac{d}{m^s\epsilon^2}\sum_{i\in\tilde{\mathcal{V}}^s}\frac{1}{S_i}\bm{Z}_{i}\bm{Z}_{i}^T)\notag\\ &+ R_s^c(\bm{Z}^s)\Big]  +\sum_{i\in\tilde{\mathcal{V}}}\sum_{j\in\tilde{\mathcal{V}}_k/\{i\}}\sum_{k\in\mathcal{K}_i}\Big\{\text{tr}[{\bm{Y}_{i,j,k}^{(t)}}^T(\frac{\bm{Z}_{i,k}\bm{Z}_{i,k}^T}{m_{i,k}}-\\  &\frac{\bm{Z}_{j,k}^{(t-1)}{\bm{Z}_{j,k}^{(t-1)}}^T)}{m_{j,k}}]+\gamma\big\|\frac{\bm{Z}_{i,k}\bm{Z}_{i,k}^T}{m_{i,k}}-\frac{\frac{\bm{Z}_{i,k}^{(t-1)}{\bm{Z}_{i,k}^{(t-1)}}^T}{m_{i,k}}+\frac{\bm{Z}_{j,k}^{(t-1)}{\bm{Z}_{j,k}^{(t-1)}}^T}{m_{j,k}}}{2}\big\|_F^2\Big\}.\notag \end{align}

The updating of $\bm{Y}_{i,k}$ can be implemented by each node in parallel. 
However, $\bm{Z}_i$ is coupled in the Lagrangian function (\ref{R2_1b_ap}).
Following we will derive how to solve the problem (\ref{R2_1b_ap}) among distributed nodes.
Note that
$
\sum_{i\in\tilde{\mathcal{V}}}\sum_{k\in\mathcal{K}_i} f(\bm{Z}_{i,k})= \sum_{s=1}^S\sum_{i\in\tilde{\mathcal{V}}^{s}}\sum_{k\in\mathcal{K}_i}  f(\bm{Z}_{i,k}),
$
both of which sum up the terms in all (virtual) nodes and all of their classes. 
Given this equivalence, problem in (\ref{R2_1b_ap}) can be reformulated as:
\begin{align}
    \label{lagran2_3}
    & \mathcal{L} = \sum_{s=1}^S\Big\{ -\frac{m^s}{2m}\log\det(\bm{I}+\frac{d}{m^s\epsilon^2}\sum_{i\in\tilde{\mathcal{V}}^s}\frac{1}{S_i}\bm{Z}_{i}\bm{Z}_{i}^T)\\
    &+ R_s^c(\bm{Z}^s) + \sum_{i\in\tilde{\mathcal{V}}^{s}}\sum_{j\in\tilde{\mathcal{V}}_k/\{i\}}\sum_{k\in\mathcal{K}_i}\Big[  \text{tr}[{\bm{Y}_{i,j,k}^{(t)}}^T(\frac{\bm{Z}_{i,k}\bm{Z}_{i,k}^T}{m_{i,k}}\notag\\
    &-\frac{\bm{Z}_{j,k}^{(t-1)}{\bm{Z}_{j,k}^{(t-1)}}^T}{m_{j,k}})]\notag\\
    & + \gamma\big\|\frac{\bm{Z}_{i,k}\bm{Z}_{i,k}^T}{m_{i,k}}-\frac{\frac{\bm{Z}_{i,k}^{(t-1)}{\bm{Z}_{i,k}^{(t-1)}}^T}{m_{i,k}}+\frac{\bm{Z}_{j,k}^{(t-1)}{\bm{Z}_{j,k}^{(t-1)}}^T}{m_{j,k}}}{2}\big\|_F^2\Big]\Big\}. \notag
\end{align}

We define $\mathcal{L} = \sum_{s=1}^S \mathcal{L}^s$, then the variables among different clusters can be decomposed in (\ref{lagran2_3}), where different clusters can be updated in parallel based on the clustered augmented Lagrangian function.
\begin{equation}
    (\bm{Z}^s)^{(t)} = \mathop{\arg\min}\mathcal{L}^s(\bm{Y}^{(t)},\bm{Z}^s). \label{R2_2b}	
\end{equation}

Within each cluster, the variables in (\ref{R2_2b}) can be decoupled into the distributed nodes using the blocked coordinate descent procedure (BCD) \cite{beck2013convergence, hong2017iteration}. 
Specifically, consider node $i$ in the $s$-th cluster, the concatenation of $\bm{Z}_{i, k}$ for all $k\in\mathcal{K}_i$ is denoted by $\bm{Z}_{i}$.
Then based on the block coordinate descent procedure, the solution in node $i$ can be obtained through iteratively updating the following step for all $i\in\mathcal{V}^s$.
\begin{align}
(\bm{Z}_{i})^{(t)} = \mathop{\arg\min}\mathcal{L}^s(&\bm{Y}^{(t)},(\bm{Z}_{i_1})^{(t)},\cdots,(\bm{Z}_{i}),\cdots,\notag\\
&(\bm{Z}_{N^s})^{(t-1)}), \label{R2_3_ap}	
\end{align}
where
\begin{align}
&\mathcal{L}^s=-\frac{m^s}{2m}\log\det(\bm{I}+\frac{d}{m^s\epsilon^2}\sum_{i\in\tilde{\mathcal{V}}^s}\frac{1}{S_i}\bm{Z}_{i}\bm{Z}_{i}^T)+ R_s^c(\bm{Z}^s)\notag\\
& + \sum_{i\in\tilde{\mathcal{V}}^{s}}\sum_{j\in\tilde{\mathcal{V}}_k/\{i\}}\sum_{k\in\mathcal{K}_i}  \Big\{ \text{tr}[{\bm{Y}_{i,j,k}^{(t)}}^T(\frac{\bm{Z}_{i,k}\bm{Z}_{i,k}^T}{m_{i,k}}-\frac{\bm{Z}_{j,k}^{(t-1)}{\bm{Z}_{j,k}^{(t-1)}}^T}{m_{j,k}})] \notag\\
& + \gamma\big\|\frac{\bm{Z}_{i,k}\bm{Z}_{i,k}^T}{m_{i,k}}-\frac{\frac{\bm{Z}_{i,k}^{(t-1)}{\bm{Z}_{i,k}^{(t-1)}}^T}{m_{i,k}}+\frac{\bm{Z}_{j,k}^{(t-1)}{\bm{Z}_{j,k}^{(t-1)}}^T}{m_{j,k}}}{2}\big\|_F^2\Big\}. \label{R2_4_ap}
\end{align}
}

\subsection{Proof of Theorem \ref{t1}}
\begin{proof}
    We prove Theorem \ref{t1} by proving the equivalence of the reformulated optimization problem and the original problem in MCR$^2$ \cite{yu2020learning}.
    Under the constraints: 
    $$
    \frac{\bm{Z}_{i,k}\bm{Z}_{i,k}^T}{m_{i,k}}=\frac{\bm{Z}_{j,k}\bm{Z}_{j,k}^T}{m_{j,k}}, \forall i\in\mathcal{V}, j\in\mathcal{N}_i, 1\le k\le K,
    $$
    in the connected communication network, this equality holds between any two nodes in the network topology $\forall i\in\mathcal{V}, j\in\mathcal{V}$.
    Moreover, according to the definition of $\bm{Z}_{i}$ and based on Assumption \ref{as1}, we further have
    $$\bm{Z}_{i}\bm{Z}_{i}^T=\sum_{k=1}^Km_{i,k}\frac{\bm{Z}_{j,k}\bm{Z}_{j,k}^T}{m_{j,k}}=\sum_{k=1}^Km_{i}\frac{\bm{Z}_{j,k}\bm{Z}_{j,k}^T}{m_{j}}=m_{i}\frac{\bm{Z}_{j}\bm{Z}_{j}^T}{m_{j}}.
    $$ 
    Then the equality in (\ref{eqR1}) holds. Thus, the reformulated global optimization problem (\ref{opt1}) is equivalent to the original one in (\ref{MCR2}) and the optimal solution of the decoupled sub-problems in the nodes is equivalent to the solution of (\ref{MCR2}). 
    Then according to Theorem 2.1 in \cite{yu2020learning}, we could have that $(\bm{Z}_{k1}^*)^T\bm{Z}_{k2}^*=\bm{0}$ for different classes $1\le k1<k2\le K$, based on which we obtain the property 1) and 2). Moreover, the optimal solutions have the diversity property 3) within a class. 
\end{proof}

\subsection{Proof of Theorem \ref{t2}}
\begin{proof}
    According to the updating rule in (\ref{R1_2c}), in the $t$-th iteration, we define the local loss function of node $i$ as $\mathcal{L}_i(\bm{Y}_i^{(t)},\bm{\theta}_i)=\mathcal{L}'(\bm{Y}_{i}^{(t)},\bm{Z}_i)$, with $\bm{Z}_i=f_i(\bm{\theta}_i, \bm{X}_i)$. 
    The value of the global loss function in the $t$-th iteration is defined as 
    $$
    \mathcal{L}(\bm{Y}^{(t)},\theta)=\sum_{i\in\mathcal{V}}\mathcal{L}_i(\bm{Y}_i^{(t)},\bm{\theta}_i).
    $$
    For each local node, without loss of generality, we have
    \begin{align}
    \label{app_eq1}
    &\mathbb{E}[\mathcal{L}_i(\bm{Y}_i^{(t+1)},\bm{\theta}_i^{(t+1)})-\mathcal{L}_i(\bm{Y}_i^{(t)},\bm{\theta}_i^{(t)})]\\
    & = \mathbb{E}[\mathcal{L}_i(\bm{Y}_i^{(t+1)},\bm{\theta}_i^{(t+1)})-\mathcal{L}_i(\bm{Y}_i^{(t+1)},\bm{\theta}_i^{(t)})]+ \notag\\ 
    &\quad \mathbb{E}[\mathcal{L}_i(\bm{Y}_i^{(t+1)},\bm{\theta}_i^{(t)})-\mathcal{L}_i(\bm{Y}_i^{(t)},\bm{\theta}_i^{(t)})], \notag
    \end{align}
    where the expectations are taken over the randomness in stochastic gradients. We analyze the two terms respectively on the right side of (\ref{app_eq1}). 
    
    According to our updating rule (\ref{R1_2c}), the first term on the right side of (\ref{app_eq1}) can be seen as the reduction of the local loss function in the $(t+1)$-th iteration after several steps of SGD. 
    Under Assumption \ref{as2}, we have 
    $$\|\nabla L_i(\bm{Y}_i^{(0)},\bm{\theta}_i)-\nabla L_i(\bm{Y}_i^{(0)},\bm{\theta}_i')\|\le L_0\| \bm{\theta}_i - \bm{\theta}_i'\|.$$
    In the next iteration, it can be derived that 
    \begin{align}
        \label{app_eq6}
        &\|\nabla L_i(\bm{Y}_i^{(1)},\bm{\theta}_i)-\nabla L_i(\bm{Y}_i^{(1)},\bm{\theta}_i')\|  \\
        &\le 2\|\nabla L_i(\bm{Y}_i^{(0)},\bm{\theta}_i)-\nabla L_i(\bm{Y}_i^{(0)},\bm{\theta}_i')\| + 2\|\nabla L_i(\bm{Y}_i^{(1)},\bm{\theta}_i)-\notag\\
        &\quad \nabla L_i(\bm{Y}_i^{(0)},\bm{\theta}_i)+\nabla L_i(\bm{Y}_i^{(0)},\bm{\theta}_i')-\nabla L_i(\bm{Y}_i^{(1)},\bm{\theta}_i')\| \notag 
    \end{align}

    The second term on the right side of (\ref{app_eq5}) can be bounded with the calculated gradients on $\theta$ and the normalization property of $\bm{Z}$, which we omit here due to limited space.
    Following this relationship between the present and previous iterations, together with the smoothness of the neural network function $f_i(\bm{\theta}_i)$, it is easy to prove for any $t$-th iteration, $\mathcal{L}_i(\bm{Y}_i^{(t)},\bm{\theta}_i)$ is $L_t$-smooth, with a limited scalar $L_t$.
    
    Then under Assumption \ref{as2} and \ref{as3}, based on the $L_t$-smoothness of $\mathcal{L}_i(\bm{Y}_i^{(t)},\bm{\theta}_i)$, we could obtain the following results according to the properties of $T'$ steps of local SGD. Due to space limitation, we refer readers to Lemma 3 in \cite{tian2023distributed} for details.
    \begin{align}
    \label{app_eq2}
        & \mathbb{E}[\mathcal{L}_i(\bm{Y}_i^{(t+1)},\bm{\theta}_i^{(t+1)})-\mathcal{L}_i(\bm{Y}_i^{(t+1)},\bm{\theta}_i^{(t)})]\le \\
        & -cT'\eta\|\nabla \mathcal{L}_i(\bm{Y}_i^{(t+1)},\bm{\theta}_i^{(t)})\|_2^2 + \frac{\eta^2T'L_{t+1}}{2}(1+4\eta T'L_{t+1})\kappa^2. \notag
    \end{align}
    For the second term on the right side of  (\ref{app_eq1}), based on the definition of $\mathcal{L}_i$, we have 
    \begin{align}
    \label{app_eq3}
        & \mathbb{E}[\mathcal{L}_i(\bm{Y}_i^{(t+1)},\bm{\theta}_i^{(t)})-\mathcal{L}_i(\bm{Y}_i^{(t)},\bm{\theta}_i^{(t)})] \le \\
        & \sum_{j\in\mathcal{N}_i}\sum_{k=1}^K\text{tr}[{(\bm{Y}_{i,j,k}^{(t+1)}-\bm{Y}_{i,j,k}^{(t)})}^T(\frac{{\bm{Z}_{i,k}^{(t)}}{\bm{Z}_{i,k}^{(t)}}^T}{m_{i,k}}-\frac{\bm{Z}_{j,k}^{(t)}{\bm{Z}_{j,k}^{(t)}}^T}{m_{j,k}})]\notag \\
        &+\frac{\gamma}{4}\sum_{j\in\mathcal{N}_i}\sum_{k=1}^K\big\|\frac{\bm{Z}_{i,k}^{(t)}{\bm{Z}_{i,k}^{(t)}}^T}{m_{i,k}}-\frac{\bm{Z}_{j,k}^{(t)}{\bm{Z}_{j,k}^{(t)}}^T}{m_{j,k}}\big\|_F^2 \notag \\
        & \le (2\rho + \frac{\gamma}{2})\sum_{j\in\mathcal{N}_i}\sum_{k=1}^K[\big\|\frac{\bm{Z}_{i,k}^{(t)}{\bm{Z}_{i,k}^{(t)}}^T}{m_{i,k}}\big\|_F^2+\big\|\frac{\bm{Z}_{j,k}^{(t)}{\bm{Z}_{j,k}^{(t)}}^T}{m_{j,k}}\big\|_F^2] \notag \\
        & \le (4\rho + \gamma)K|\mathcal{N}_i|, \notag
    \end{align}
    where the first inequality follows from the updating rule in (\ref{R1_2a}) and the last inequality follows from the first constraint in the optimization problem, i.e., 
    $$\|\bm{Z}_{i,k}\|_F^2 = m_{i,k}, \forall i\in\mathcal{V}, 1\le k\le K$$
    and it is satisfied in each iteration/step with the normalization function on the output layer.

    Take sum of (\ref{app_eq2}) and (\ref{app_eq3}), we could obtain that
    \begin{align}
        \label{app_eq4}
        &\mathbb{E}[\mathcal{L}_i(\bm{Y}_i^{(t+1)},\bm{\theta}_i^{(t+1)})-\mathcal{L}_i(\bm{Y}_i^{(t)},\bm{\theta}_i^{(t)})]\le (4\rho + \gamma)K|\mathcal{N}_i| \notag \\
        & -cT'\eta\|\nabla \mathcal{L}_i(\bm{Y}_i^{(t+1)},\bm{\theta}_i^{(t)})\|_2^2 + \frac{\eta^2T'L_{t+1}}{2}(1+4\eta T'L_{t+1})\kappa^2 \notag \\
        & \le -cT'\eta\|\nabla \mathcal{L}_i(\bm{Y}_i^{(t+1)},\bm{\theta}_i^{(t)})\|_2^2 +\eta^2E,
    \end{align}
    where $E = \frac{T'L}{2}(1+4\eta T'L)\kappa^2+c'K|\mathcal{N}_i|$, $L=max_{t\in[T]}L_t$, and the last inequality comes from $(4\rho + \gamma)\le c'\eta^2$.
    Take sum of $t=0,\dots, T$ over all node $i$, and rearrange the terms, we could derive that
    \begin{align}
        \label{app_eq5}
        &\min_{t\in[T]} \mathbb{E}[\|\nabla \mathcal{L}(\bm{Y}^{(t+1)},\theta^{(t)})\|_2^2]\le \frac{\mathcal{L}_0-\mathcal{L}_*}{cTT'\eta}+\frac{\eta NE}{c}.
    \end{align}
    If we choose $(4\rho+\gamma)\le c'\eta^2$ and properly choose $\eta=\frac{c''}{\sqrt{TT'}}$ such that $\eta<\frac{1}{4T'L}$, then it is easy to prove that $L$ is a limited scalar. This completes the proof.
\end{proof}

%\small
%\bibliographystyle{IEEEtran}
%\bibliography{ref}

\end{document}